\documentclass[nohyperref]{article}

\usepackage{microtype}
\usepackage{graphicx}
\usepackage{subfigure}
\usepackage{booktabs} % for professional tables

\usepackage{hyperref}

\usepackage{algorithm}% http://ctan.org/pkg/algorithms
\usepackage{algorithmic}
\usepackage[accepted]{icml2022}

\usepackage{amsmath}
\usepackage{amssymb}
\usepackage{mathtools}
\usepackage{amsthm}

\usepackage{multirow}
\usepackage{amsfonts}
\usepackage{import}
\usepackage{pifont}
\usepackage{textcomp}
\usepackage{color, colortbl}
\usepackage{xcolor}
\definecolor{greyC}{RGB}{180,180,180}
\definecolor{greyL}{RGB}{235,235,235}
\usepackage{footmisc}
\usepackage{bm}
\usepackage[flushleft]{threeparttable}
\usepackage[nospace]{cite}
\usepackage[]{hyperref} 
\usepackage{makecell}

\usepackage[capitalize,noabbrev]{cleveref}

\theoremstyle{plain}
\newtheorem{theorem}{Theorem}[section]

\newtheorem{lemma}[theorem]{Lemma}

\theoremstyle{definition}
\newtheorem{definition}[theorem]{Definition}

\theoremstyle{remark}
\newtheorem{remark}[theorem]{Remark}

\usepackage[textsize=tiny]{todonotes}

\icmltitlerunning{Virtual Homogeneity Learning: Defending against Data Heterogeneity in Federated Learning}

\begin{document}

\twocolumn[
\icmltitle{Virtual Homogeneity Learning: Defending against Data Heterogeneity in Federated Learning}



\icmlsetsymbol{equal}{*}

\begin{icmlauthorlist}
\icmlauthor{Zhenheng Tang}{equal,yyy}
\icmlauthor{Yonggang Zhang}{equal,yyy}
\icmlauthor{Shaohuai Shi}{comp}
\icmlauthor{Xin He}{yyy}
\icmlauthor{Bo Han}{yyy}
\icmlauthor{Xiaowen Chu}{sch,comp,yyy}
\end{icmlauthorlist}

\icmlaffiliation{yyy}{Department of Computer Science, Hong Kong Baptist University}
\icmlaffiliation{comp}{Department of Computer Science and Engineering, The Hong Kong University of Science and Technology}
\icmlaffiliation{sch}{Data Science and Analytics Thrust, The Hong Kong University of Science and Technology (Guangzhou)}

\icmlcorrespondingauthor{Xiaowen Chu}{xwchu@ust.hk}

\icmlkeywords{Federated Learning, Transfer Learning}

\vskip 0.3in
]



\printAffiliationsAndNotice{\icmlEqualContribution} % otherwise use the standard text.

\begin{abstract}

In federated learning (FL), model performance typically suffers from client drift induced by data heterogeneity, and mainstream works focus on correcting client drift. We propose a different approach named virtual homogeneity learning (VHL) to directly ``rectify'' the data heterogeneity. In particular, VHL conducts FL with a virtual homogeneous dataset crafted to satisfy two conditions: containing \emph{no} private information and being separable. The virtual dataset can be generated from pure noise shared across clients, aiming to calibrate the features from the heterogeneous clients. Theoretically, we prove that VHL can achieve provable generalization performance on the natural distribution. Empirically, we demonstrate that VHL endows FL with drastically improved convergence speed and generalization performance. VHL is the first attempt towards using a virtual dataset to address data heterogeneity, offering new and effective means to FL.
\end{abstract}

%Theoretically, we prove that training with the virtual dataset can bound the \mbox{generalization} risk of the natural dataset. 

\section{Introduction}
% \vspace{-4pt}
% % 1. the success and challenges of FL
% success: what is FL, what FL does, why FL is important
Federated learning (FL)~\citep{mcmahan2017communication} has emerged as an important paradigm that enables decentralized data collection and model training. In the FL scenario, a large number of clients collaboratively train a shared model without sharing their natural (or private) data~\citep{kairouz2019advances}, enabling various real-world applications~\citep{keyboard,hsu2020federated,luo2019real}.
% 2. challenges and shortcoming caused by it
However, FL faces heterogeneity challenges caused by Non-IID data distributions from different clients as well as their diversity of computing and communication capabilities~\citep{fedprox,zhao2018federated,hsu2020federated}.
Severe heterogeneity can easily cause client drift~\citep{karimireddy2019scaffold}, leading to unstable convergence~\citep{li2019convergence} and poor model performance~\citep{wang2020tackling,zhao2018federated}.

% % 2. existing methods and shortcomings
To tackle federated heterogeneity, the first FL algorithm, FedAvg~\citep{mcmahan2017communication}, proposes to conduct more local computations and fewer communications during training. Though FedAvg addresses the diversity of computing and communication, the client drift induced by the Non-IID data distribution (data heterogeneity) has significant negative impact on FedAvg~\citep{fedprox,karimireddy2019scaffold,li2019convergence,wang2020tackling,luo2021no}. To address client drift, many efforts have been devoted to designing new learning paradigms, either on the client side, e.g., local training strategy~\citep{fedprox,karimireddy2019scaffold}, or on the server side, e.g., model aggregation strategy~\citep{aggregation,aggregation2}.
However, a recent benchmark work~\citep{crosssilononiid} shows that FedAvg can outperform its variants in many experimental settings. This reveals that it is challenging for a single method to address the client drift problem in multiple scenarios, reflecting the notoriety and subtlety of client drift. Therefore, addressing client drift remains a fundamental challenge of FL~\citep{kairouz2019advances}.

Another exciting direction is to directly ``rectify'' the cause of client drift, i.e., data heterogeneity. Specifically, sharing a small portion of private data~\citep{zhao2018federated} or the private statistical information~\citep{shin2020xor, yoon2021fedmix} can make data located in different clients more homogeneous. Nonetheless, the data sharing approach exposes FL to the danger of privacy leakage. Although differential privacy is a competitive candidate for avoiding privacy leakage, employing differential privacy can cause performance degradation~\citep{tramer2021differentially}. All these challenges motivate the following fundamental question:

\emph{Is it possible to defend against data heterogeneity in FL systems by sharing data containing \textbf{no} private information?}

In this work, we provide an affirmative answer to the question by proposing a novel approach, called virtual homogeneity learning (VHL). The key insight of VHL is that introducing more homogeneous data shared across clients can reduce the data heterogeneity~\citep{sharing1}. Specifically, VHL constructs a new rectified dataset for each client by sharing a virtual homogeneous dataset among all clients, which is independent of the natural datasets.

The key challenge of VHL is how to generate the virtual dataset to benefit model performance. Intuitively, combining an out-of-distribution dataset with the original data may sacrifice the generalization performance in different aspects such as distribution shift~\citep{da1,distributionshift}, noisy labels~\citep{noisy}, and garbage data~\citep{garbage}. In practice, it is challenging to sample data from the natural distribution for constructing a virtual dataset. Hence, introducing much virtual data drawn from a different distribution will cause the training distribution different from the test distribution, i.e., distribution shift, leading to poor generalization performance on the test set~\citep{da1,distributionshift}. Thus, the distribution shift is one crucial detrimental impact of introducing virtual datasets.

Fortunately, we can access both the labeled virtual dataset, i.e., source domain, and the labeled natural dataset, i.e., target domain, so that we can mitigate the distribution shift through the lens of domain adaptation (DA)~\citep{da1,distributionshift2,dann}. In particular, we propose matching the conditional distribution of the source and target domains. Our theoretical analysis shows that matching the virtual and natural distributions conditioned on label information can achieve provable generalization performance. Moreover, this can be instantiated by pulling natural and virtual data features from the same class together, as depicted in Figure~\ref{fig:FeatureCalibration-Theory}. To show the efficacy of VHL, we apply VHL in several popular FL algorithms including FedAvg, FedProx, SCAFFOLD, and FedNova on four datasets. Our experimental results show that VHL can boost the generalization ability and the convergence speed. 

Our main contributions include:
\begin{enumerate}
    \item
    We raise a fundamental question to explore how to reduce data heterogeneity in FL by sharing data that contains no private information.
    \item 
    To answer the question, we propose virtual homogeneity learning (VHL), making the first attempt to use virtual datasets to defend against data heterogeneity. To avoid the distribution shift induced by virtual datasets, VHL calibrates natural data features with that of virtual data, inspired by the rationale of domain adaptation.
    \item
    Through comprehensive experiments\footnote{The code is publicly available: \url{https://github.com/wizard1203/VHL}}, we demonstrate that VHL can drastically benefit the convergence speed and the generalization performance of FL models. 
\end{enumerate}

\begin{figure}[t!]
% \vspace{-2pt}
\small
% \subfigtopskip=2pt
    % \setlength{\belowdisplayskip}{2pt}
    % \setlength{\abovedisplayskip}{-5pt}
    % \subfigbottomskip=2pt
    % \subfigcapskip=-1pt
   \centering
    {\includegraphics[width=0.99\linewidth]{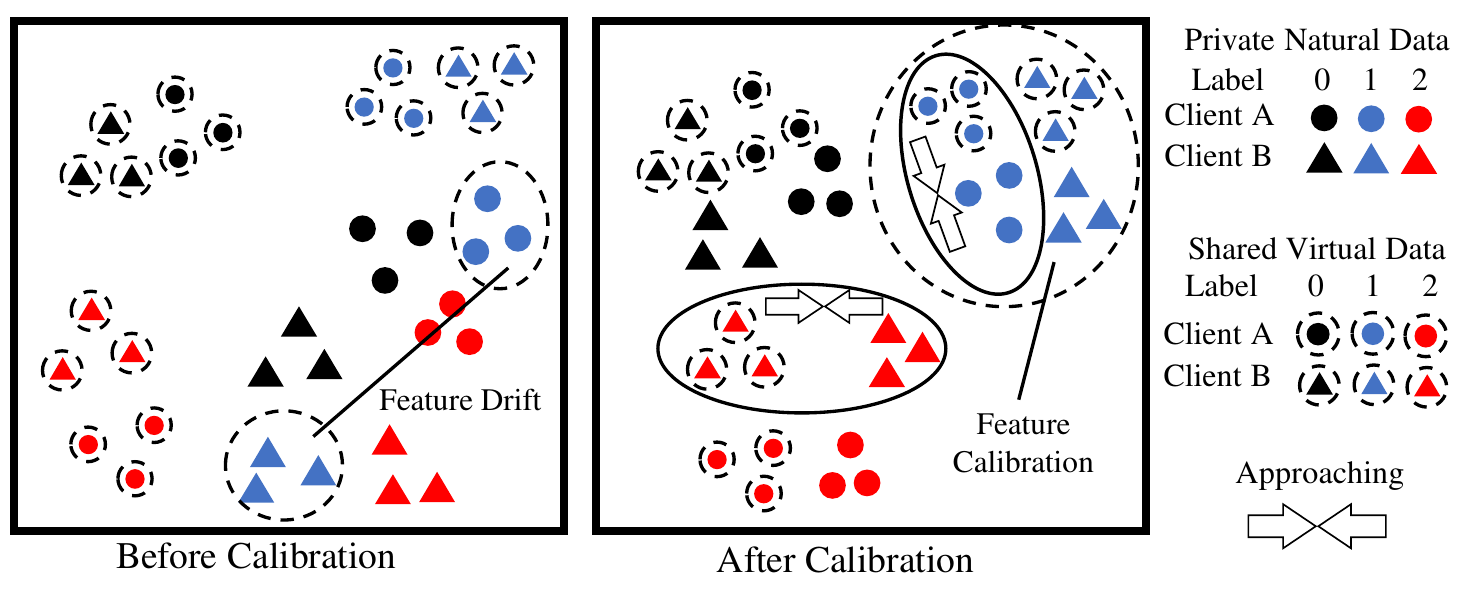}}
    \caption{Feature distribution of different clients. Circle and triangle points represent data on client A and B, respectively. Points surrounded by a dotted circle represent the virtual data. Different colors represent data of different labels. After local training (left), the features of private natural data of the same label on different clients are far away from others, but features of shared virtual data of the same label are close to others. Figure~\ref{fig:FeatureCalibration-Experiment} in section~\ref{sec:exp} shows the experimental observation of this Figure.}
    \label{fig:FeatureCalibration-Theory}
\vspace{-0.1cm}
\end{figure}

% \vspace{-4pt}
\section{Preliminaries}
% \vspace{-4pt}

%Before a detailed description of virtual homogeneity learning (VHL), we introduce preliminaries of this work. 
%VHL is proposed to facilitate \emph{federated learning} (FL) with better performance by introducing virtual datasets to defend against \emph{data heterogeneity} directly. To mitigate potential detrimental impacts of virtual datasets, i.e., distribution shift, we propose understanding and analyzing it through the lens of \emph{domain adaptation}.

\subsection{Federated Learning}
\textbf{Federated learning} aims to collaboratively train a global model parameterized by $w$ while omitting the need to access private data distributed over many clients~\citep{mcmahan2017communication, kairouz2019advances}. Formally, FL aims to minimize a global objective function $J(w)$:
\begin{equation} \label{fl_g}
    \min_{w} J(w) \triangleq \sum_{k=1}^{K} p_k J_{k}(w),
\end{equation}
where $K$ stands for the total number of clients, $p_k \geq 0$ is the weight of the $k$-th client such that $\sum_{k=1}^{K} p_k=1$, and $J_k(w)$ is the local objective defined as:
\begin{equation} \label{fl_l}
    J_{k}(w) \triangleq 
    \mathbb{E}_{(x,y) \sim \mathcal{P}_k(x,y)
    }
    \ell (f(x;w), y).
\end{equation}
Here, we denote $\mathcal{P}_k(x,y)$ as the joint distribution of data in client $k$, $\ell(\cdot, \cdot)$ as the loss function, e.g., cross-entropy loss, and $f =\phi \circ \rho$ as the prediction with $\phi$ being the feature extractor as well as $\rho$ the predictor, e.g., classifier for classification tasks.

A common optimization strategy proceeds through the communication between clients and the server round-by-round. In particular, the central server broadcasts the model $w^{r}$ to all selected clients $\mathcal{S}^{r}$ at round $r$ and then each client performs $E$ ($E \geq 1$) local updates, e.g., for the $k$-th client:
\begin{equation} \nonumber \label{l_ud}
    w_{k, j+1}^{r} \leftarrow w_{k, j}^{r} - \eta_{k, j} \nabla J_{k}(w_{k, j}^{r}), j = 0, 1, \cdots, E-1,
\end{equation}
where $w_{k, j}^{r}$ represents the $j$-th updates, i.e., $w_{k, 0}^{r} = w^{r}$, and $\eta_{k, j}$ is the learning rate. In the end of round $r$, all selected clients send the optimized parameter back to the server, and the central server performs a certain aggregation operation:
\begin{equation} \label{aggre1}
    w^{r+1} \leftarrow \text{AGGREGATE}(\cup_{k \in \mathcal{S}^{r}} w_{k, E-1}^{r}),
\end{equation}
where the $\text{AGGREGATE}$ operation is typically instantiated in a simple yet effective averaging manner~\citep{mcmahan2017communication}:
\begin{equation} \nonumber \label{avg}
    w^{r+1} \leftarrow \sum_{k \in \mathcal{S}^{r}} p_{k} w_{k, E-1}^{r},
    \ 
    p_{k} = \frac{n_{k}}{\sum_{i \in \mathcal{S}^{r}} n_{i}},
\end{equation}
with $n_k$ being the number of samples on the $k$-th client.

% and the central server performs a simple yet effective averaging operation to obtain new global model~\citep{mcmahan2017communication}:
% \begin{equation} \nonumber \label{avg}
%     w^{r+1} \leftarrow \sum_{k \in \mathcal{S}^{r}} p_{k} w_{k, E-1}^{r},
%     \ 
%     p_{k} = \frac{n_{k}}{\sum_{i \in \mathcal{S}^{r}} n_{i}},
% \end{equation}
% with $n_k$ being the number of samples on the $k$-th client.

% \textbf{Data heterogeneity} is one of the fundamental challenges in FL~\citep{kairouz2019advances}. Specifically, the distribution of data in the $k$-th client $\mathcal{P}_{k}$ may be significantly different from that of any other clients, leading to inconsistent objective~\citep{wang2020tackling} and unstable convergence~\citep{karimireddy2019scaffold}. Theoretically, according to~\citep{li2019convergence}, the data heterogeneity can be defined as:
% \begin{definition}\label{def:NonIIDbyOptimum}
% (Quantifying the degree of Data heterogeneity~\citep{li2019convergence}.)
%     Let $J^\star $ and $J_k^\star$ be the minimum values of $J$ (in Eq.~\ref{fl_g}) and $J_k$ (in Eq.~\ref{fl_l}), respectively. Then, the data heterogeneity is defined as $\Gamma = J^\star - \sum_{k=1}^K p_k J_k^\star, \text{and } \ p_k=\frac{n_k}{\sum_{i}n_i}$. If $\Gamma$ gets larger, it means the higher degree of Non-IID On the other hand, if $\Gamma$ goes to zero, then the data is i.i.d.
% \end{definition}
% We will use this definition throughput the paper.

% \vspace{-4pt}
\subsection{Domain Adaptation}
% \vspace{-4pt}
\textbf{Domain adaptation} approaches aim to solve the problem where the training and test data are drawn from different distributions~\citep{da1}. To perform domain adaptation, three assumptions are imposed on how distributions change across domains, i.e.., distribution shifts stem from: $\mathcal{P}(x)$, i.e., covariate shift~\citep{distributionshift,px}, $\mathcal{P}(y|x)$, i.e., conditional shift~\citep{conditionalshift,mingming}, or both, i.e., dataset shift~\citep{datashift,condition}.

One widely adopted strategy is to minimize the distribution discrepancy between the source distribution $\mathcal{P}_S$ and the target distribution $\mathcal{P}_T$ by finding a feature extractor $\phi$, such that the source and target domains have similar distributions in the learned feature space~\citep{dann,conditionalshift2,joint}. For the covariate shift assumption, the marginal distribution discrepancy is minimized for reducing generalization risk on the target domain~\citep{dann}:
\begin{equation} \nonumber
    d
    \left(\mathcal{P}_S\left(\phi\left(x\right)\right), \mathcal{P}_T\left(\phi\left(x\right)\right)
    \right),
\end{equation}
where $d$ is a measure of distribution discrepancy.
%, $\phi_S = \phi_S(x_S), \phi_T = \phi_T(x_T)$ stand for representations of variables $x_S \sim \mathcal{P}_S(x), x_T \sim \mathcal{P}_T(x)$. 
For the conditional shift assumption, the distribution discrepancy is usually conditioned on the label~\citep{mingming}:
\begin{equation}  \nonumber
    d\left(
    \mathcal{P}_S\left(\phi\left(x\right)|y\right), 
    \mathcal{P}_T\left(\phi\left(x\right)|y\right)
    \right),
\end{equation}
where $y$ denotes the label. For the most challenging dataset shift, the joint distribution discrepancy is usually employed:
\begin{equation}  \nonumber
    d\left(
    \mathcal{P}_S\left(\phi\left(x\right),y\right), 
    \mathcal{P}_T\left(\phi\left(x\right),y\right)
    \right).
\end{equation}

In this paper, we regard the virtual distribution as the source domain and the natural distribution as the target domain.

% \vspace{-4pt}
\section{Virtual Homogeneity Learning}
% \vspace{-4pt}
This section gives a detailed description of how virtual heterogeneity learning (VHL) promotes federated learning (FL) by mitigating data heterogeneity with a virtual dataset. 

% \vspace{-4pt}
\subsection{Heterogeneity Mitigation}
% \vspace{-4pt}
Clients in FL systems usually collect data independently, where the data distribution can vary with clients~\citep{kairouz2019advances}. Hence, data heterogeneity stems from the discrepancy in the data held by the clients. A straightforward solution to mitigate data heterogeneity is to send an extra common dataset to each client and minimize the local objective with both the local data and the shared data.

Unfortunately, the underlying assumption of data sharing is that these extra data are drwan from the natural distribution, because using arbitrary noisy or garbage data to train models produces poor generalization performance~\citep{noisy,garbage}. However, it is not realistic due to the privacy concern. That is, collecting and sharing data from the natural distribution are challenging and will expose FL to the danger of privacy leakage.

To address the challenge, we propose introducing a virtual dataset to reduce data heterogeneity such that the shared virtual data contains no private information and avoids the performance degradation of FL models on the natural dataset. Specifically, we first introduce a dataset collected independently of the natural data to bypass the private information concerns. Then, we expect that models learned from the virtual dataset can perform well on the natural distribution.

\subsection{Feature Calibration}
The key to promoting FL with a virtual dataset is calibrating the natural feature distribution and the virtual feature distribution. It is effortless to collect data independent of the natural data. For example, the virtual dataset can be generated from pure noise, e.g., Gaussian noise and structural noise drawn from untrained style-GAN~\citep{stylegan}. Thus, the critical challenge is how to ensure that the knowledge learned from the virtual dataset can be transferred to the natural dataset. Drawing inspiration from domain adaptation, we consider the virtual dataset as the source domain and the natural dataset as the target domain and perform domain adaptation to reduce the generalization risk of natural distribution. Built upon the rationale of domain adaptation, we thus propose calibrating these two distributions in the feature space.

The virtual data are crafted independent of the natural data, leading to dataset shift~\citep{datashift,condition} rather than covariate shift~\citep{distributionshift} or conditional shift~\citep{conditionalshift}. Note that, the virtual data (domain) is different from the heterogeneous domain studied in~\citep{feng1}, because the virtual dataset is arbitrarily constructed, e.g., pure noise. To tackle the challenging dataset shift, advanced methods~\citep{condition,joint,joint2} follow the rational of domain adaptation to match the joint distribution. Inspired by these approaches, we propose matching the virtual and natural distributions conditioned on the label information to mitigate the adverse impact induced by distribution shift. Thanks to the flexibility of constructing virtual dataset, we can collect virtual data such that the virtual and natural distribution have the same label distributions. Consequently, the joint distribution discrepancy minimization is reduced to the problem of conditional distribution minimization. Specifically, we introduce a conditional distribution mismatch penalty to alleviating the dataset shift:
\begin{equation} \label{fl_l_bar}
 \begin{split} 
    \mathbb{E}_{
	(x,y) \sim \mathcal{P}_k}
	\ell & \left(\rho \circ \phi \left(x\right), y\right) + 
	\mathbb{E}_{
	(x,y) \sim \mathcal{P}_v
	}
	\ell \left(\rho \circ \phi \left(x\right), y\right)
     \\ 
    + \ 
     & \lambda  \mathbb{E}_{y}
     d
     \left(
     \mathcal{P}_k\left(\phi\left(x\right)|y\right), 
     \mathcal{P}_v\left(\phi\left(x\right)|y\right)
     \right),
\end{split}   
\end{equation}
where $\mathcal{P}_v$ denotes the virtual distribution, e.g., a Gaussian distribution, $\rho \circ \phi=f$ stands for the local model, $\mathcal{P}_k\left(\phi\left(x\right)|y\right)$ is the feature distribution obtained by applying the feature extractor $\phi$ to a random variable $x \sim \mathcal{P}_k(x|y)$ given label $y$, and $\lambda$ is a hyperparameter. Similarly, $\mathcal{P}_v\left(\phi\left(x\right)|y\right)$ is the feature distribution obtained on $x \sim \mathcal{P}_v(x|y)$. The penalty $d
     \left(
     \mathcal{P}_k\left(\phi\left(x\right)|y\right), 
     \mathcal{P}_v\left(\phi\left(x\right)|y\right)
     \right)$ 
encourages the feature extractor to make the data sampled from different distributions, i.e., $\mathcal{P}_k$ and $\mathcal{P}_v$, having the same conditional feature distribution.

\subsection{Analysis}
The insight of VHL is depicted in Figure~\ref{fig:FeatureCalibration-Theory}. After local training of each round, the features of natural data in each category on different clients are far away from others, but the features of virtual data in each category are relatively close to others. Thus, pulling the natural feature towards the virtual feature with the same label can contribute to local models having similar features for each category, which is consistent with our empirical observation shown in Figure~\ref{fig:FeatureCalibration-Experiment}.

To provide a theoretical justification on the effectiveness of the proposed local objective function with distribution mismatch penalty, we analyze the relationship between the generalization performance and the misalignment between domains theoretically for classification tasks. To the best of our knowledge, we are the first to study the relationship for classification tasks. For example, the theoretical conclusions derived in~\citep{joint2} are mainly focused on regression problems and only empirical supports are given in~\citep{condition,joint}.

In light of the margin theory~\citep{margin2} that maximizing the margin between data points and the decision boundary achieves strong generalization performance, we first introduce the definition of \emph{statistical margin} \footnote{The margin definition is used in~\citep{margin3,margin4,margin5} for studying the generalization in the presence of Gaussian noise and adversarial noise.} to measure the generalization performance before stating the main theorem.
\begin{definition}\label{def:margin}
(Statistical margin.)
    We define statistical margin for a classifier $f$ on a distribution $\mathcal{P}$ with a distance metric $m$: $\textit{SM}_m(f,\mathcal{P})=\mathbb{E}_{(x,y) \sim \mathcal{P}} \inf_{f(x') \neq y} m(x',x)$.
\end{definition}

The defined statistical margin quantifies the degree of generalization performance. Hence, we can use it to quantify models' generalization performance on the target distribution, where the models are learned from the virtual dataset. In particular, the problem of achieving strong generalization performance is rephrased to maximizing the statistical margin:
\begin{equation} \label{goal}
    \mathbb{E}_{f:=\textit{VHL}(\mathcal{P}_v(x,y))} \textit{SM}_m(f,\mathcal{P}(x,y)),
\end{equation}
where $f:=\textit{VHL}(\mathcal{P}_v(x,y))$ means that the model $f$ is learned on the distribution $\mathcal{P}_v(x,y)$ using Eq.~\ref{fl_l_bar}, $\textit{SM}_m(f,\mathcal{P}(x,y))$ represents that the generalization performance of $f$ is evaluated on the distribution $\mathcal{P}(x,y)$.

Built upon Definition~\ref{def:margin} and Eq.~\ref{goal}, we are ready to state the following theorem (proof can be found in Appendix~\ref{appendix:Proof}).
\begin{theorem} \label{theo}
Let $f=\phi \circ \rho$ be a neural network decompose of a feature extractor $\phi$ and a classifier $\rho$, $\mathcal{P}(x,y)$ and $\mathcal{P}_v(x,y)$ are two separable distributions with identical label distributions. Then, learning $f$ with the proposed distribution mismatch penalty, i.e., Eq.~\ref{fl_l_bar} elicits an model with bounded statistical margin, i.e., Definition~\ref{def:margin}.
\end{theorem}
\begin{remark}
Theorem~\ref{theo} shows that models learned with Eq.~\ref{fl_l_bar} can bound the statistical margin, learning to provable generalization performance. Moreover, larger statistical margin of the virtual distribution $\mathcal{P}_v$ will further provides a tighter bound of the statistical margin of the natural distribution, according to the proof of Theorem~\ref{theo} (details can be found in Appendix~\ref{appendix:Proof}).
\end{remark}

Theorem~\ref{theo} matches the underlying insights of previous effective methods~\citep{condition,joint} that aligning the joint distribution between the source and target domains can guarantee the generalization performance on target domains.

% \vspace{-4pt}
\subsection{Overview}
% \vspace{-4pt}
To perform VHL, we first craft a virtual dataset independent of the natural dataset and being separable. The conclusion is drawn from Theorem~\ref{theo}. Details can be found in Appendix~\ref{appendix:Proof}. In particular, we use an \emph{untrained} style-GAN to generate noise for each category. Visualization can be found in Appendix~\ref{appendix:GenerateData}. We also explore other methods (i.e., with a simple generative model and up-sampling with pure Gaussian noises) to generate virtual datasets in Sec.~\ref{sec:AblationStudy}. Then, we train local models with the local objective function Eq.~\ref{fl_l_bar}.

Algorithm~\ref{algo:VHL} summarizes the training procedure of applying VHL to FedAvg, highlighting modifications to FedAvg. The server works almost the same as FedAvg, except that it sends the virtual dataset to all selected clients at the beginning. Different from FedAvg, clients optimize its local model based on both natural data and virtual data. Since only the utilized training data and the objective function are modified, other federated learning algorithms, e.g., FedProx, SCAFFOLD, could be effortlessly combined with VHL.

\begin{algorithm}[!t]
% 	\caption{VHL: Virtual Homogeneity Learning}
	\caption{FedAvg with \textbf{VHL} }
	\label{algo:VHL}
% 	\small
	\textbf{server input: } initial ${w}^0$,  maximum communication round $ R  $ \\
	\textbf{client $k$'s input: } local epochs $ E$
	\begin{algorithmic}
% 		\small
		\STATE {\bfseries Initialization:} server distributes the initial model ${w}^0$ to all clients, 
		\colorbox{green!20}{as well as the virtual dataset $\tilde{D}$}.
        \STATE
        \STATE \textbf{Server\_Executes:}
            \FOR{each round   $ r=0,1, \cdots, R$}
        		\STATE server samples a set of clients $\mathcal{S}_r \subseteq \left \{1, ..., K \right \}$
                \STATE server \textbf{communicates} ${w}_r$ to all clients $k \in \mathcal{S}$
        		\FOR{ each client $ k \in \mathcal{S}^{r}$ \textbf{ in parallel do}}
                       \STATE $\small {w}_{k, E-1}^{r+1} \leftarrow \text{ClientUpdate}(k, {w}^{r})$ 
        		\ENDFOR
             \STATE $\small {w}^{r+1} \leftarrow \sum_{k=1}^K p_{k} {w}_{k,E-1}^{r}$
        	\ENDFOR
       \STATE
        \STATE \textbf{Client\_Training($\small k, {w}$):}
            \FOR{each local epoch $j$ with $j=0,\cdots, E-1 $}
              \STATE \colorbox{green!20}{$\small {w}_{k,j+1} \leftarrow {w}_{k,j} - \eta_{k,j} \nabla_{{w}}\bar{J}_{k}(w) $, i.e., Eq.~\ref{fl_l_bar}}
            \ENDFOR
        \STATE \textbf{Return} $w$ to server
        % \RETURN{ $w$ to server}
\end{algorithmic}
\vspace{-0.1cm}
\end{algorithm}
\vspace{-0.1cm}

\begin{figure*}[!t]
% \small
    \setlength{\abovedisplayskip}{-1pt}
    \subfigbottomskip=2pt
    \subfigcapskip=1pt
    \setlength{\abovecaptionskip}{0.cm}
  \centering
     \subfigure[FedAvg without VHL]{\includegraphics[height=0.27\textwidth,width=0.3\textwidth]{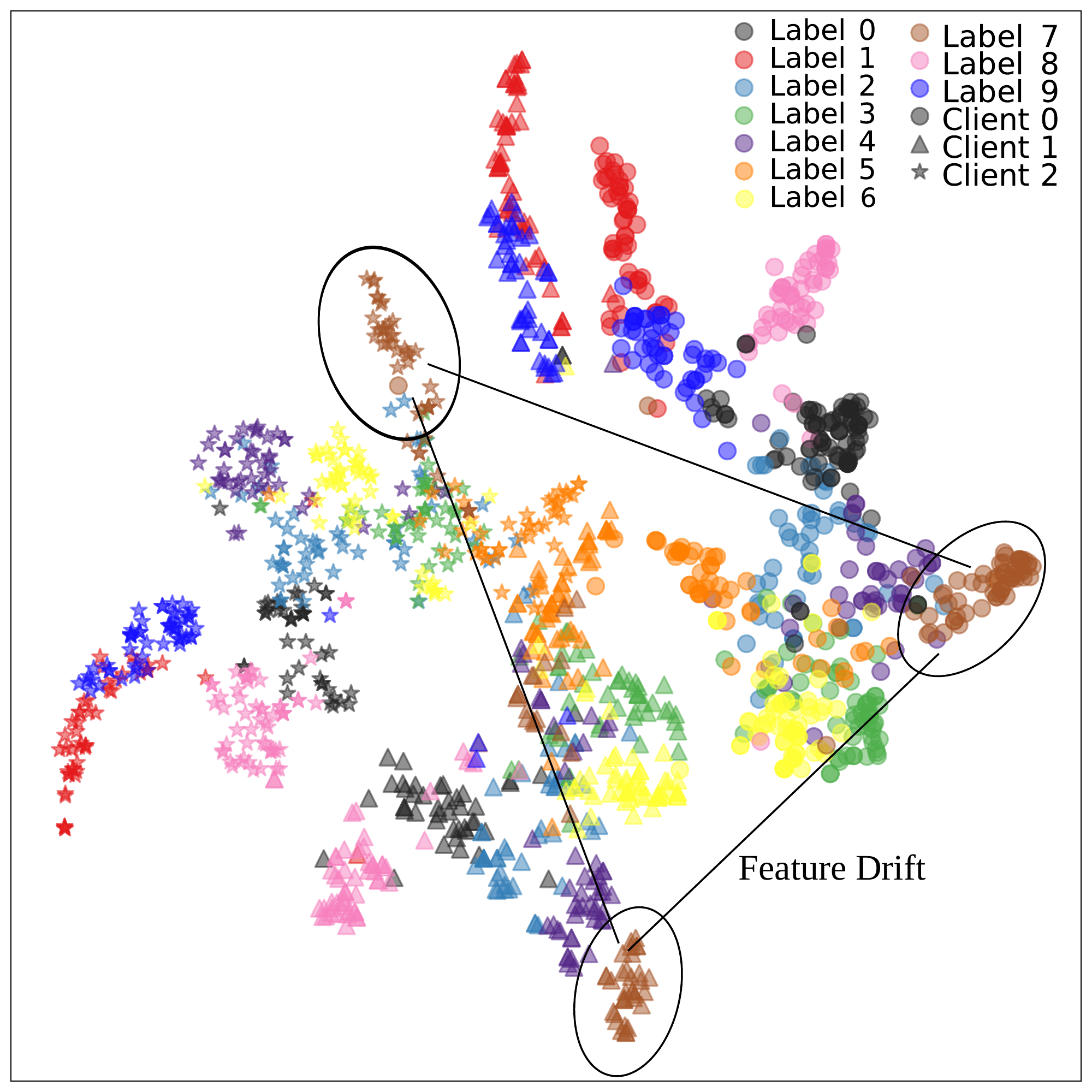}}
     \subfigure[FedAvg with Naive VHL]{\includegraphics[height=0.27\textwidth,width=0.3\textwidth]{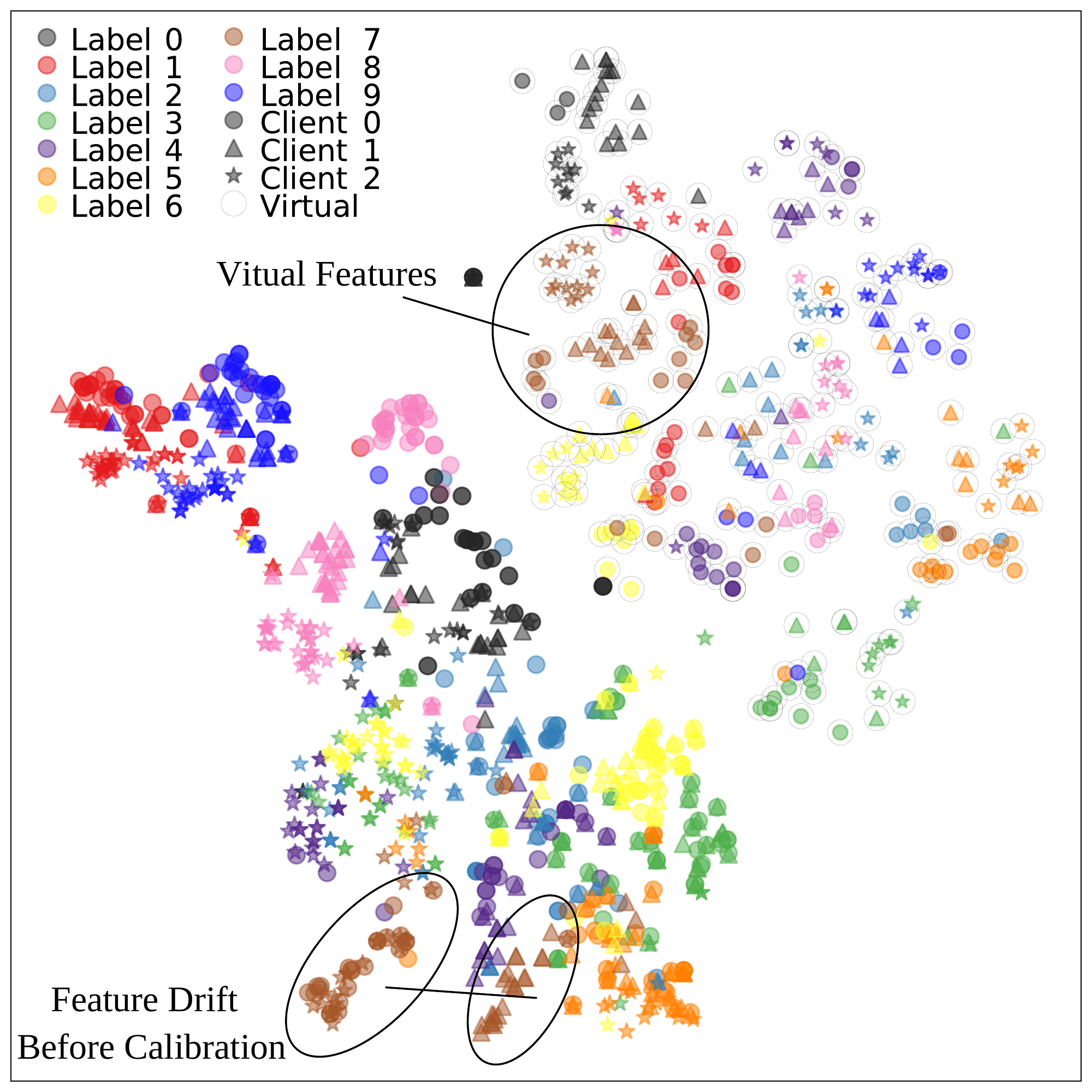}}
    \subfigure[FedAvg with VHL]{\includegraphics[height=0.27\textwidth,width=0.3\textwidth]{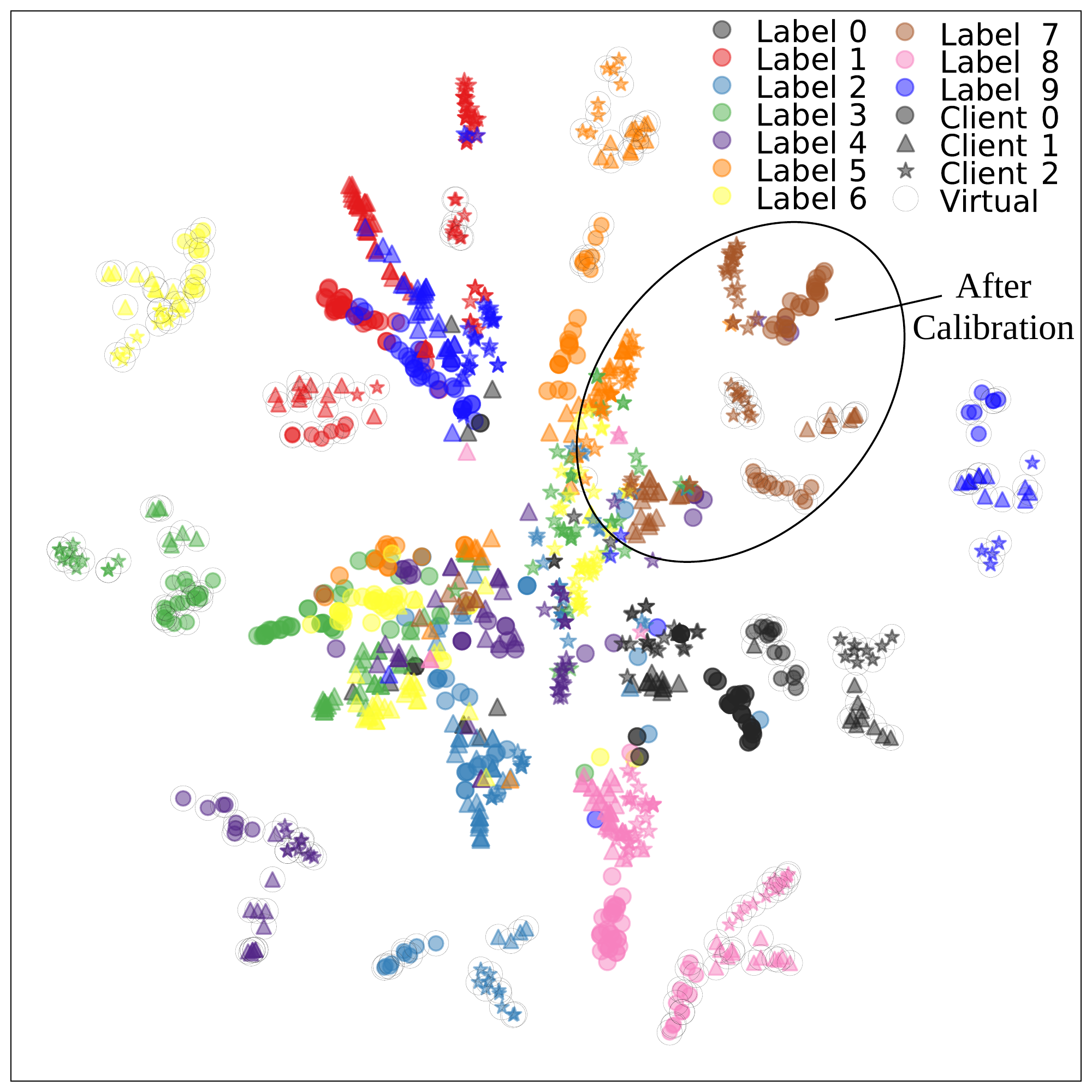}}
\caption{The t-SNE Visualization of empirical feature distribution of 3 different clients training with FedAvg, at $299$-th communication round. Different colors represent different classes of data, different shapes represent different clients, and the dotted circle means the virtual data. The \textbf{Naive VHL} means training with both the private natural and shared virtual data without feature calibration.}
\label{fig:FeatureCalibration-Experiment}
% \vspace{-0.0cm}
\end{figure*}

% \vspace{-4pt}
\section{Related Work}
% \vspace{-4pt}
We merely introduce the most related works in this section due to the limited space and leave more detailed discussion and literature review in Appendix~\ref{appendix:MoreRelated}.

\textbf{Mitigating Client Drift.}
A series of works focus on adding regularization to calibrate the optimization direction of local models, restricting local models from being too far away from the server model, including FedProx~\citep{fedprox}, FedDyn~\citep{acar2021federated}, SCAFFOLD~\citep{karimireddy2019scaffold}, MOON~\citep{li2021model}, and FedIR~\citep{hsu2020federated}. From the optimization point of view, some methods propose to correct the updates from clients, accelerating and stabilizing the convergence, like
FedNova~\citep{wang2020tackling}, FedAvgM~\citep{LDA}, 
FedAdaGrad, FedYogi, and FedAdam~\citep{reddi2020adaptive}.

\textbf{Sharing Data.}
Some works propose to utilize GANs~\citep{sharing1, long2021g} or adversarial learning~\citep{FLviaSyntheticData} to generate synthetic data based on the information of the raw data. Nonetheless, all above methods could expose the information of the raw data at some degree, i.e. intermediate features, statistic information, raw data.

The statistical information of the raw data~\citep{yoon2021fedmix,shin2020xor} is also employed to learn and share synthetic data.
The profiles of the raw images could be still seen from the synthetic images, exposing the raw data information at a degree. Thus, the intermedia activations~\citep{hao2021towards} or logits~\citep{li2019fedmd,chang2019cronus,luo2021no} of the raw data are exploited to enhance FL. However, these methods may expose the accurate high-level semantic information of raw data, making clients under the risks of exposing the private data. Specifically, the raw data may be reconstructed by feature inversion methods~\citep{zhao2021what}. Table~\ref{tab:demystifySharingData} in Appendix~\ref{appendix:MoreRelated} summarizes these sharing data approach, but we omitted the comparison with works because of the above privacy concerns .

% The statistical information of the raw data is also employed to learn and share synthetic data.
% %And some works utilize the statistic of the raw data information to learn and share synthetic data. 
% The profiles of the raw images could be still seen from the synthetic images, exposing the raw data information at a degree. XorMixFL~\citep{shin2020xor} proposes to collect other clients' XOR-encoded data samples that are decoded only using each client's own data samples, also exposing raw data information to other clients. FedMIX~\citep{yoon2021fedmix} proposes to let clients send and receive the averaged local data of each batch. Strictly speaking, this method exposes the statistic information of local data, i.e. the mean values.

% Some works propose to share the intermedia activations~\citep{hao2021towards} or logits~\citep{li2019fedmd,chang2019cronus,luo2021no} to enhance FL. However, these methods may expose the accurate high-level semantic information of raw data, making clients under the risks of exposing the private data. Specifically, the raw data may be reconstructed by feature inversion methods~\citep{zhao2021what}. Table~\ref{tab:demystifySharingData} in Appendix~\ref{appendix:MoreRelated} summarizes these sharing data approach, but we omitted the comparison with works because of the mentioned risk of exposing the raw data.

% \vspace{-4pt}
\section{Experimental Studies}\label{sec:exp}
% \vspace{-4pt}

\subsection{Experiment Setup}
% \vspace{-4pt}

\textbf{Federated Datasets and Models.}
To verify the effectiveness of VHL, we exploit a popular FedML framework~\citep{chaoyanghe2020fedml} to conduct experiments over various datasets including CIFAR-10~\citep{Krizhevsky09learningmultiple}, FMNIST~\citep{xiao2017fashion}, SVHN~\citep{SVHN}, and CIFAR-100~\citep{Krizhevsky09learningmultiple}. We use ResNet-18~\citep{he2016deep} for CIFAR-10, FMNIST and SVHN, and ResNet-50 for CIFAR-100. To emulate data distribution in FL, we partition the datasets with a commonly used Non-IID partition method, Latent Dirichlet Sampling~\citep{LDA}. Following existing works~\citep{li2021model, luo2021no}, each dataset is partitioned with two different Non-IID degrees using $a = 0.1$ and $a = 0.05$. 
%Furthermore, we also consider the impact of the number of local epochs and the number of clients, expanding the application scope of VHL.

Moreover, to verify the effect of VHL on other kinds of Non-IID data distributions, we conduct experiments on CIFAR-10 with 10 clients using another two kinds of Non-IID partition methods: (1) 2-classes partition~\citep{mcmahan2017communication}: each client only has 500 data samples of 2 classes; (2) subset partition~\citep{zhao2018federated}: all clients have data of all 10 classes, but each client has one dominant class with 4950 samples, and each of the remaining classes has 5 or 6 samples, which is an extremely imbalanced distribution.

% \subsection{Other Non-IID Partition Methods}\label{appendix:MoreFacets}

\textbf{Baselines.}
We conduct FedAvg~\citep{mcmahan2017communication} and recent effective FL algorithms that are proposed to address the client drift problem, including FedProx~\citep{fedprox}, SCAFFOLD~\citep{karimireddy2019scaffold}, and FedNova~\citep{wang2020tackling}, with or without VHL. The detailed hyper-parameters of each algorithm in each setting are reported in Appendix~\ref{appendix:DetailedExp}.

% We compare our methods with FedAvg~\citep{mcmahan2017communication} and recent effective FL algorithms that are proposed to address the client drift problem, including FedProx~\citep{fedprox}, SCAFFOLD~\citep{karimireddy2019scaffold}, and FedNova~\citep{wang2020tackling}. We first run experiments on 10-client and 100-client settings with the four training algorithms, which try to align the results with the original papers. Then we apply our VHL into these four algorithms to show how VHL significantly improves the performance in different FL algorithms. The detailed hyper-parameters of each algorithm and each setting are reported in Appendix ~\ref{appendix:DetailedExp}.

\textbf{Generating Virtual Datasets.}
We use an \emph{un-pretrained} StyleGAN-v2~\citep{Karras2019stylegan2} without using any training data to generate the virtual dataset. Specifically, the server uses the StyleGAN to generate images with different latent styles and noises which are sampled from different Gaussian Distributions. We show these virual noise in Figure~\ref{fig:NoiseStyleGan32x32c10} and ~\ref{fig:NoiseStyleGan32x32c100} in Appendix~\ref{appendix:GenerateData}. The generated virtual images are then distributed to all clients at the beginning of training. We also explore other methods (i.e., generated with a simple generative model and up-sampling with pure Gaussian noises) to generate virtual datasets in section~\ref{sec:AblationStudy}. More generation details of virtual datasets are listed in Appendix~\ref{appendix:GenerateData}.

\begin{table}[t]
\centering
\caption{Results with/without VHL on CIFAR-10.} 
% \vspace{1pt}
\small{
\begin{tabular}{c|ccc}
\toprule[1.5pt]
                  & ACC $\uparrow$ & ROUND  $\downarrow$      & Speedup $\uparrow$      \\
                  \cline{2-4} 
\multirow{-2}{*}{} & \multicolumn{3}{c}{w/ (w/o) \textbf{VHL}} \\
\midrule[1pt]
\multicolumn{4}{c}{centralized training ACC = \textbf{92.88\%} (92.53\%) } \\
\midrule[1pt]
\multicolumn{4}{c}{\cellcolor{greyL} $a=0.1, E=1, K=10$ ( Target ACC = 79\% )} \\
\midrule[1pt]
FedAvg             & \textbf{87.82} (79.98) & \textbf{128} (287) & \textbf{$\times$ 2.2} ($\times$ 1.0) \\
FedProx             & \textbf{87.30} (83.56) & \textbf{128} (188) & \textbf{$\times$ 2.2} ($\times$ 1.5) \\
SCAFFOLD             & \textbf{84.87} (83.58) & \textbf{90} (291) & \textbf{$\times$3.2} ($\times$ 1.0) \\
FedNova             & \textbf{87.56} (81.35) & \textbf{128} (351) & \textbf{$\times$ 2.2} ($\times$ 0.8)\\
\midrule[1pt]
\multicolumn{4}{c}{\cellcolor{greyL} $a=0.05, E=1, K=10$  ( Target ACC =  69\% )} \\
\midrule[1pt]
FedAvg             & \textbf{79.23} (69.02) & \textbf{112} ( 411 ) & \textbf{$\times$ 3.7} ($\times$ 1.0)  \\
FedProx             & \textbf{80.84} (78.66) & \textbf{151} ( 201 ) & \textbf{$\times$ 2.7} ($\times$ 2.0)  \\
SCAFFOLD             & \textbf{55.73} (38.55) & Nan (Nan) & \textbf{ Nan } (Nan)  \\
FedNova             & \textbf{80.59} (64.78) & \textbf{247} (Nan) & \textbf{$\times$ 1.66} (Nan)\\
\midrule[1pt]
\multicolumn{4}{c}{\cellcolor{greyL} $a=0.1, E=5, K=10$  ( Target ACC =  84\% )} \\
\midrule[1pt]
FedAvg             & \textbf{89.93} (84.79) & \textbf{91} (255
) & \textbf{$\times$ 2.8} ($\times$ 1.0)  \\
FedProx             & \textbf{86.41} (82,18) & \textbf{255} (Nan) & \textbf{$\times$ 1.0} (Nan)  \\
SCAFFOLD             & \textbf{87.27} (86.20) & \textbf{45} (66) & \textbf{$\times$ 5.7} ($\times$ 2.0) \\
FedNova             & \textbf{90.24} (86.09) & \textbf{67} (127) & \textbf{$\times$ 3.8} ($\times$ 1.0) \\
\midrule[1pt]
\multicolumn{4}{c}{\cellcolor{greyL}  $a=0.1, E=1, K=100$  ( Target ACC =  49\% )} \\
\midrule[1pt]
FedAvg             & \textbf{70.20} (49.61) & \textbf{385} (957) & \textbf{$\times$ 2.5} ($\times$ 1.0)  \\
FedProx             & \textbf{73.90} (49.97) & \textbf{325} (842) & \textbf{$\times$ 2.9} ($\times$ 1.1)  \\
SCAFFOLD             & \textbf{59.66} (52.24) & \textbf{479} (664) & \textbf{$\times$ 2.0} ($\times$ 1.4)  \\
FedNova             & \textbf{61.59} (46.53) & \textbf{ 554} (Nan) & \textbf{$\times$ 1.7} (Nan)  \\
\bottomrule[1.5pt] 
\end{tabular}
}
\begin{tablenotes}
    % \small
	\item  ``ROUND'' represents the communication rounds that need to attain the target accuracy. The notion $\downarrow$ ($\uparrow$) indicates smaller (larger) values are preferred. ``Nan'' means that the target accuracy is never attained during the whole training process.
\end{tablenotes}
\vspace{-0.3cm}
\label{tab:MainResult-cifar10}
\end{table}

\begin{table}[t]
\centering
% \caption{The best test accuracy and communication round to attain the target test accuracy with/without VHL on FMNIST.}
\caption{Results with/without VHL on FMNIST.} 
% \vspace{1pt}
\small{
\begin{tabular}{c|ccc}
\toprule[1.5pt]
                  & ACC $\uparrow$ & ROUND $\downarrow$       & Speedup  $\uparrow$      \\
                  \cline{2-4} 
\multirow{-2}{*}{} & \multicolumn{3}{c}{w/ (w/o) \textbf{VHL}} \\
\midrule[1pt]
\multicolumn{4}{c}{centralized training ACC = \textbf{93.80\%} (93.70\%) } \\
\midrule[1pt]
\multicolumn{4}{c}{\cellcolor{greyL} $a=0.1, E=1, K=10$ ( Target ACC = 86\% )} \\
\midrule[1pt]
FedAvg             & \textbf{92.05} (86.81) & \textbf{52} (119) & \textbf{$\times$ 2.3} ($\times$ 1.0)  \\
FedProx             & \textbf{90.68} (87.12) & \textbf{31} (135) & \textbf{$\times$ 3.8} ($\times$ 0.9)  \\
SCAFFOLD             & \textbf{90.27} (86.21) & \textbf{14} (143) & \textbf{$\times$ 8.5} ($\times$ 0.8)  \\
FedNova             & \textbf{91.88} (86.99) & \textbf{52} (83) & \textbf{$\times$ 2.3} ($\times$ 1.4) \\
\midrule[1pt]
\multicolumn{4}{c}{\cellcolor{greyL} $a=0.05, E=1, K=10$  ( Target ACC =  78\% )} \\
\midrule[1pt]
FedAvg             & \textbf{89.06} (78.57) & \textbf{53} (425) & \textbf{$\times$ 8.0} ($\times$ 1.0)  \\
FedProx             & \textbf{87.76} (81.96) & \textbf{30} (41) & \textbf{$\times$ 14.2} ($\times$ 10.4)  \\
SCAFFOLD             & \textbf{80.68} (76.08) & \textbf{58} (Nan) & \textbf{$\times$ 7.3 } (Nan)  \\
FedNova             & \textbf{87.25} (79.06) & \textbf{30} (538) & \textbf{$\times$ 14.2} ($\times$ 0.8) \\
\midrule[1pt]
\multicolumn{4}{c}{\cellcolor{greyL} $a=0.1, E=5, K=10$  ( Target ACC =  87\% )} \\
\midrule[1pt]
FedAvg             & \textbf{91.52} (87.45) & \textbf{51} (278) & \textbf{$\times$ 5.5} ($\times$ 1.0 )  \\
FedProx             & \textbf{88.27} (86.07) & \textbf{74} (Nan) & \textbf{$\times$ 3.8} (Nan)  \\
SCAFFOLD             & \textbf{91.82} (87.10) & \textbf{20} (105) & \textbf{$\times$ 13.9} ($\times$ 2.7)  \\
FedNova             & \textbf{91.86} (87.53) & \textbf{51} (193) & \textbf{$\times$ 5.5} ($\times$ 1.4) \\
\midrule[1pt]
\multicolumn{4}{c}{\cellcolor{greyL}  $a=0.1, E=1, K=100$  ( Target ACC =  90\% )} \\
\midrule[1pt]
FedAvg             & \textbf{91.14} (90.11) & \textbf{436} (658) & \textbf{$\times$ 1.5} ($\times$ 1.0)  \\
FedProx             & \textbf{91.37} (90.71) & \textbf{283} (491) & \textbf{$\times$ 2.3} ($\times$ 1.3)  \\
SCAFFOLD             & \textbf{87.91} (85.99) & Nan (Nan) & Nan (Nan)  \\
FedNova             & \textbf{88.34} (87.09) & Nan (Nan) & Nan (Nan) \\
\bottomrule[1.5pt] 
\end{tabular}
}
\vspace{-0.3cm}
\label{tab:MainResult-fmnist}
\end{table}

% \vspace{-4pt}
\subsection{Experimental Results}
% \vspace{-4pt}

\textbf{Main Results.} We use two metrics, the best accuracy during training and the number of communication rounds to achieve target accuracy, to compare the performance of different algorithms. The target accuracy is set as the best accuracy of FedAvg. The results on CIFAR-10, FMNIST, SVHN, and CIFAR-100 are shown in Tables~\ref{tab:MainResult-cifar10}, ~\ref{tab:MainResult-fmnist}, ~\ref{tab:MainResult-SVHN}, and ~\ref{tab:MainResult-cifar100}, respectively. We also compare the convergence speed in Figure ~\ref{fig:Convergence-MainResults} (a) on CIFAR-10\footnote{We leave more convergence figures in appendix~\ref{appendix:VisualizationFeature}}. The results show that VHL not only improves the generalization ability, but also accelerates the convergence. We also measure the client drift~\citep{karimireddy2019scaffold}, $ \frac{1}{|\mathcal{S}^r|}\sum_{i \in \mathcal{S}^r}  \| \bar{w} - w_i\|$, in Figure~\ref{fig:Convergence-MainResults} (b), where we report the client drift for the first 500 rounds as the early stage of training has obvious client drift problems. Note that we exclude FedNova from comparison as its client drift is very severe. The Figure~\ref{fig:Convergence-MainResults} (b) show that VHL could greatly alleviate client drift, thus accelerating convergence.

\textbf{Impacts of Non-IID Degree.} As shown in Tables \ref{tab:MainResult-cifar10}, \ref{tab:MainResult-fmnist}, \ref{tab:MainResult-SVHN}, and \ref{tab:MainResult-cifar100}, for all datasets with high Non-IID degree ($a=0.05$), the VHL gains more performance improvement than the case of lower Non-IID degree ($a=0.1$), which verifies that VHL could effectively defend against data heterogeneity.

\textbf{Different Non-IID Partition Methods.} The experiment results of another two kinds of Non-IID partition methods are listed in Table~\ref{tab:OtheNonIID}. Compared with LDA parition with $a=0.1$, these two kinds of Non-IID distribution make the generalization performance of FL drops much more. For these more difficult tasks, VHL gains much more performance improvement than other algorithms, which demonstrates that VHL could well defend different kinds of data heterogeneity.
\begin{table}[!t]
\centering
\caption{Experiment results of other Non-IID partition methods.} 
\vspace{1pt}
% \small{
% \footnotesize{
\scriptsize{
    \begin{tabular}{c|cccc}
    \toprule[1pt]
     &  \multicolumn{4}{c}{Test Accuracy  w/(w/o) \textbf{VHL}} \\ \cline{2-5}
    \multirow{-2}{*}{\makecell{Partition}} & \makecell{FedAvg} &   \makecell{FedProx} & SCAFFOLD & FedNova    \\
    \hline
    2-classes  & \textbf{71.88}/41.74 & \textbf{75.43}/54.21 & \textbf{66.02}/46.67 & \textbf{73.46}/43.38  \\
    Subset & \textbf{50.29}/38.33 & \textbf{39.41}/33.66  & \textbf{43.77}/33.35 &  \textbf{38.22}/32.18 \\
    \bottomrule
\end{tabular}
}
\label{tab:OtheNonIID}
\vspace{-0.3cm}
\end{table}

\textbf{Different Number of Clients.} The 10-client setting simulates the cross-silo FL in which clients have larger dataset, and the 100-client setting simulates cross-device FL in which clients have smaller dataset. As shown in Tables \ref{tab:MainResult-cifar10}, \ref{tab:MainResult-fmnist}, \ref{tab:MainResult-SVHN}, and \ref{tab:MainResult-cifar100}, VHL boosts the generalization performance and significantly reduces the communication cost to achieve the target accuracy in both scenarios.

\textbf{Different Local Epochs.} Results in Tables \ref{tab:MainResult-cifar10}, \ref{tab:MainResult-fmnist}, \ref{tab:MainResult-SVHN}, and \ref{tab:MainResult-cifar100} show that when increasing local training epochs, VHL also benefits FL. However, as shown in Table \ref{tab:MainResult-cifar100}, for CIFAR-100 when $E=5$, VHL does not work well. We suppose that reasons exist in two aspects: 1) As shown in Table~\ref{tab:MainResult-cifar100}, the test accuracy of centralized training using VHL is worse than the normal training. The larger number of classes need larger model capacity. The ResNet-50 cannot attain a high test accuracy on CIFAR-100, not to say 200 classes with the virtual dataset. 2) The larger local epochs means that clients will train more local iterations, which may cause clients to have a drifted feature expression of the virtual data, leading to the failure of calibration.

\textbf{Visualization of Feature Distribution.}
To understand how VHL helps generalize the model, we exploit t-SNE~\citep{T-SNE} to show the feature distributions of clients at the $299$-th communication round. To highlight the generalization performance, we show the features on test datasets, where merely 3 clients are used considering the better visualization. And we also report feature distributions at different communication rounds in Appendix~\ref{appendix:VisualizationFeature}. in Figure~\ref{fig:FeatureCalibration-Experiment}. From Figure~\ref{fig:FeatureCalibration-Experiment} (a), we can see that FedAvg makes most features from different labels at the same client be very close, and it makes the features from the same class be separate. It indicates that different clients have very different feature representation for the same or similar inputs in FedAvg, which results in poor generalization ability of the trained model. As shown in Figure~\ref{fig:FeatureCalibration-Experiment} (b), after local training with the virtual data without feature calibration (called \textbf{Naive VHL}), we can see the features of virtual data are consistent in all clients as the virtual data is shared, but the feature distributions of the original data are only slightly better than the naive FedAvg. Thus, we would like to pull the features of the original dataset close to the virtual dataset so that the corrected training data has a relatively low data heterogeneity. As shown in Figure~\ref{fig:FeatureCalibration-Experiment} (c), after we applied the feature calibration in VHL, the features from the same class at different clients are gathered closely and become compact. It indicates that training with VHL can help different clients learn homogeneous natural features for the same classes.

\begin{table}[t]
\small
% \footnotesize
\centering
% \caption{The best test accuracy and communication round to attain the target test accuracy with/without VHL on SVHN.} 
\caption{Results with/without VHL on SVHN.} 
% \vspace{1pt}
\small{
\begin{tabular}{c|ccc}
\toprule[1.5pt]
                  & ACC $\uparrow$ & ROUND  $\downarrow$       & Speedup $\uparrow$      \\
                  \cline{2-4} 
\multirow{-2}{*}{} & \multicolumn{3}{c}{w/ (w/o) \textbf{VHL}} \\
\midrule[1pt]
\multicolumn{4}{c}{centralized training ACC = 95.01\% (\textbf{95.27\%)} } \\
\midrule[1pt]
\multicolumn{4}{c}{\cellcolor{greyL} $a=0.1, E=1, K=10$ ( Target ACC = 88\% )} \\
\midrule[1pt]
FedAvg             & \textbf{93.49} (88.56) & \textbf{75} (251) & \textbf{$\times$ 3.3} ($\times$ 1.0)  \\
FedProx             & \textbf{91.70} (86.51) & \textbf{271} (Nan) & \textbf{$\times$ 0.9} (Nan)  \\
SCAFFOLD             & \textbf{87.54} (80.61) & Nan (Nan) & Nan (Nan)  \\
FedNova             & \textbf{93.35} (89.12) & \textbf{75} (251) & \textbf{$\times$ 3.3} ($\times$ 1.0)  \\
\midrule[1pt]
\multicolumn{4}{c}{\cellcolor{greyL} $a=0.05, E=1, K=10$  ( Target ACC =  82\% )} \\
\midrule[1pt]
FedAvg             & \textbf{92.26} (82.67) & \textbf{94} (357) & \textbf{$\times$ 3.8} ($\times$ 1.0)  \\
FedProx             & \textbf{89.30} (78.57) & \textbf{320} (Nan) & \textbf{$\times$ 1.1} (Nan)  \\
SCAFFOLD             & \textbf{83.89} (74.23) & \textbf{147} (Nan) & \textbf{$\times$ 2.4} (Nan)  \\
FedNova             & \textbf{91.82} (82.22) & \textbf{128} (741) & \textbf{$\times$ 2.8} ($\times$ 0.5)  \\
\midrule[1pt]
\multicolumn{4}{c}{\cellcolor{greyL} $a=0.1, E=5, K=10$  ( Target ACC =  87\% )} \\
\midrule[1pt]
FedAvg             & \textbf{90.52} (87.92) & \textbf{145} (131) & \textbf{$\times$ 0.9} ($\times$ 1.0)  \\
FedProx             & \textbf{87.20} (78.43) & \textbf{351} (Nan) & \textbf{$\times$ 0.4} (Nan)  \\
SCAFFOLD             & \textbf{88.04} (81.07) & \textbf{210} (Nan) & \textbf{$\times$ 0.6} (Nan)  \\
FedNova             & \textbf{90.99} (88.17) & \textbf{75} (162) & \textbf{$\times$ 1.7} ($\times$ 0.8)  \\
\midrule[1pt]
\multicolumn{4}{c}{\cellcolor{greyL}  $a=0.1, E=1, K=100$  ( Target ACC =  89\% )} \\
\midrule[1pt]
FedAvg             & \textbf{92.05} (89.44) & \textbf{362} (618) & \textbf{$\times$ 1.7} ($\times$ 1.0)  \\
FedProx             & \textbf{92.08} (89.51) & \textbf{356} (618) & \textbf{$\times$ 1.7} ($\times$ 1.0)  \\
SCAFFOLD             & 89.21 (\textbf{89.55}) & 968 (\textbf{643}) & $\times$ 0.6 ( \textbf{$\times$ 1.0 })  \\
FedNova             & \textbf{92.01} (82.08) & \textbf{676} (Nan) & \textbf{$\times$ 0.9} (Nan)  \\
\bottomrule[1.5pt] 
\end{tabular}
}
\vspace{-0.3cm}
\label{tab:MainResult-SVHN}
\end{table}

\begin{table}[t]
\centering
% \caption{The best test accuracy and communication round to attain the target test accuracy with/without VHL on CIFAR-100.} 
\caption{Results with/without VHL on CIFAR-100.} 
% \vspace{1pt}
\small{
\begin{tabular}{c|ccc}
\toprule[1.5pt]
                  & ACC $\uparrow$ & ROUND  $\downarrow$       & Speedup $\uparrow$      \\
                  \cline{2-4} 
\multirow{-2}{*}{} & \multicolumn{3}{c}{w/ (w/o) \textbf{VHL}} \\
\midrule[1pt]
\multicolumn{4}{c}{ centralized training ACC = 71.90 \% (\textbf{74.25 \%}) } \\
\midrule[1pt]
\multicolumn{4}{c}{\cellcolor{greyL} $a=0.1, E=1, K=10$ ( Target ACC = 67\% )} \\
\midrule[1pt]
FedAvg             & \textbf{70.04} (67.95) & \textbf{384} (497) & \textbf{$\times$ 1.3} ($\times$ 1.0)  \\
FedProx             & \textbf{68.29} (65.29) & \textbf{617} (Nan) & \textbf{$\times$ 0.8} (Nan)  \\
SCAFFOLD             & \textbf{67.88} (67.14) & \textbf{294} (766) & \textbf{$\times$ 1.7} ($\times$ 0.6)  \\
FedNova             & \textbf{69.58} (68.26) & \textbf{384} (472) & \textbf{$\times$ 1.3} ($\times$ 1.1)  \\
\midrule[1pt]
\multicolumn{4}{c}{\cellcolor{greyL} $a=0.05, E=1, K=10$  ( Target ACC =  62\% )} \\
\midrule[1pt]
FedAvg             & \textbf{65.61} (62.07) & \textbf{354} (514) & \textbf{$\times$ 1.5} ($\times$ 1.0)  \\
FedProx             & \textbf{64.39} (61.52) & \textbf{482} (Nan) & \textbf{$\times$ 1.1} (Nan)  \\
SCAFFOLD             & \textbf{60.67} (59.04) & Nan (Nan) & Nan (Nan)  \\
FedNova             & \textbf{66.45} (60.35) & \textbf{320} (Nan) & \textbf{$\times$ 1.6} (Nan)  \\
\midrule[1pt]
\multicolumn{4}{c}{\cellcolor{greyL} $a=0.1, E=5, K=10$  ( Target ACC =  69\% )} \\
\midrule[1pt]
FedAvg             & \textbf{69.85} (69.81) & \textbf{327} (283) & \textbf{$\times$ 0.9} ($\times$ 1.0)  \\
FedProx             & \textbf{63.83} (62.62) & Nan (Nan) & Nan (Nan)  \\
SCAFFOLD             & 69.43 (\textbf{70.68}) & 291 (\textbf{171}) & $\times$ 1.0 (\textbf{$\times$ 1.7})  \\
FedNova             & 68.86 (\textbf{70.05}) & Nan (\textbf{292}) & Nan (\textbf{$\times$ 1.0})  \\
\midrule[1pt]
\multicolumn{4}{c}{\cellcolor{greyL}  $a=0.1, E=1, K=100$  ( Target ACC =  48\% )} \\
\midrule[1pt]
FedAvg             & \textbf{53.45} (48.33) & \textbf{717} (967) & \textbf{$\times$ 1.3} ($\times$ 1.0)  \\
FedProx             & \textbf{52.68} (48.14) & \textbf{717} (955) & \textbf{$\times$ 1.3} ($\times$ 1.0)  \\
SCAFFOLD             & \textbf{54.93} (51.63) & \textbf{656} (827) & \textbf{$\times$ 1.5} ($\times$ 1.2)  \\
FedNova             & \textbf{53.50} (48.12) & \textbf{797} (967) & \textbf{$\times$ 1.2} ($\times$ 1.0)  \\
\bottomrule[1.5pt] 
\end{tabular}
}
\vspace{-0.3cm}
\label{tab:MainResult-cifar100}
\end{table}

\begin{figure}[h!]
   \centering
    \subfigtopskip=2pt
    \setlength{\belowdisplayskip}{2pt}
    \setlength{\abovedisplayskip}{-5pt}
    \setlength{\abovecaptionskip}{-2pt}
    \subfigbottomskip=-2pt
    \subfigcapskip=-1pt
    \subfigure[Test Accuracy]{\includegraphics[width=0.23\textwidth]{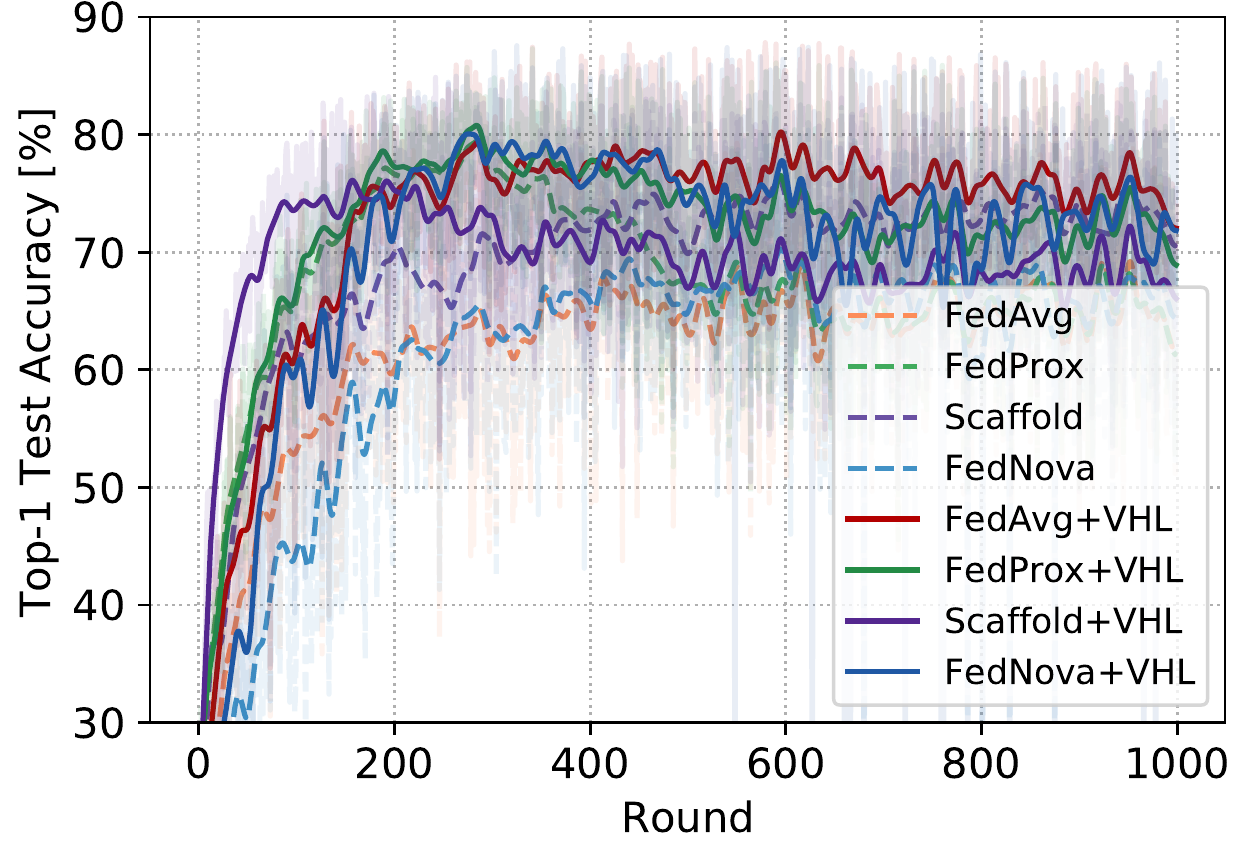}}
    \subfigure[Model Divergence]{\includegraphics[width=0.23\textwidth]{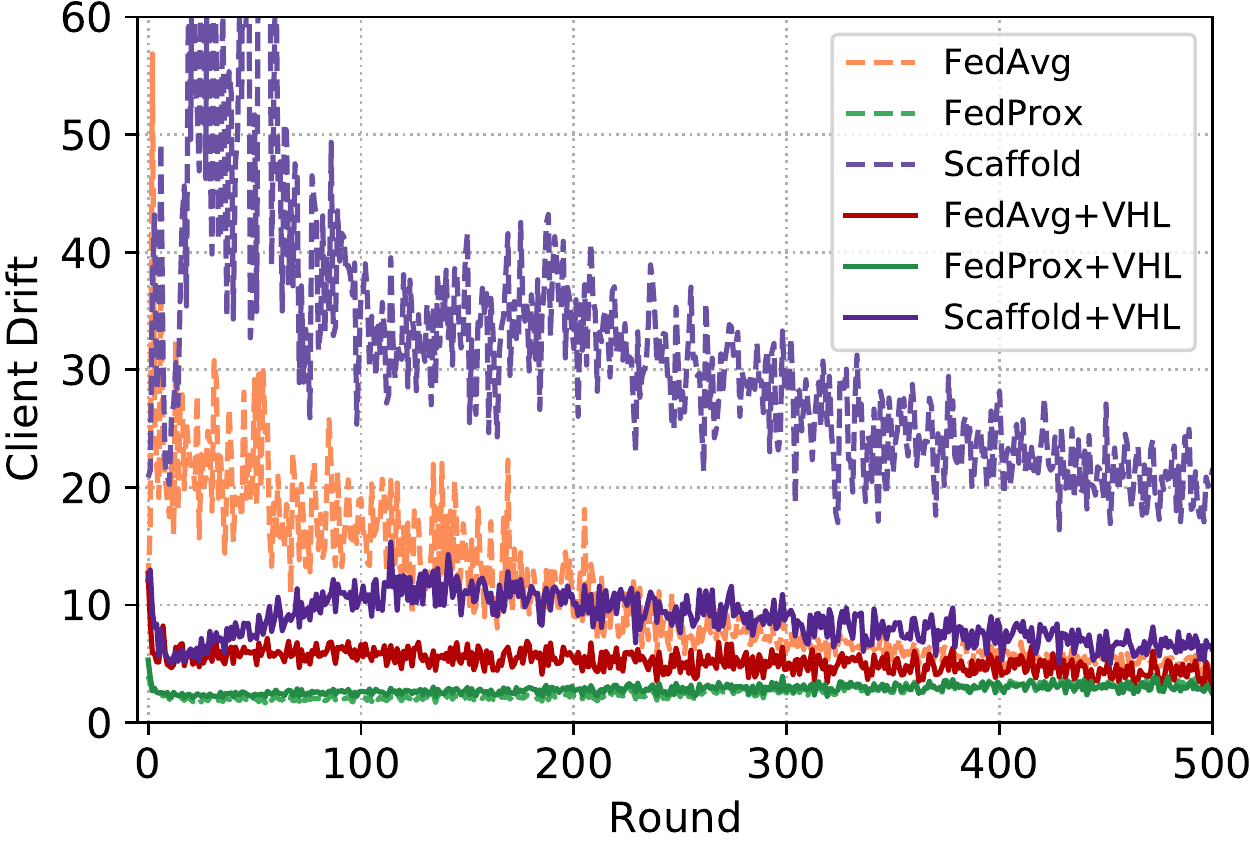}}
    \caption{CIFAR10 with $a=0.1$, $E=1$, $K=10$. Grey lines represent the test accuracy, and solid lines represents smoothed version of the grey lines for better visualization. }
    \label{fig:Convergence-MainResults}
\vspace{-0.1cm}
\end{figure}

% \begin{figure}[h!]
%   \centering
% % \!\!\!\!\!\!\!\!
%      \subfigure[CIFAR-10 ]{\includegraphics[width=0.24\textwidth]{Convergence/normal-resnet18_v2-cifar10-0.1-10-1.pdf}}
%      \subfigure[FMNIST]{\includegraphics[width=0.23\textwidth]{Convergence/normal-resnet18_v2-fmnist-0.1-10-1.pdf}}
%     \subfigure[SVHN]{\includegraphics[width=0.23\textwidth]{Convergence/normal-resnet18_v2-SVHN-0.1-10-1.pdf}}
%     \subfigure[CIFAR-100]{\includegraphics[width=0.23\textwidth]{Convergence/normal-resnet34_v2-cifar100-0.1-10-1.pdf}}
%     \caption{Convergence speed with $a=0.1$, $E=1$, $K=10$.}
%     \label{fig:Convergence-MainResults}
% % \vspace{-0.5cm}
% \end{figure}

% \vspace{-4pt}
\subsection{Other Facets of VHL} \label{sec:DiffFacetsVHL}
% \vspace{-4pt}

In this section, we dive into VHL deeper by further studying other facets related to it. In addition, we conduct experiments on CIFAR to investigate the intriguing property of VHL. All of these experimental results are listed in Table~\ref{tab:FurtherStudy}.

\textbf{Transfer Learning from Virtual Dataset.}
The deep learning model pretrained on virtual datasets could also perform well on realistic datasets~\citep{baradad2021learning}. Thus, to understand whether the benefits of VHL come from only conducting transfer learning from the virtual dataset, we pretrain the server model with the virtual dataset, and then conduct vanilla FedAvg. We call this simple algorithm as virtual federated transfer learning (VFTL). The results of this experiment are shown in Table~\ref{tab:FurtherStudy}. Due to the limited space, we show the convergence curve of this experiment in the Figure~\ref{fig:Ablation-Convergence} (a) in Appendix~\ref{appendix:MoreFacets}. The experimental results show that the model performance could not benefit from such a pretraining on virtual dataset. But the Figure~\ref{fig:Ablation-Convergence} (a) shows the VFTL could provide some convergence speedup at the beginning of training. We suppose such an initial speedup comes from the better feature extractor of the pretrained model. This phenomenon is compatible with the recently finding that the loading pretrained model could accelerate FL training~\citep{he2021fedcv}.

% The density of this virtual feature distribution is xxxxxxxxxxxxxxxxxxxx, in which different mean vector is defined as the center of class xxxx.

% \subsubsection{Naive Training with Virtual Dataset}

% \textbf{Naive Training with Virtual Dataset}
\textbf{Naive VHL.}
We study the effect of simply training with the natural data together with the virtual dataset, without any extra algorithm designs. In this experiment, clients will sample both natural data together with the virtual data, without feature calibration.

The experiment results are listed in Table~\ref{tab:FurtherStudy}, showing that even simply adding the virtual dataset could also benefit FL. We conjecture that the benefit of this may come from several aspects: 1) As the data heterogeneity decreases, the client drift is alleviated; 2) The client models could learn low-level features from the virtual data, thus there exists the effect of transfer learning from the virtual dataset; 3) By adding the virtual dataset, the clients now have more classes than before, which can effectively alleviate the negative impact of local dominant classes on classifiers, which is also studied by~\citep{luo2021no, Fed2}.

\textbf{Virtual Feature Alignment.}
To validate the effect of feature calibration, we completely eliminate the effect of the virtual data by removing the virtual dataset, and sampling new features from different Gaussian distributions. Then, we conduct feature calibration based on these shared features. We name this simple method as \textit{Virtual Feature Alignment} (VFA). Besides aligning the features, we can also employ approaches like style normalization~\citep{stylejin} to mitigate the domain gap. As shown in Table~\ref{tab:FurtherStudy}, feature calibration benefits FL even only with the random features. Although the performance of VFA falls behind the VHL, it is valuable to explore VFA, due to its computation efficiency.

\begin{table}[h!]
\centering
\caption{Experiment results of further study of VHL on CIFAR10.} 
% \vspace{-1pt}
\scriptsize{
\begin{tabular}{c|cccc}
\toprule[1.5pt]
  \!\!& {\small FedAvg} & \small FedProx &\small SCAF. &\small FedNova \\
\midrule[1.5pt]
Baselines          & 79.98 & 83.56 & 83.58 & 81.35  \\
VHL          & 87.82 & 87.30 & 84.87 & 87.56  \\
\midrule[1pt]
\multicolumn{5}{c}{\cellcolor{greyL} Other Facets of VHL} \\
\midrule[1pt]
VFTL  & 80.38 & 82.20 & 83.83 & 80.63  \\
\midrule[0.5pt]
Naive VHL & 86.50 & 85.66 & 85.70 & 85.74  \\
\midrule[0.5pt]
VFA          & 85.14 & 84.75 & 85.31 & 86.59   \\
\midrule[1pt]
\multicolumn{5}{c}{\cellcolor{greyL} Ablation Study } \\
\midrule[1pt]
Pure Noise          & 87.01 & 86.46 & 85.57 & 87.81  \\
Simple-CNN          & 84.87 & 85.30 & 84.15 & 85.25  \\
Tiny-ImageNet          & 84.05 & 83.62 & 81.57 & 85.41  \\
%\bottomrule[0.5pt] 
%Share Labels          & 87.24 & 85.90 & 86.25 & 88.36  \\
\bottomrule[0.5pt] 
$B_v = 64$        & 87.36 & 86.36 & 82.39 & 87.20   \\
$B_v = 128$        & 87.82 & 87.30 & 84.87 & 87.56   \\
$B_v = 256$         & 88.95 & 86.82 & 84.68 & 87.89  \\
$B_v = 384$         & 89.69 & 86.59 & 85.87 & 88.73  \\
\bottomrule[0.5pt] 
$\lambda = 0.1$        & 87.02 & 86.12 & 84.25 & 87.02  \\
$\lambda = 0.2$        & 87.04 & 86.02 & 84.41 & 87.65  \\
$\lambda = 0.5$        & 87.02 & 85.99 & 84.29 & 87.15  \\
$\lambda = 1.0$        & 87.82 & 87.30 & 84.87 & 87.56  \\
$\lambda = 2.0$        & 87.87 & 86.86 & 84.81 & 88.71  \\
$\lambda = 5.0$        & 83.52 & 88.34 & 85.58 & 88.47  \\
$\lambda = 10.0$        & 88.65 & 88.41 & 85.29 & 88.39  \\
\bottomrule[0.5pt] 
$h_{shallow}$        & 86.10 & 85.59 & 85.19 & 84.95  \\
$h_{middle}$        & 87.06 & 86.30 & 87.97 & 87.27  \\
$h_{deep}$        & 89.26 & 87.54 & 86.00 & 88.42  \\
$h_{last}$        & 87.82 & 87.30 & 84.87 & 87.56  \\
\midrule[1pt]
\multicolumn{5}{c}{\cellcolor{greyL} Model Capacity} \\
\midrule[1pt]
Res10-Baselines          & 83.55 & 82.76 & 82.91 & 83.14  \\
Res10-VHL          & 87.39 & 86.08 & 85.09 & 88.03  \\
\midrule[0.5pt]
Res18-Baselines         & 79.98 & 83.56 & 83.58 & 81.35  \\
Res18-VHL          & 87.82 & 87.30 & 84.87 & 87.56  \\
\midrule[0.5pt]
Res34-Baselines          & 82.73 & 84.04 & 81.11 & 82.48   \\
Res34-VHL           & 87.92 & 88.05 & 84.80 & 88.74  \\
\bottomrule[1.5pt] 
\end{tabular}
}
\vspace{-0.1cm}
\label{tab:FurtherStudy}
\end{table}

% \vspace{-4pt}
\subsection{Ablation Study} \label{sec:AblationStudy}
% \vspace{-4pt}

\textbf{VHL with Other Datasets.}
We study the effect of replacing the virtual dataset generated by the StyleGAN to other datasets. The first dataset is generated by upsampling the pure Gaussian Noise, which is shown in Figure~\ref{fig:Gaussian_Noise} in the Appendix~\ref{appendix:GenerateData}. The second dataset is generated by a simple CNN rather than StyleGAN shown in Figure~\ref{fig:cifar_conv_decoder} in the Appendix~\ref{appendix:GenerateData}. And the third dataset is a 10-class subset of Tiny-ImageNet~\citep{le2015tiny}.

As shown in Table~\ref{tab:FurtherStudy}, VHL with the simple pure noise could benefit FL, reducing the difficulty of generating the virtual dataset. And it shows that using a Tiny-ImageNet as our virtual dataset could not outperform using pure noise, which is a bit surprising. Because the realistic dataset may have more plentiful semantic information than the virtual dataset. The possible reason is that, the Tiny-ImageNet is more difficult to be separated than the generated virtual dataset, thus the feature calibration cannot work well. To verify this hypothesis, We report the curve of the training accuracy on the virtual dataset in Figure~\ref{fig:Ablation-Convergence} (c) in Appendix~\ref{appendix:MoreFacets}.

\textbf{Sampling More Virtual Data.}
% NOTE THAT: this subsetction should be placed to another setction, e.g., different facets of VHKL.
We note that under the high Non-IID degree, some clients would own much more data samples than the virtual data, reducing the effect of feature calibration. Therefore, we consider increasing the sampling weight of the virtual data, making clients to see more virtual data to strengthen the effect of calibration. We conduct VHL with CIFAR-10 with different batch size $B_v$ to verify the effect of sampling more virtual data. The experiment results are shown in Table~\ref{tab:FurtherStudy} and Figure~\ref{fig:Ablation-Convergence} (a) in  Appendix~\ref{appendix:MoreFacets}, demonstrating that sampling more virtual data could strengthen the effect of VHL.

\textbf{Different Calibration Weight.} We adjust different calibration weights $\lambda$ as sensitivity test. The experimental results are listed in Table~\ref{tab:FurtherStudy}, showing that VHL is not sensitive to the calibration weight.

\textbf{Using Features of Different Layers.}
To study the impact of features from different layers, we conduct calibration based on different output layers of ResNet-18. We find that using the features from the $4$-th layer of ResNet-18 has the best performance. We suppose that calibration based on too shallow layer may be not enough to impact the subsequent layers feature shift. However, the calibration based on the last layer may interrupt the normal classification between the natural data and the virtual data, because which have completely different high-level semantic information.

\textbf{Model Capacity.} Adding more data means deep learning models may need more model capacity, so we investigate the sensitivity of VHL to model capacity. We conduct experiments on CIFAR-10 with different models, including ResNet-10, ResNet-18 and ResNet-34. Results in Table~\ref{tab:FurtherStudy} show that VHL could perform well on all three models of different model capacity.

% \vspace{-4pt}
\subsection{Limitations} \label{sec:limitations}
% \vspace{-4pt}

\textbf{Guarantee the Diversity of the Virtual Dataset.} In our method, the virtual dataset needs the same number of classes as the original datasets, so that the label alignment could be implemented. However, when meeting numerous classes of original datasets, generating a virtual dataset with enough diversity may cost more calculation resources.

% \textbf{Model Capacity.} In VHL, the learning task of all clients consists of not only original datasets, but also the new virtual datasets. Therefore, we may need a stronger model than before.

\textbf{Extra Calculation.} VHL needs to conduct training on the virtual data, causing the extra calculation cost. VFA is a possible direction to reduce the calculation, as mentioned in Sec.~\ref{sec:DiffFacetsVHL}. We compare the calculation cost of different algorithms in Table~\ref{tab:CalculationCost} in Appendix~\ref{appendix:CalculationCost}.

\section{Conclusion}
In this paper, we find that FL could remarkably benefit from a virtual dataset containing no information related to the natural dataset, shedding light on the data heterogeneity mitigation from a virtual data perspective. Through the calibration on the homogeneous virtual data, FL could attain significant performance improvement and much less client drift. Our contribution not only exists in the improvements of model performance in FL, but also in inspiring a new viewpoint to FL with many intriguing experimental phenomena. We hope the future works could unearth more things about VHL and exploit it to enhance FL or other machine learning tasks.

\section*{Acknowledgements}
% \textbf{Acknowledgments.}
This work was supported in part by Hong Kong CRF grant C2004-21GF, Hong Kong RIF grant R6021-20, and Hong Kong Research Matching Grant RMGS2019\_1\_23.
YGZ and BH were supported by the RGC Early Career Scheme No. 22200720, NSFC Young Scientists Fund No. 62006202, and Guangdong Basic and Applied Basic Research Foundation No. 2022A1515011652.

% We summarize the possible future directions as below:

% Our contribution not only exists in the improvements of model performance in FL, but also in inspiring a new interesting angle to FL with many interesting experiment phenomena..

% \textbf{Transfer Learning from Virtual Datasets.} As Table~\ref{tab:FurtherStudy} shows, we find that without the class alignment, even naive training with noise, the model performance and convergence speed could be improved.

% \textbf{The Need of Better Definition of Non-IID Degree.} As pointed by Theorem xxxxx\ref{}, the definition of Non-IID degree in current literature is not well-defined, due to which could be hacked by our virtual dataset.

% \clearpage
% \newpage

% Introduce FedAvg

\nocite{langley00}

\bibliography{cite}
\bibliographystyle{icml2022}

\newpage
\appendix
\onecolumn

% \section*{Appendix}

% \section{Do \emph{not} have an appendix here}

% \section{Notations and Proof}\label{appendix:notation}
\section{Proof}\label{appendix:Proof}

\subsection{Proof of Theorem~\ref{theo}}
Recall that the model $f$ can be decomposed as $f=\phi \circ \rho$, the statistical margin for $f$ then becomes:
\begin{equation}
    \textit{SM}_m(f,\mathcal{P}(x,y)) = 
    \textit{SM}_m(\rho,\mathcal{P}(\phi(x),y)).
\end{equation}
In what follows, we use $\mathcal{P}(\phi,y)$ to represent $\mathcal{P}(\phi(x),y)$ for brevity.

Before giving detailed proof, we introduce a lemma:
\begin{lemma} \label{lemma}
Let $\mathcal{P}$ and $\mathcal{P}_v$ be two distributions with identical label distributions, $d\left(\cdot, \cdot\right)$ be the Wasserstein distance of two distributions, i.e., 
\begin{equation}  \nonumber
    d
    \left(
    \mathcal{P}\left(\phi|y\right), \mathcal{P}_v\left(\phi|y\right) \right) 
    = \inf_{J \in \mathcal{J}(\mathcal{P}(\phi|y),\mathcal{P}_v(\phi|y))} \mathbb{E}_{(\phi, \phi') \sim J} \ m(\phi,\phi'),
\end{equation}
where $\mathcal{J}(\mathcal{P}(\phi|y),\mathcal{P}_v(\phi|y))$ is the set of joint distributions of $\mathcal{P}(\phi|y)$ and $\mathcal{P}_v(\phi|y)$. 
Then, we have
\begin{equation}  \nonumber
\mathbb{E}_{y}
d(\mathcal{P}(\phi|y), \mathcal{P}_v(\phi|y))
\geq
|\mathbb{E}_{\rho:=\textit{VHL}(\mathcal{P}_v(\phi,y))}
[\textit{SM}_m(\rho,\mathcal{P}(\phi,y)) - \textit{SM}_m(\rho,\mathcal{P}_v(\phi,y))]|.
\end{equation}
\end{lemma}

\begin{proof}
We construct an optimal transport $\mathcal{J}^{*}_y$ between the conditional distributions $\mathcal{P}|y$ and $\mathcal{P}_v|y$. Using the auxiliary distribution, we have
\begin{equation} \label{lemma1}
\begin{split}
    \textit{SM}_m(\rho,\mathcal{P}(\phi,y))
    &= 
    \mathbb{E}_{(\phi,y) \sim \mathcal{P}} \inf_{\rho(\phi') \neq y} m(\phi',\phi)
    = 
    \mathbb{E}_{y} \mathbb{E}_{\phi \sim \mathcal{P}|y} \inf_{\rho(\phi') \neq y} m(\phi',\phi)
    = 
    \mathbb{E}_{y} \mathbb{E}_{(\phi,\phi'') \sim \mathcal{J}^{*}_y} \inf_{\rho(\phi') \neq y} m(\phi',\phi) \\
    &\leq
    \mathbb{E}_{y} \mathbb{E}_{(\phi,\phi'') \sim \mathcal{J}^{*}_y} \inf_{\rho(\phi') \neq y} [m(\phi',\phi'') + m(\phi'',\phi)] \\
    &=
    \mathbb{E}_{y} \mathbb{E}_{(\phi,\phi'') \sim \mathcal{J}^{*}_y} \inf_{\rho(\phi') \neq y} m(\phi'',\phi') + 
    \mathbb{E}_{y} \mathbb{E}_{(\phi,\phi'') \sim \mathcal{J}^{*}_y} m(\phi'',\phi) \\
    &=
    \mathbb{E}_{(\phi'',y) \sim \mathcal{P}_v}  \inf_{\rho(\phi') \neq y} m(\phi'',\phi') + 
    \mathbb{E}_{y}d(\mathcal{P}(\phi|y), \mathcal{P}_v(\phi|y)) \\
    &=
    \textit{SM}_m(\rho,\mathcal{P}_v(\phi,y)) + 
    \mathbb{E}_{y}d(\mathcal{P}(\phi|y), \mathcal{P}_v(\phi|y)).
\end{split}
\end{equation}
Similarly, we have
\begin{equation} \label{lemma2}
    \textit{SM}_m(\rho,\mathcal{P}_v(\phi,y)) 
    \leq
     \textit{SM}_m(\rho,\mathcal{P}(\phi,y))
    + 
    \mathbb{E}_{y}d(\mathcal{P}(\phi|y), \mathcal{P}_v(\phi|y)).
\end{equation}
Combining Eq.~\ref{lemma1} and Eq.~\ref{lemma2}, we have
\begin{equation}
    \mathbb{E}_{y}
d(\mathcal{P}(\phi|y), \mathcal{P}_v(\phi|y))
\geq
|\mathbb{E}_{\rho:=\textit{VHL}(\mathcal{P}_v(\phi,y))}
[\textit{SM}_m(\rho,\mathcal{P}(\phi,y)) - \textit{SM}_m(\rho,\mathcal{P}_v(\phi,y))]|.
\end{equation}
This completes the proof.
\end{proof}

In addition, we decompose $\textit{SM}_m(f,\mathcal{P}(\phi,y))$ as follows:
\begin{equation} \nonumber
    \textit{SM}_m(f,\mathcal{P}(\phi,y)) 
    = \textit{SM}_m(f,\mathcal{P}(\phi,y))
    - \textit{SM}_m(f,\mathcal{P}_v(\phi,y))
    + \textit{SM}_m(f,\mathcal{P}_v(\phi,y))
    - \textit{SM}_m(f,\tilde{D}(\phi,y))
    + \textit{SM}_m(f,\tilde{D}(\phi,y)),
\end{equation}
where $\tilde{D}(\phi,y)$ stands for the training set containing $n_v$ samples drawn from $\mathcal{P}_v(\phi,y)$. In what follows, we omit $(\phi,y)$ for each joint distribution, e.g., using $\mathcal{P}$ for $\mathcal{P}(\phi,y)$, and denote $\mathbb{E}_{\rho:=\textit{VHL}(\mathcal{P}(\phi,y))}$ as $\mathbb{E}_{\rho \leftarrow \mathcal{P}}$ for brevity.

\begin{theorem} \label{theo_ap}
Let $f=\phi \circ \rho$ be a neural network decompose of a feature extractor $\phi$ and a classifier $\rho$, $\mathcal{P}(x,y)$ and $\mathcal{P}_v(x,y)$ are natural and virtual distributions with identical label distributions. Then, learning $f$ with the proposed distribution mismatch penalty, i.e., Eq.~\ref{fl_l_bar} elicits an model with bounded statistical margin, i.e., Definition~\ref{def:margin}.
\end{theorem}

\begin{proof}
Built upon the decomposition and lemma~\ref{lemma}, we can bound the statistical margin $\mathbb{E}_{\rho \leftarrow \mathcal{P}_v} \textit{SM}_m(\rho,\mathcal{P})$:
\begin{equation}
\begin{split}
\mathbb{E}_{\rho \leftarrow \mathcal{P}_v}
\textit{SM}_m(\rho,\mathcal{P})
& \geq
\mathbb{E}_{\rho \leftarrow \mathcal{P}_v}
\textit{SM}_m(\rho,\tilde{D}) 
 -
|\mathbb{E}_{\rho \leftarrow \mathcal{P}_v}
[\textit{SM}_m(\rho,\mathcal{P}_v) - \textit{SM}_m(\rho,\tilde{D})]| 
 -
|\mathbb{E}_{\rho \leftarrow \mathcal{P}_v}
[\textit{SM}_m(\rho,\mathcal{P}) - 
\textit{SM}_m(\rho,\mathcal{P}_v)]| \\
& \geq
\mathbb{E}_{\rho \leftarrow \mathcal{P}_v}
\textit{SM}_m(\rho,\tilde{D}) 
 -
|\mathbb{E}_{\rho \leftarrow \mathcal{P}_v}
[\textit{SM}_m(\rho,\mathcal{P}_v) - \textit{SM}_m(\rho,\tilde{D})]| 
 -
\mathbb{E}_{y}
d(\mathcal{P}(\phi|y), \mathcal{P}_v(\phi|y))
\end{split}
\end{equation}
The first and second term $\mathbb{E}_{\rho \leftarrow \mathcal{P}_v}
\textit{SM}_m(\rho,\tilde{D}), |\mathbb{E}_{\rho \leftarrow \mathcal{P}_v}
[\textit{SM}_m(\rho,\mathcal{P}_v) - \textit{SM}_m(\rho,\tilde{D})]| $ can be bounded by maximizing the statistical margin~\citep{margin4}, requiring that the distribution $\mathcal{P}_v$ is separable. Specifically, the first term reveals that the statistical margin on the training set should be large, and the second term further requires that the statistical margin should be large on the (virtual) distribution. The last term $\mathbb{E}_{y}
d(\mathcal{P}(\phi|y), \mathcal{P}_v(\phi|y))$ is bounded by the conditional distribution, i.e., Eq~\ref{fl_l_bar}. Hence, the statistical margin elicited by Eq.~\ref{fl_l_bar} is bounded, leading to bounded generalization performance.
\end{proof}

\section{More Related work}\label{appendix:MoreRelated}

\subsection{Federated Learning}

Federated Learning (FL) is first proposed by~\citep{mcmahan2017communication} as a distributed algorithm to protect users' data privacy while collaboratively training a global model. The heterogeneous data distribution across all clients severely damage the convergence of federated learning and final performance~\citep{zhao2018federated, li2019convergence,kairouz2019advances,Resampling}. When training with the heterogeneous data, local and global models are much more unstable than centralized training~\citep{karimireddy2019scaffold}. And there exists a more severe divergence between local models of FL than that of distributed training with IID data~\citep{fedprox, karimireddy2019scaffold}. This inconsistency is called~\textbf{client drift}. 

To mitigate the client drift, hyper-parameters may need to be carefully adjusted~\citep{wang2021field}, like learning rate and local training iterations, which directly decide how fast local model move and thus how far drift they could be. However, too small learning rate also means slower convergence. Thus, how to simultaneously achieve faster training speed and a more gentle client drift is important. To this end, many research works try to design more effective federated learning algorithms to address client drift problem.

\textbf{Model Regularization.}
This line of research focus on adding regularization to calibrate the optimization direction of local models, restricting local models from being too far away from the server model. FedProx~\citep{fedprox} adds a penalty of the L2 distance between local models to the server model.

SCAFFOLD~\citep{karimireddy2019scaffold} utilizes the history information to reduce the ``client variance'', thus decreasing the client drift. MOON~\citep{li2021model} performs contrastive learning between the server model and client models, calibrating clients' learned representation.

\textbf{Optimization Schemes.}
From the optimization point of view, some methods propose to correct the updates from clients, accelerating and stabilizing the convergence.
FedNova ~\citep{wang2020tackling} propose to normalize the local updates to eliminate the inconsistency between the local and global optimization objective functions.
~\citep{LDA} proposes FedAvgM, which exploits the history updates of server model to avoid the overfits on the selected clients at current round.
A recent work~\citep{reddi2020adaptive} proposes FEDOPT, which generalizes the centralized optimization methods into FL scenario, like FedAdaGrad, FedYogi, FedAdam.

\textbf{Sharing Data.}
Because the original cause of client drift is the data heterogeneity, another line of research focuses on generating more data that carries some information of the natural data, and sharing them to all clients. It has been found that sharing a part of natural data could significantly benefit federated learning~\citep{zhao2018federated}, yet which sacrifices the privacy of clients' data.

A series of works~\citep{5670948, NIPSAsimpleDPDataRelease, CompressiveLearninWithDP, johnson2018towards, DataSynthesisDPMRF} add noise on data queries to implement sharing data with privacy guarantee at some degree. In addition, sharing incomplete data is also a promising direction, where the generation process can be similar to ~\citep{kaiwen}.

FD~\citep{sharing1} proposes to let clients collectively train a generative model to locally reproduce the data samples of all devices, augmenting their local data to become IID. 
G-PATE~\citep{long2021g} leverages GAN to generate data, with private aggregation among different discriminators to ensure privacy guarantees at some degree. ~\citep{FLviaSyntheticData} utilized the adversarial learning to generate data samples based on the raw data to help FL.

Fed-ZDAC~\citep{hao2021towards} utilize the intermedia activations and BN layer statistic of the raw data information to learn and share synthetic data. The profiles of the raw images could be still seen from the synthetic images, exposing the raw data information at a degree. XorMixFL~\citep{shin2020xor} proposes to collect other clients' XOR-encoded data samples that are decoded only using each client's own data samples, also exposing raw data information to other clients. FedMIX~\citep{yoon2021fedmix} proposes to let clients send and receive the averaged local data of each batch. Strictly speaking, this method exposes the statistic information of local data, i.e. the mean values.

FedMD~\citep{li2019fedmd}, Cronus~\citep{ chang2019cronus} and CCVR~\citep{luo2021no} transmit the logits information of data samples to enhance FL, which may expose the accurate high-level semantic information of raw data.

FedDF~\citep{FedDF} utilizes other data and conduct knowledge distillation based on these data, to transfer knowledge of models between server and clients. The most difference between FedDF and VHL is that our model will be directly trained on the virtual data, which is seen as a part of the new objective function. And VHL does not require the shared dataset to be meaningful and has transferability with raw data.

To highlight our novelty, we compare and demystify different FL algorithms related to the sharing data in Table~\ref{tab:demystifySharingData}. We would like to clarify that we are the first to raise and analyze the question of how to introduce virtual datasets (even pure noise) containing absolutely no private information to promote FL. Specifically, related works, i.e., the first 9 lines in Table~\ref{tab:demystifySharingData}, have risks of exposing private data, while the 10-th related work in Table~\ref{tab:demystifySharingData}, FedDF, requires a distribution similar to the private distribution, see Theorem 5.1 in the FedDF. In contrast, we consider \textbf{arbitrary} datasets with \textbf{no private information}.

We omitted the comparison with works in Table~\ref{tab:demystifySharingData}, because of the above differences. Because FedDF does not expose the information of the raw data, we compare VHL with the FedDF. We have conducted experiments using the same settings as in FedDF. The results are listed in Table~\ref{tab:ResultsFedDF}. For FedDF, the distribution similar to the private distribution (FedDF + real) improves the accuracy, while noise (FedDF + noise) degrades the accuracy. VHL could perform much better with the noise dataset.

% We demystify different FL algorithms with the virtual datasets about sharing data in Table~\ref{tab:demystifySharingData}.

\begin{table*}[t]
\centering
\caption{Demystifying different FL algorithms related to the sharing data.} 
\vspace{1pt}
\centering
\small{
\begin{tabular}{c|ccc}
\toprule[1.5pt]
                  &  Shared Data &  Relation with Private Raw Data & Objective of the Algorithm     \\
\midrule[1pt]
% PMW~\citep{5670948}   & New Sythetic Data  &  DP on Raw Data & Database Privacy \\
% MWEM~\citep{NIPSAsimpleDPDataRelease}   & New Sythetic Data  & DP on Raw Data  & Database Privacy \\
FD~\citep{sharing1}  & Sythetic Data  &  Generated Based on Raw Data &  Server Model Performance \\
G-PATE~\citep{long2021g} & Sythetic Data   &  Generated Based on Raw Data &  Server Model Performance\\
~\citep{FLviaSyntheticData}  & Sythetic Data   &  Generated Based on Local Model &  Server Model Performance\\
Fed-ZDAC~\citep{hao2021towards} & Intermediate Features   &  Features of Raw Data &  Server Model Performance\\
XorMixFL~\citep{shin2020xor} & STAT. of raw Data   &  -- &  Server Model Performance \\
FedMIX~\citep{yoon2021fedmix} & STAT. of raw Data   &  -- &  Server Model Performance\\
FedMD~\citep{li2019fedmd}  &  logits &  Features of Raw Data &     Personalized FL     \\
CCVR~\citep{luo2021no}  &  logits  &  Features of Raw Data &     Server Model Performance     \\
Cronus~\citep{ chang2019cronus}  &  Intermediate Features &  Features of Raw Data  &    Defend Poisoning Attack      \\
FedDF~\citep{FedDF} & New Data   &  Has Transferability with Raw Data  &  Server Model Performance\\
VHL (ours) & New Data   &  No limitation  &  Server Model Performance\\
\bottomrule[1.5pt] 
\end{tabular}
}
\begin{tablenotes}
    % \small
	\item Note: ``STAT.'' means statistic information, like mean or standard deviation. 
\end{tablenotes}
\vspace{-0.5cm}
\label{tab:demystifySharingData}
\end{table*}

\begin{table*}[t]
\centering
\caption{Comparisons with FedDF.} 
\vspace{1pt}
\small{
    \begin{tabular}{c|cccc}
    \toprule[1pt]
    Non-IID Degree & FedAvg & FedDF + real & FedDF + noise  & FedAvg + VHL  \\
    \hline
    $\alpha=0.1$  & 62.22 &  71.36  & -- & \textbf{75.02} \\
    $\alpha=1.0$  & 76.01 & 80.69  & 46.9  &  \textbf{83.13} \\
    \bottomrule
\end{tabular}
}
\vspace{-0.5cm}
\label{tab:ResultsFedDF}
\end{table*}

% ~\citep{standford2019privacy} 

\subsection{Virtual Dataset}

Generative learning~\citep{croce2020gan, goodfellow2020generative, Karras2019stylegan2} attains great progress in recent years, making it possible to synthesis images, texts or sound that cannot be distinguished from the real data. And the Unreal Engine~\citep{UnrealCV} or other advanced computer graphics techniques also could generate vivid virtual images.

Recently, some researchers exploit the virtual datasets to tackle many problems in machine learning tasks. ~\citep{baradad2021learning} proposes that the neural networks could benefit from noise data. Specifically, the deep learning models pretrained on the noise data, are capable of generalizing well on real-world datasets. Apollo Synthetic Dataset ~\citep{huang2019apolloscape} is a virtual dataset for training better autonomous driving models. ~\citep{miyato2018virtual} exploits virtual labels for adversarial training.
VHL plays as a role to connect the federated learning with the virtual dataset.

\section{Details of Experiment Configuration}\label{appendix:DetailedExp}

\subsection{Hyper-parameters}
% For most experiments, we use the learning rate as 0.01. However, some ot
The learning rate configuration has been listed in Table~\ref{tab:LearningRate}. We report the best results and their learning rates (searched in $\left \{0.0001, 0.001, 0.01, 0.1 \right \}$).

And for all experiments, we use learning-rate decay of 0.992 per round. The batch size of the real data is set as 128, which also serves as the batch size of the virtual data. We use momentum-SGD as optimizers for all experiments, with momentum of 0.9, and weight decay of 0.0001. For different FL settings, when $K=10$ or $100$, and $E=1$, the maximum communication round is 1000. For $K=10$ and $E=5$, the maximum communication round is 400 (due to the $E=5$ increase the calculation cost). The number of clients selected for calculation is 5 per round for $K=10$, and 10 for $K=100$.

\begin{table*}[t]
\centering
\caption{Learning rate of all experiments.} 
\vspace{1pt}
\scriptsize{
\begin{tabular}{ccc|cccc}
\toprule[1.5pt]
  \!\! & &  & {\small FedAvg} & \small FedProx &\small SCAFFOLD &\small FedNova \\
\cmidrule[1.5pt]{1-7}
\multirow{8}{*}{CIFAR-10} &  \multirow{2}{*}{$a=0.1$, $E=1$, $K=10$}  &  Baselines          & 0.01 & 0.01 & 0.01 & 0.01  \\
        &        &  VHL          & 0.01 & 0.01 & 0.01 & 0.01  \\
\cmidrule[1pt]{2-7}
        &  \multirow{2}{*}{$a=0.05$, $E=1$, $K=10$}  &  Baselines          & 0.01 & 0.01 & 0.001 & 0.01  \\
        &        &  VHL          & 0.01 & 0.01 & 0.001 & 0.01  \\
\cmidrule[1pt]{2-7}
        &  \multirow{2}{*}{$a=0.1$, $E=5$, $K=10$}  &  Baselines          & 0.01 & 0.01 & 0.01 & 0.01  \\
        &        &  VHL          & 0.01 & 0.01 & 0.01 & 0.01  \\
\cmidrule[1pt]{2-7}
        &  \multirow{2}{*}{$a=0.1$, $E=1$, $K=100$}  &  Baselines          & 0.01 & 0.01 & 0.0001 & 0.001  \\
        &        &  VHL          & 0.01 & 0.01 & 0.0001 & 0.001  \\
\cmidrule[1.5pt]{1-7}
\multirow{8}{*}{FMNIST} &  \multirow{2}{*}{$a=0.1$, $E=1$, $K=10$}  &  Baselines          & 0.01 & 0.01 & 0.01 & 0.01  \\
        &        &  VHL          & 0.01 & 0.01 & 0.01 & 0.01  \\
\cmidrule[1pt]{2-7}
        &  \multirow{2}{*}{$a=0.05$, $E=1$, $K=10$}  &  Baselines          & 0.01 & 0.01 & 0.0001 & 0.0001  \\
        &        &  VHL          & 0.01 & 0.01 & 0.0001 & 0.001  \\
\cmidrule[1pt]{2-7}
        &  \multirow{2}{*}{$a=0.1$, $E=5$, $K=10$}  &  Baselines          & 0.01 & 0.01 & 0.01 & 0.01  \\
        &        &  VHL          & 0.01 & 0.01 & 0.01 & 0.01  \\
\cmidrule[1pt]{2-7}
        &  \multirow{2}{*}{$a=0.1$, $E=1$, $K=100$}  &  Baselines          & 0.01 & 0.01 & 0.001 & 0.001  \\
        &        &  VHL          & 0.01 & 0.01 & 0.001 & 0.001  \\
\cmidrule[1.5pt]{1-7}
\multirow{8}{*}{SVHN} &  \multirow{2}{*}{$a=0.1$, $E=1$, $K=10$}  &  Baselines          & 0.01 & 0.01 & 0.001 & 0.01  \\
        &        &  VHL          & 0.01 & 0.01 & 0.001 & 0.01  \\
\cmidrule[1pt]{2-7}
        &  \multirow{2}{*}{$a=0.05$, $E=1$, $K=10$}  &  Baselines          & 0.01 & 0.01 & 0.001 & 0.01  \\
        &        &  VHL          & 0.01 & 0.01 & 0.01 & 0.01  \\
\cmidrule[1pt]{2-7}
        &  \multirow{2}{*}{$a=0.1$, $E=5$, $K=10$}  &  Baselines          & 0.01 & 0.001 & 0.001 & 0.001  \\
        &        &  VHL          & 0.01 & 0.001 & 0.001 & 0.001  \\
\cmidrule[1pt]{2-7}
        &  \multirow{2}{*}{$a=0.1$, $E=1$, $K=100$}  &  Baselines          & 0.001 & 0.001 & 0.0001 & 0.0001  \\
        &        &  VHL          & 0.001 & 0.001 & 0.0001 & 0.001  \\
\cmidrule[1.5pt]{1-7}
\multirow{8}{*}{CIFAR-100} &  \multirow{2}{*}{$a=0.1$, $E=1$, $K=10$}  &  Baselines          & 0.01 & 0.01 & 0.01 & 0.01  \\
        &        &  VHL          & 0.01 & 0.01 & 0.01 & 0.01  \\
\cmidrule[1pt]{2-7}
        &  \multirow{2}{*}{$a=0.05$, $E=1$, $K=10$}  &  Baselines          & 0.01 & 0.01 & 0.01 & 0.01  \\
        &        &  VHL          & 0.01 & 0.01 & 0.01 & 0.01  \\
\cmidrule[1pt]{2-7}
        &  \multirow{2}{*}{$a=0.1$, $E=5$, $K=10$}  &  Baselines          & 0.01 & 0.01 & 0.01 & 0.01  \\
        &        &  VHL          & 0.01 & 0.01 & 0.01 & 0.01  \\
\cmidrule[1pt]{2-7}
        &  \multirow{2}{*}{$a=0.1$, $E=1$, $K=100$}  &  Baselines          & 0.01 & 0.01 & 0.01 & 0.01  \\
        &        &  VHL          & 0.01 & 0.01 & 0.01 & 0.01  \\
\bottomrule[1.5pt] 
\end{tabular}
}
\label{tab:LearningRate}
\end{table*}

\subsection{Implementation}

We sample the virtual data along with the natural data together. Therefore, in each epoch, the number of sampled virtual data is the $B_v / B_d$ times as much as the natural data, in which $B_v$ and $B_d$ are batch sizes of virtual data and natural data respectively.

We utilize the supervised contrastive loss~\citep{khosla2020supervised} to pull natural features toward the virtual features. In addition, to avoid that the virtual features moving towards the natural features, which may cause the homogeneous virtual features are biased by the natural features, we detach the virtual features before sending them into the supervised contrastive loss.

\subsection{Generating Virtual Dataset}\label{appendix:GenerateData}

% \subsection{How to Generate Virtual Dataset}

% We use an \textbf{un-pretrained} StyleGAN-v2~\citep{Karras2019stylegan2} without using any training data to generate the virtual dataset. Specifically, the server generates image outputs with latent styles and Gaussian noises to generate outputs through the StyleGAN-v2 model. The outputs are then distributed to all clients at the beginning of training. We will also explore other methods (i.e., generated with a simple CNN and up-sampling with pure Gaussian noises) to generate virtual datasets in Section~\ref{sec:DiffFacetsVHL}

We use $C_v$ to represent the number of classes of the virtual dataset. And $C_v=10$ for 10-classes datasets, $C_v=100$ for CIFAR-100\footnote{Note that the $C_v$, number of classes of virtual dataset, has no need to be equal to the $C_d$, the number of classes of the natural dataset. In this paper, we set $C_v = C_d$ for convenience.}.

No matter which generating method we use, we sample noise from $C_v$  different Gaussian distributions with different means but the same standard deviation, representing different classes of the virtual data. The server collects the outputs of the StyleGAN-v2, and then distributes them to all clients at the beginning. Thus, there is no extra calculation of clients to generate datasets.

For the upsampling from the pure noise, we firstly sample noise point as a $8 \times 8$ size image, and conduct upsampling to attain a $32 \times 32$ size image. Through upsampling, these noise images could form some color blocks, instead of pure noise points in the images. Thus, the client models could learn some basic vision features from them.

For the simple CNN generator, we sample the 64-dim noises and feed them into the CNN. Then this simple generator outputs the $32 \times 32$ size images. This simple CNN consists of for transpose convolutional layers for upsampling and 3 convolutional layers for downsampling. Because an untrained simple CNN does not have enough output diversity, we use 10 CNNs with different initial weights, to generate the virtual dataset as the data of 10 different classes.

We display these generated virtual datasets as Figure~\ref{fig:NoiseStyleGan32x32c10}, ~\ref{fig:NoiseStyleGan32x32c100}, ~\ref{fig:Gaussian_Noise} and ~\ref{fig:cifar_conv_decoder}. One can see that different images of different labels are easy to be distinguished. And the images generated by Style-GAN of the same class are the most diverse.

\begin{figure*}[htb!] 
    \setlength{\abovedisplayskip}{-2pt}
    \subfigbottomskip=-1pt
    \subfigcapskip=1pt
    \setlength{\abovecaptionskip}{-2pt}
   \centering
    \subfigure{\includegraphics[width=0.19\linewidth]{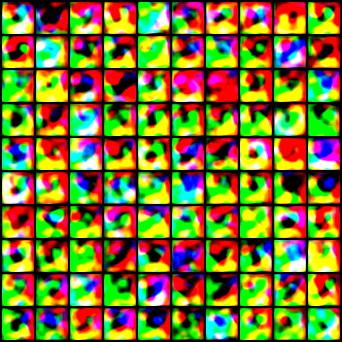}}
    \subfigure{\includegraphics[width=0.19\linewidth]{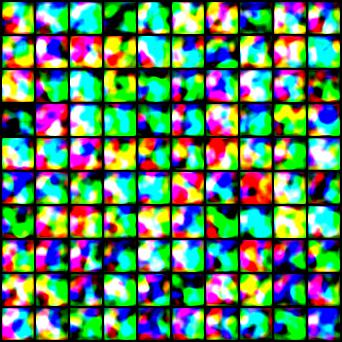}}
    \subfigure{\includegraphics[width=0.19\linewidth]{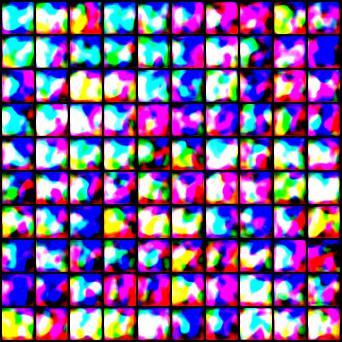}}
    \subfigure{\includegraphics[width=0.19\linewidth]{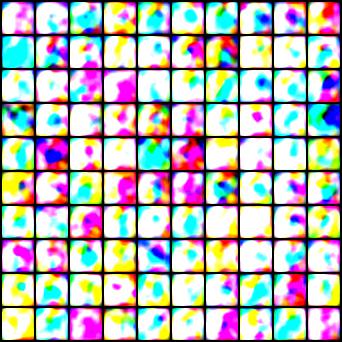}}
    \subfigure{\includegraphics[width=0.19\linewidth]{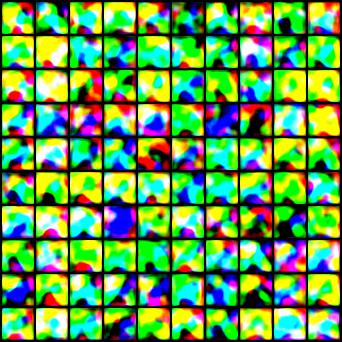}}
    \subfigure{\includegraphics[width=0.19\linewidth]{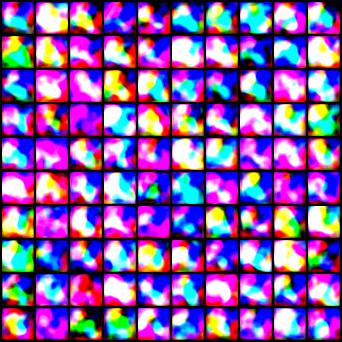}}
    \subfigure{\includegraphics[width=0.19\linewidth]{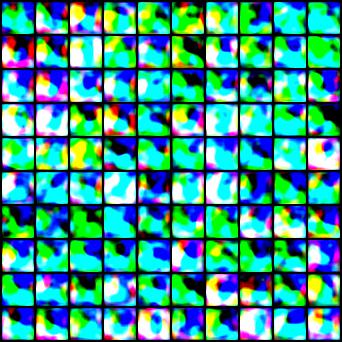}}
    \subfigure{\includegraphics[width=0.19\linewidth]{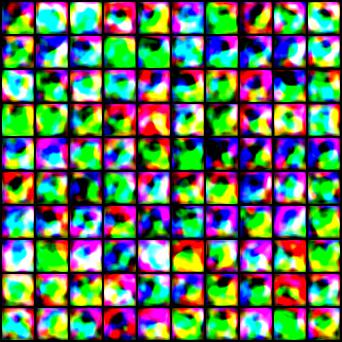}}
    \subfigure{\includegraphics[width=0.19\linewidth]{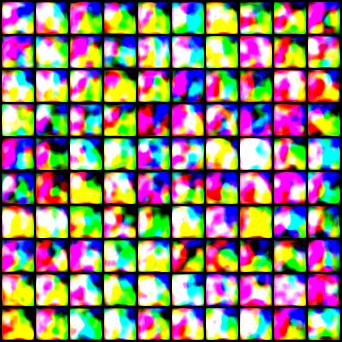}}
    \subfigure{\includegraphics[width=0.19\linewidth]{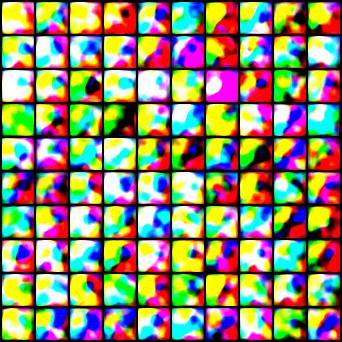}}
    \caption{Generated Virtual Data for CIFAR-10, FMNIST and SVHN. Each figure has 90 data samples of the same class.} 
    \label{fig:NoiseStyleGan32x32c10}
\vspace{-0.3cm}
\end{figure*}

\begin{figure*}[htb!] 
    \setlength{\abovedisplayskip}{-2pt}
    \setlength{\abovecaptionskip}{-2pt}
    \subfigbottomskip=-1pt
    \subfigcapskip=1pt
   \centering
    \subfigure{\includegraphics[width=0.19\linewidth]{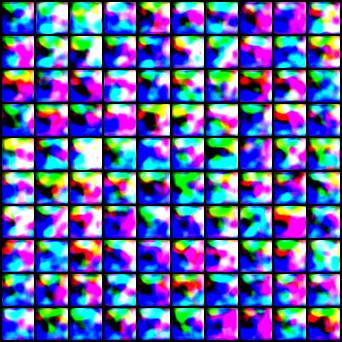}}
    \subfigure{\includegraphics[width=0.19\linewidth]{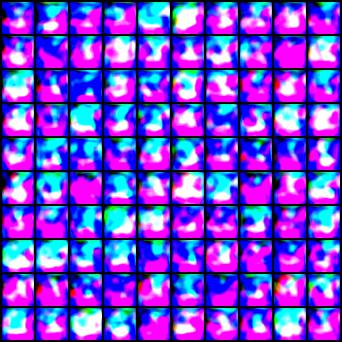}}
    \subfigure{\includegraphics[width=0.19\linewidth]{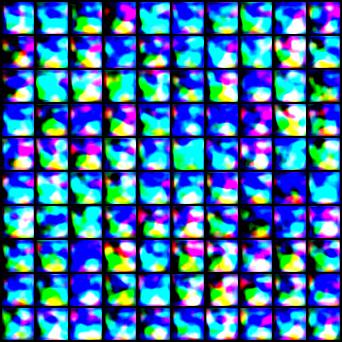}}
    \subfigure{\includegraphics[width=0.19\linewidth]{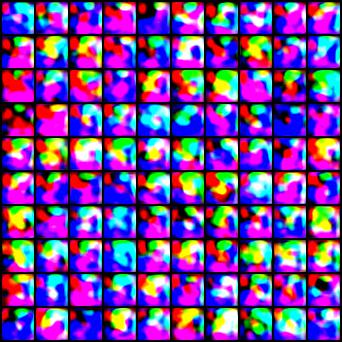}}
    \subfigure{\includegraphics[width=0.19\linewidth]{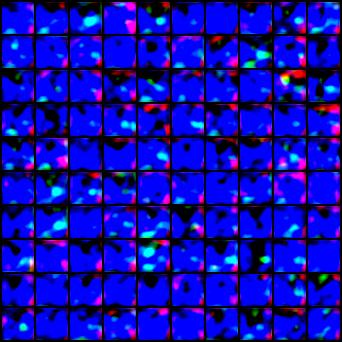}}
    \subfigure{\includegraphics[width=0.19\linewidth]{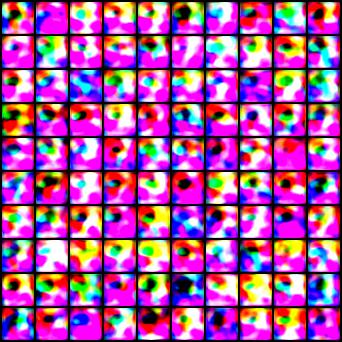}}
    \subfigure{\includegraphics[width=0.19\linewidth]{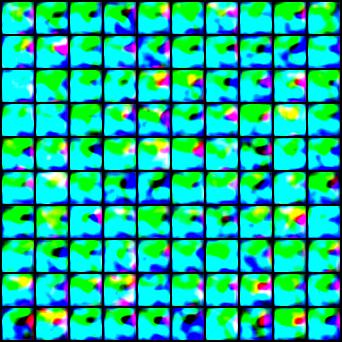}}
    \subfigure{\includegraphics[width=0.19\linewidth]{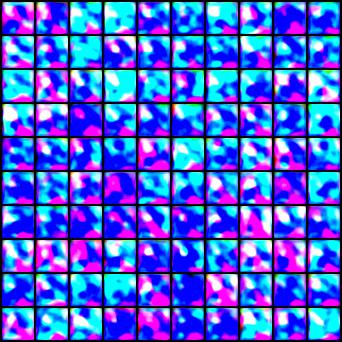}}
    \subfigure{\includegraphics[width=0.19\linewidth]{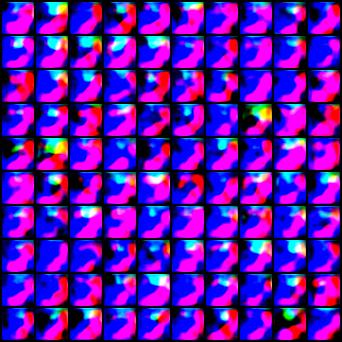}}
    \subfigure{\includegraphics[width=0.19\linewidth]{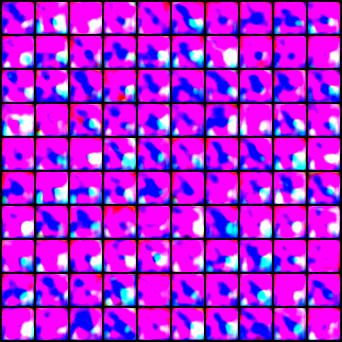}}
    % \subfigure{\includegraphics[width=0.19\linewidth]{StyleGan32x32c100/Distribution-10-overview.jpg}}
    % \subfigure{\includegraphics[width=0.19\linewidth]{StyleGan32x32c100/Distribution-11-overview.jpg}}
    % \subfigure{\includegraphics[width=0.19\linewidth]{StyleGan32x32c100/Distribution-12-overview.jpg}}
    % \subfigure{\includegraphics[width=0.19\linewidth]{StyleGan32x32c100/Distribution-13-overview.jpg}}
    % \subfigure{\includegraphics[width=0.19\linewidth]{StyleGan32x32c100/Distribution-14-overview.jpg}}
    % \subfigure{\includegraphics[width=0.19\linewidth]{StyleGan32x32c100/Distribution-15-overview.jpg}}
    % \subfigure{\includegraphics[width=0.19\linewidth]{StyleGan32x32c100/Distribution-16-overview.jpg}}
    % \subfigure{\includegraphics[width=0.19\linewidth]{StyleGan32x32c100/Distribution-17-overview.jpg}}
    % \subfigure{\includegraphics[width=0.19\linewidth]{StyleGan32x32c100/Distribution-18-overview.jpg}}
    % \subfigure{\includegraphics[width=0.19\linewidth]{StyleGan32x32c100/Distribution-19-overview.jpg}}
    \caption{Generated Virtual Data for CIFAR-100. Due to space limitation, we only show 10 out of 100 classes.} 
    \label{fig:NoiseStyleGan32x32c100} 
\vspace{-0.3cm}
\end{figure*}

\begin{figure*}[htb!] 
    \setlength{\abovedisplayskip}{-2pt}
    \setlength{\abovecaptionskip}{-2pt}
    \subfigbottomskip=-1pt
    \subfigcapskip=1pt
   \centering
    \subfigure{\includegraphics[width=0.19\linewidth]{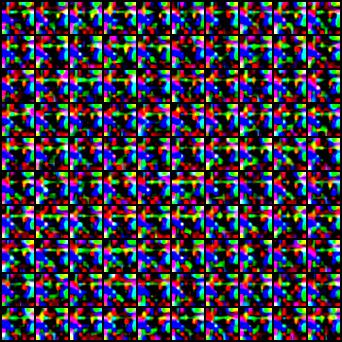}}
    \subfigure{\includegraphics[width=0.19\linewidth]{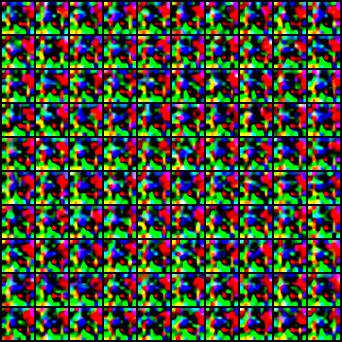}}
    \subfigure{\includegraphics[width=0.19\linewidth]{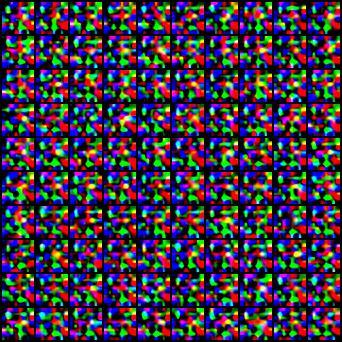}}
    \subfigure{\includegraphics[width=0.19\linewidth]{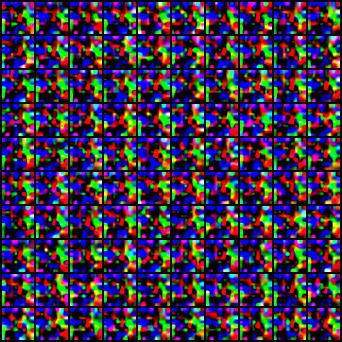}}
    \subfigure{\includegraphics[width=0.19\linewidth]{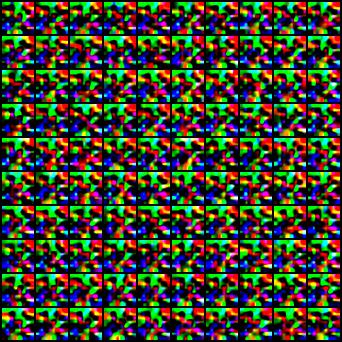}}
    \subfigure{\includegraphics[width=0.19\linewidth]{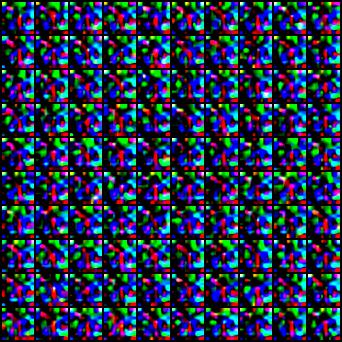}}
    \subfigure{\includegraphics[width=0.19\linewidth]{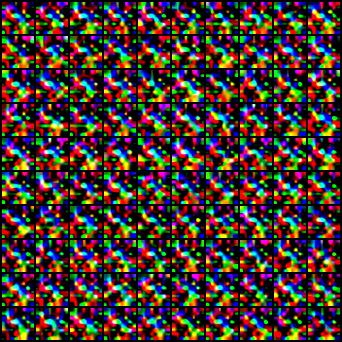}}
    \subfigure{\includegraphics[width=0.19\linewidth]{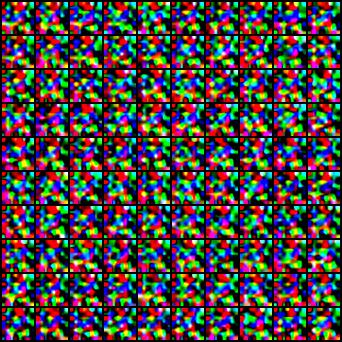}}
    \subfigure{\includegraphics[width=0.19\linewidth]{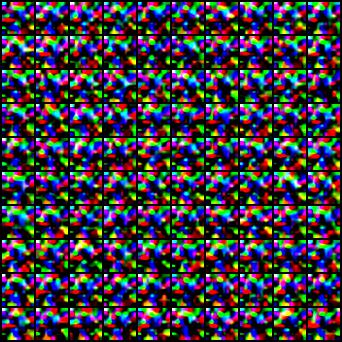}}
    \subfigure{\includegraphics[width=0.19\linewidth]{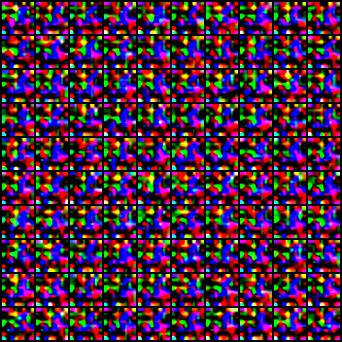}}
    \caption{Generated Gaussian Noise with upsampling.} 
    \label{fig:Gaussian_Noise} 
\vspace{-0.3cm}
\end{figure*}

\begin{figure*}[htb!] 
    \setlength{\abovedisplayskip}{-2pt}
    \setlength{\abovecaptionskip}{-2pt}
    \subfigbottomskip=-1pt
    \subfigcapskip=1pt
   \centering
    \subfigure{\includegraphics[width=0.19\linewidth]{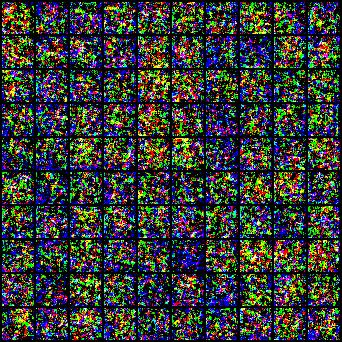}}
    \subfigure{\includegraphics[width=0.19\linewidth]{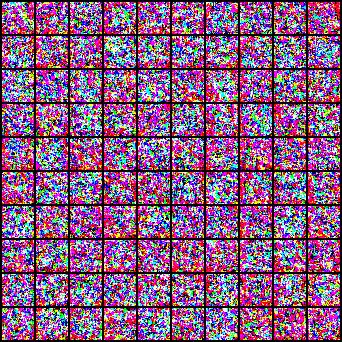}}
    \subfigure{\includegraphics[width=0.19\linewidth]{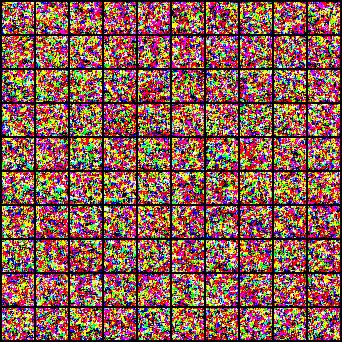}}
    \subfigure{\includegraphics[width=0.19\linewidth]{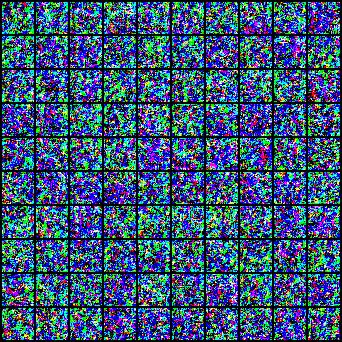}}
    \subfigure{\includegraphics[width=0.19\linewidth]{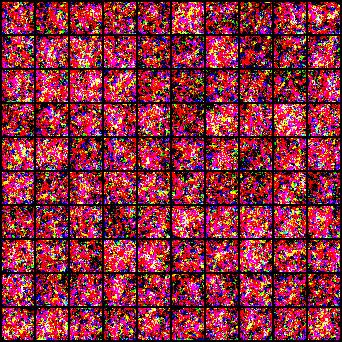}}
    \subfigure{\includegraphics[width=0.19\linewidth]{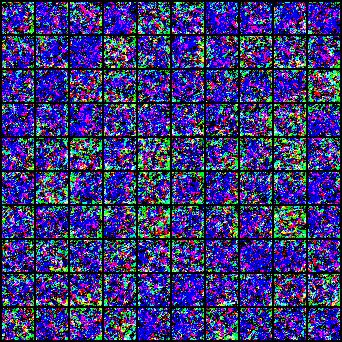}}
    \subfigure{\includegraphics[width=0.19\linewidth]{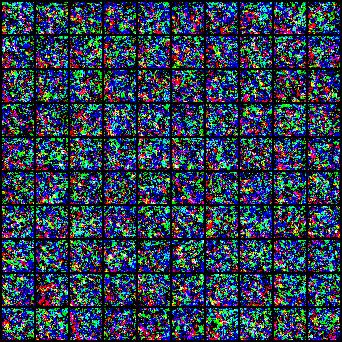}}
    \subfigure{\includegraphics[width=0.19\linewidth]{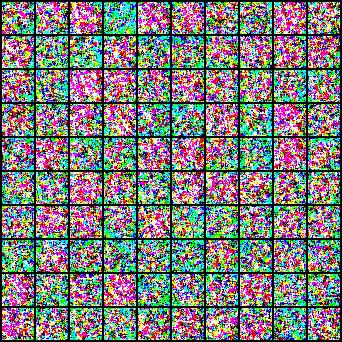}}
    \subfigure{\includegraphics[width=0.19\linewidth]{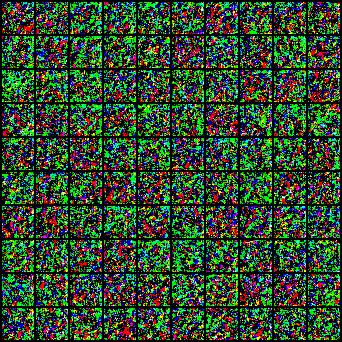}}
    \subfigure{\includegraphics[width=0.19\linewidth]{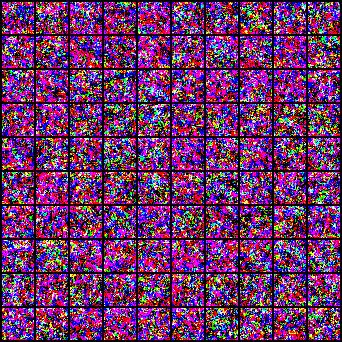}}
    \caption{Generated virtual dataset using ten simple decoders.} 
    \label{fig:cifar_conv_decoder} 
\vspace{-0.3cm}
\end{figure*}

We show all convergence curves in Figure~\ref{fig:Convergence-cifar10}, ~\ref{fig:Convergence-fmnist}, ~\ref{fig:Convergence-SVHN} and ~\ref{fig:Convergence-CIFAR100}. One can see that nearly all experiments could benefit convergence acceleration from VHL.

\section{More Experiment Results and Discussion}

\subsection{Feature Shifts between Clients}\label{appendix:VisualizationFeature}

We show more t-SNE visualization of feature distribution between client models in Figure~\ref{fig:Apppendix-Clients-Feature-fedavg}, ~\ref{fig:Apppendix-Clients-Feature-NaiveVHL} and ~\ref{fig:Apppendix-Clients-Feature-VHL}, and them of the server model in Figure~\ref{fig:Apppendix-Server-Feature-fedavg}, ~\ref{fig:Apppendix-Server-Feature-NaiveVHL} and ~\ref{fig:Apppendix-Server-Feature-VHL}.

One can see that the clients can quickly learn to distinguish virtual data and attain the consensus on the feature expression of the virtual data of the same class (see Figure~\ref{fig:Apppendix-Clients-Feature-NaiveVHL} (a) and Figure~\ref{fig:Apppendix-Clients-Feature-VHL} (a)). Therefore, the subsequent local training based on the homogeneous virtual feature could benefit the consistent expression on the natural data.

One interesting observation is that only training with the virtual data also benefits FL. And clients also can more quickly learn consistent feature expression though such Naive VHL than through FedAvg (see  Figure~\ref{fig:Apppendix-Clients-Feature-NaiveVHL} (c) and Figure~\ref{fig:Apppendix-Clients-Feature-VHL} (c)).

Another interesting finding is that, after local training, the features from one client are easily to be clustered together by t-SNE (see Figure~\ref{fig:Apppendix-Clients-Feature-fedavg} (a) and (b)), which means that these client models are greatly different. This happens even at the $999$-th communication round for FedAvg (see Figure~\ref{fig:Apppendix-Clients-Feature-fedavg} (d)).

\begin{figure*}[htb!] 
    \setlength{\abovedisplayskip}{-2pt}
    \setlength{\abovecaptionskip}{-2pt}
    \subfigbottomskip=-1pt
    \subfigcapskip=1pt
  \centering
    \subfigure[$9$-th round. ]{\includegraphics[width=0.24\linewidth]{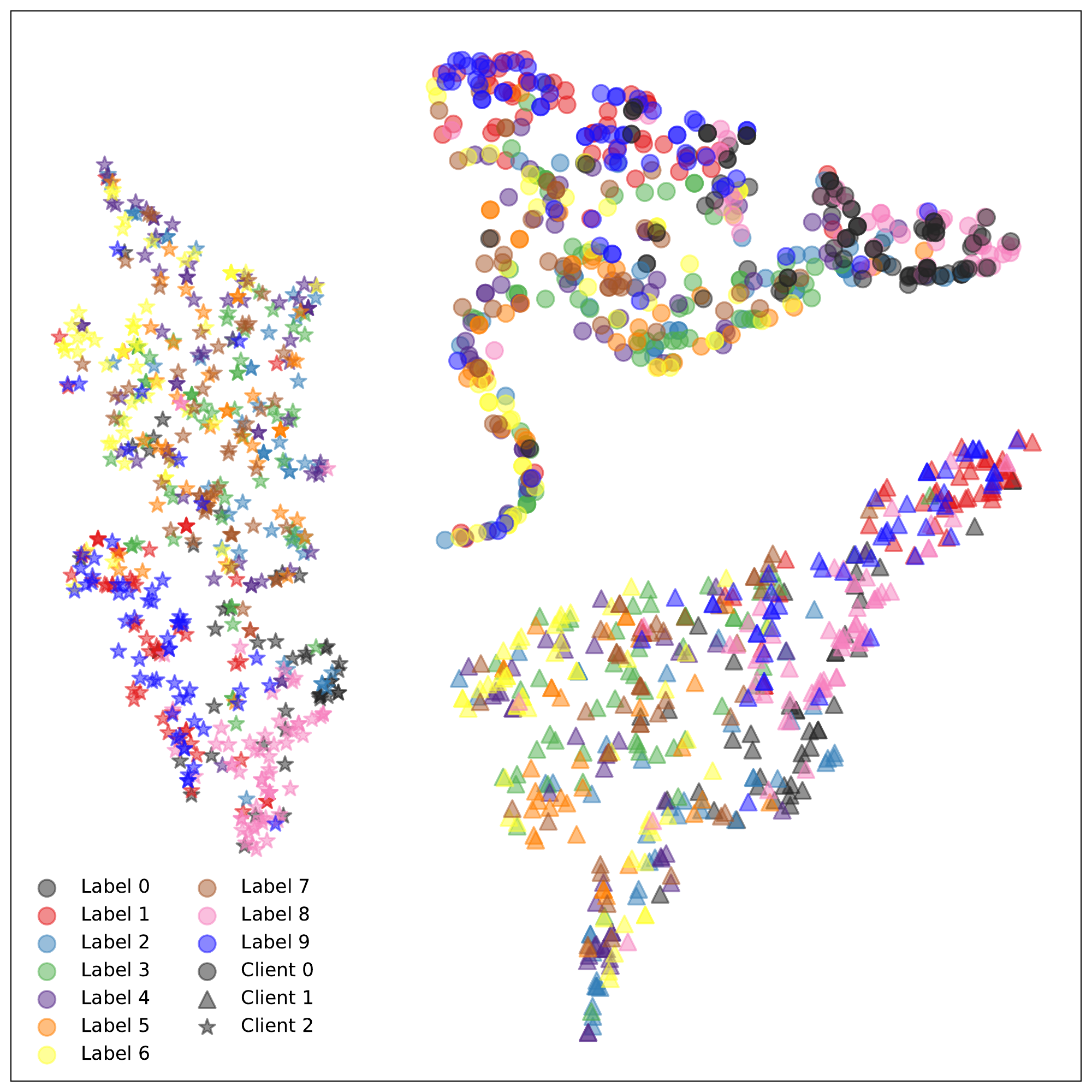}}
    \subfigure[$99$-th round. ]{\includegraphics[width=0.24\linewidth]{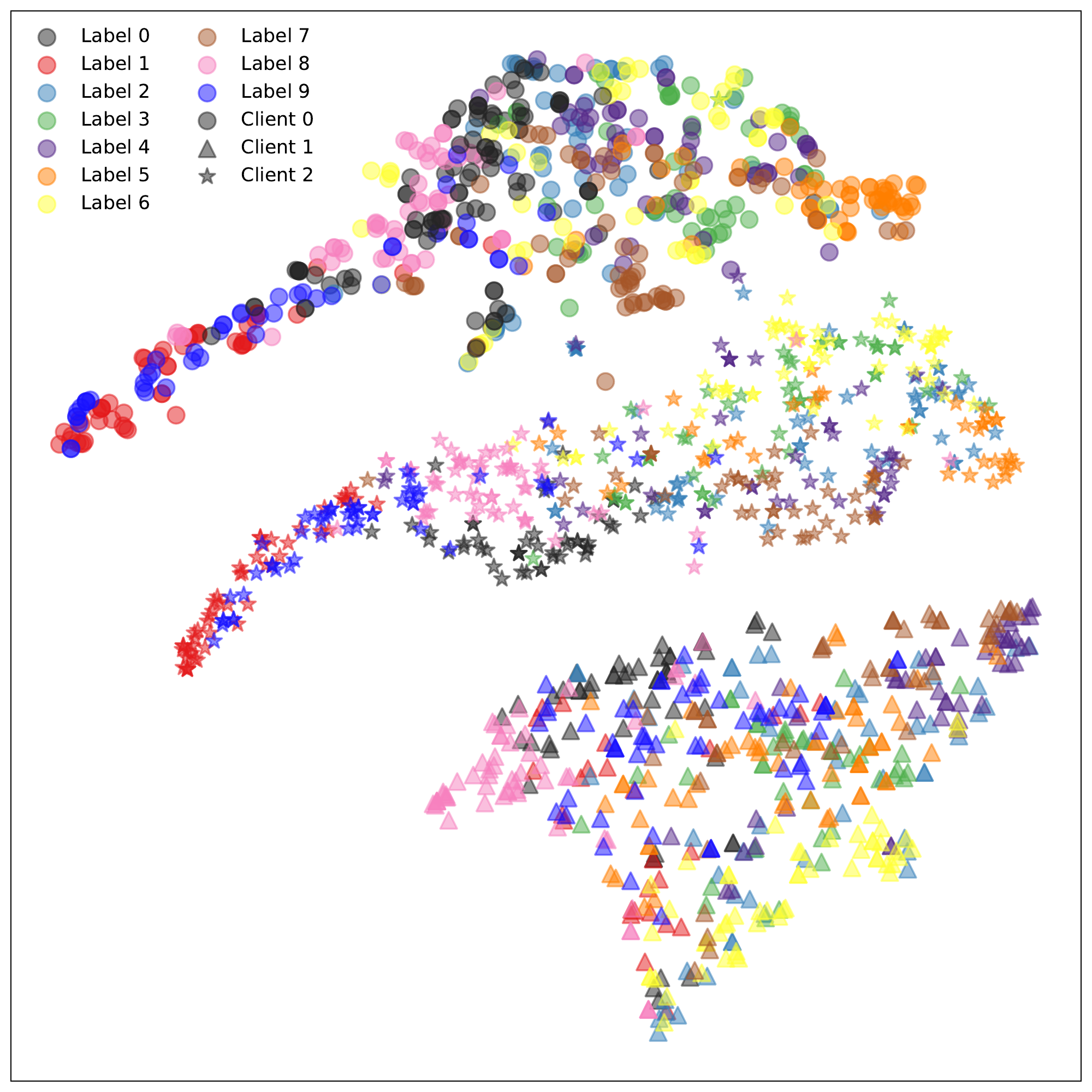}}
    \subfigure[$299$-th round. ]{\includegraphics[width=0.24\linewidth]{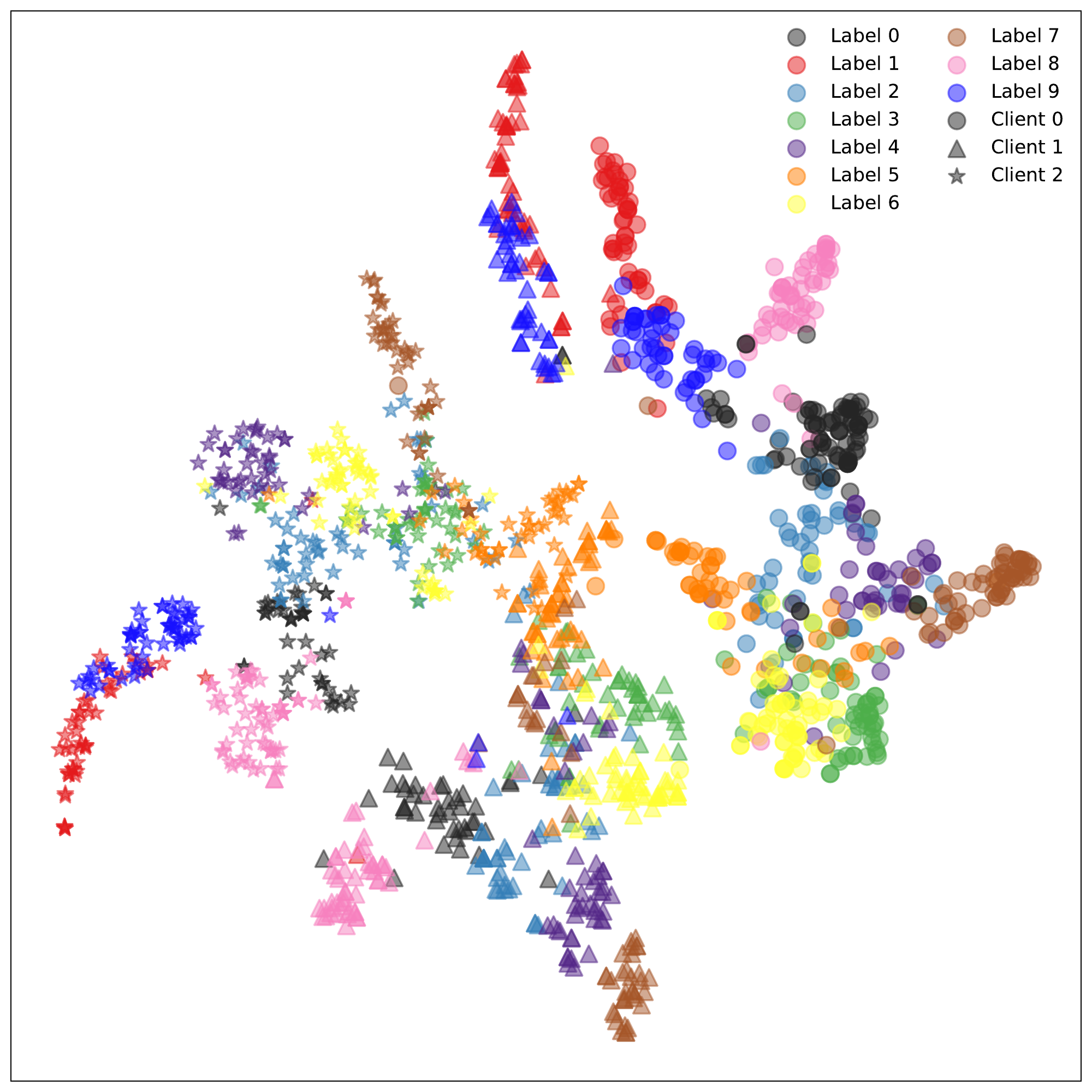}}
    \subfigure[$999$-th round. ]{\includegraphics[width=0.24\linewidth]{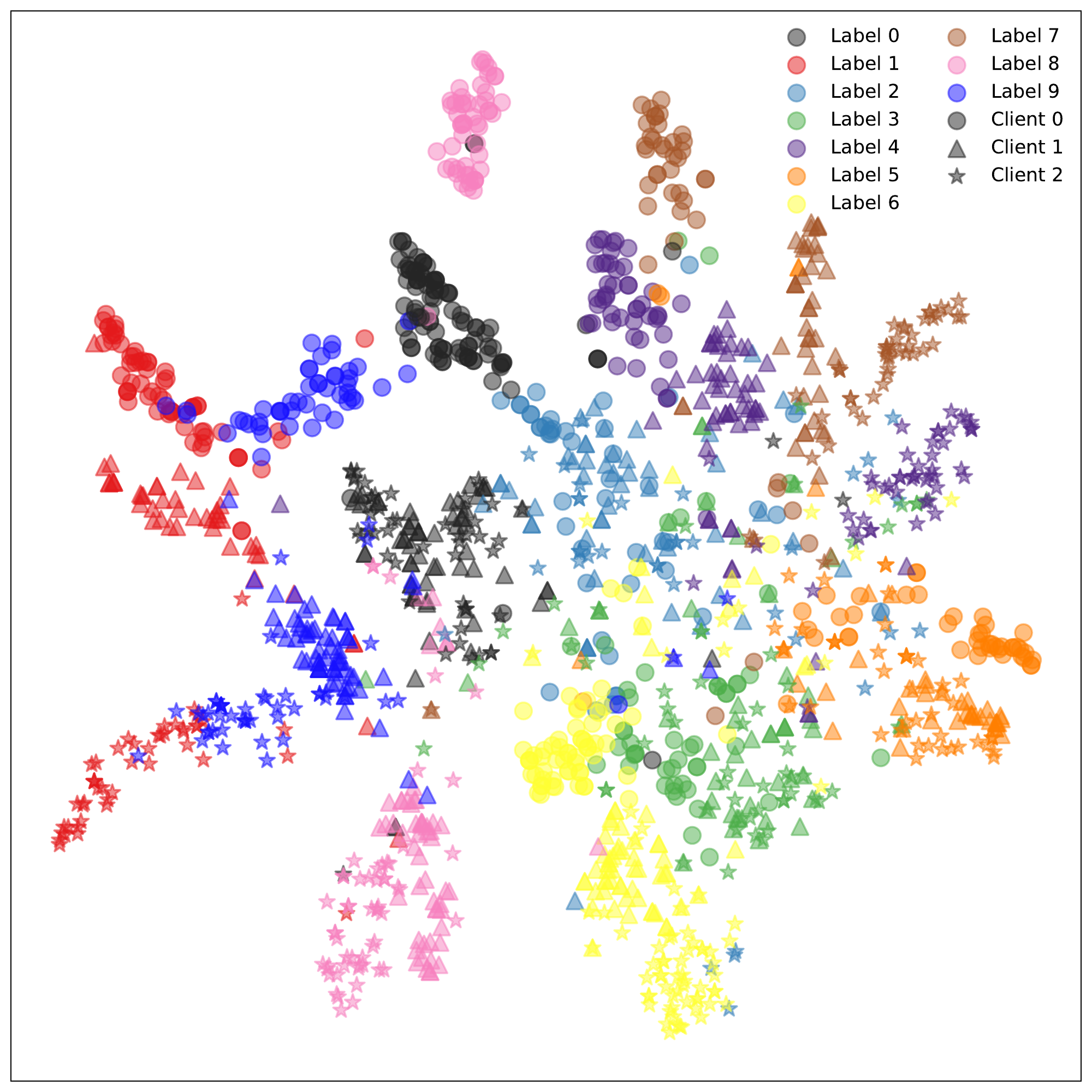}}
   \caption{Features of CIFAR10 test data of different client models \textbf{without} VHL.}
\label{fig:Apppendix-Clients-Feature-fedavg} 
\vspace{-0.3cm}
\end{figure*}
% \vspace{-0.3cm}

\begin{figure*}[htb!] 
    \setlength{\abovedisplayskip}{-2pt}
    \setlength{\abovecaptionskip}{-2pt}
    \subfigbottomskip=-1pt
    \subfigcapskip=1pt
  \centering
    \subfigure[$9$-th round. ]{\includegraphics[width=0.24\linewidth]{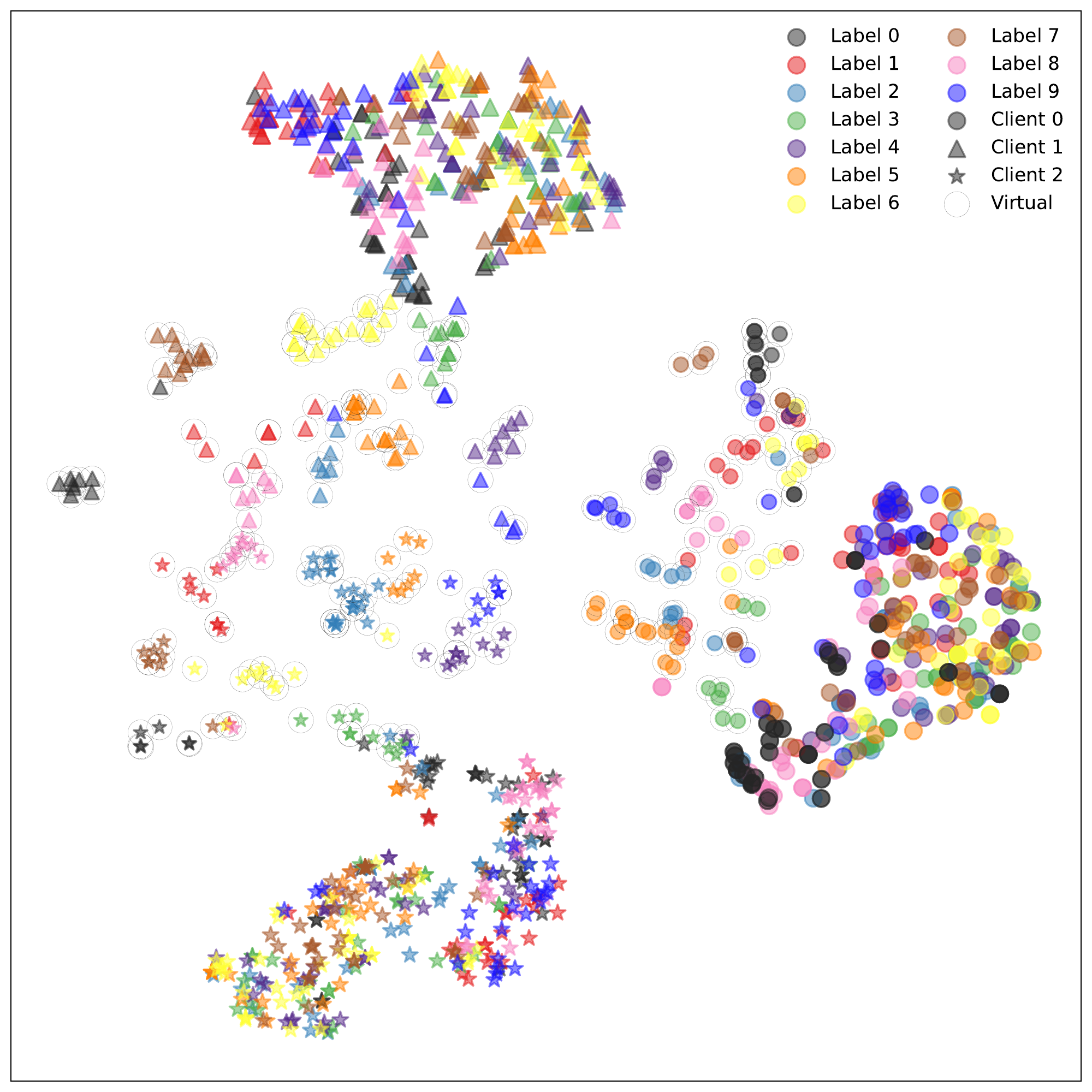}}
    \subfigure[$99$-th round. ]{\includegraphics[width=0.24\linewidth]{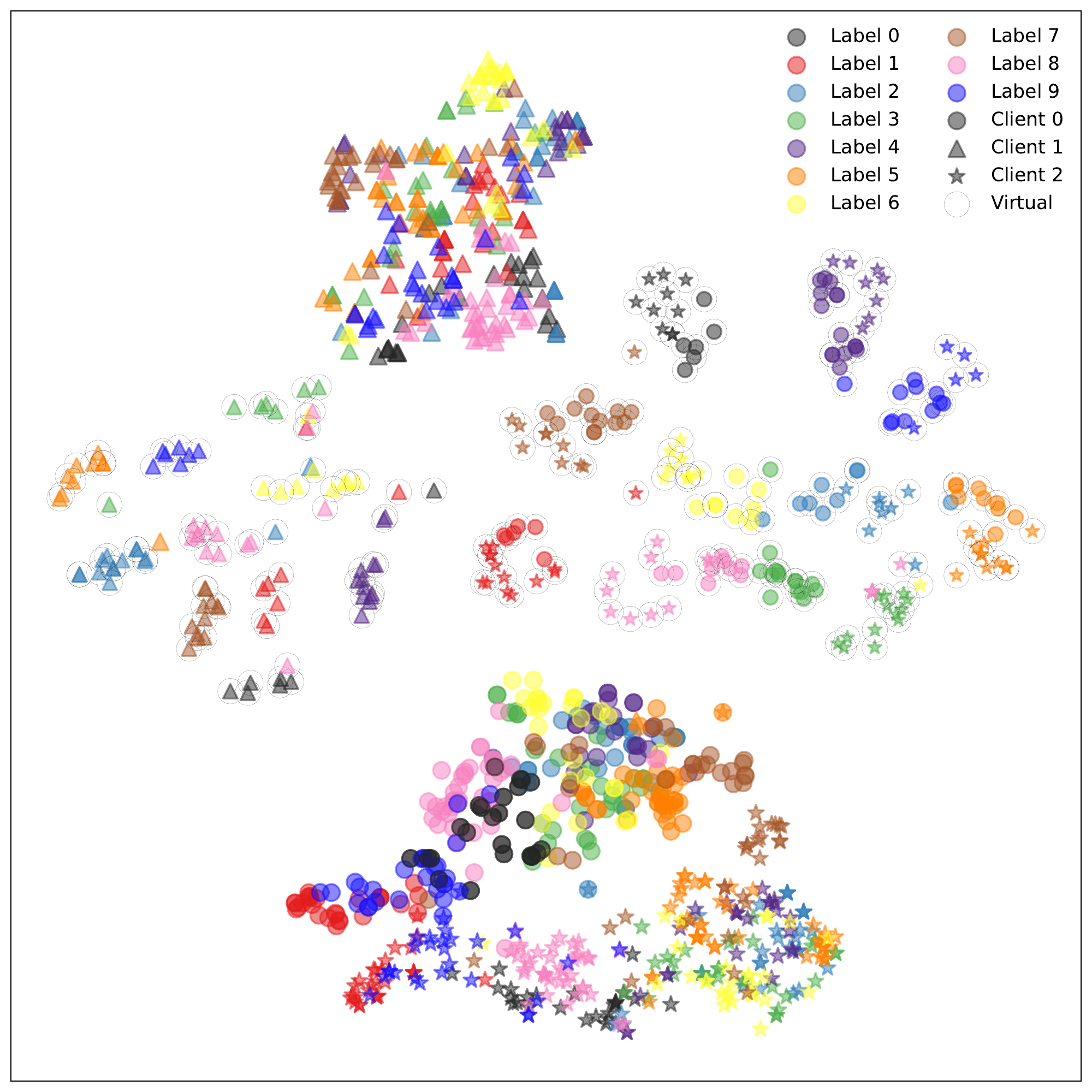}}
    \subfigure[$299$-th round. ]{\includegraphics[width=0.24\linewidth]{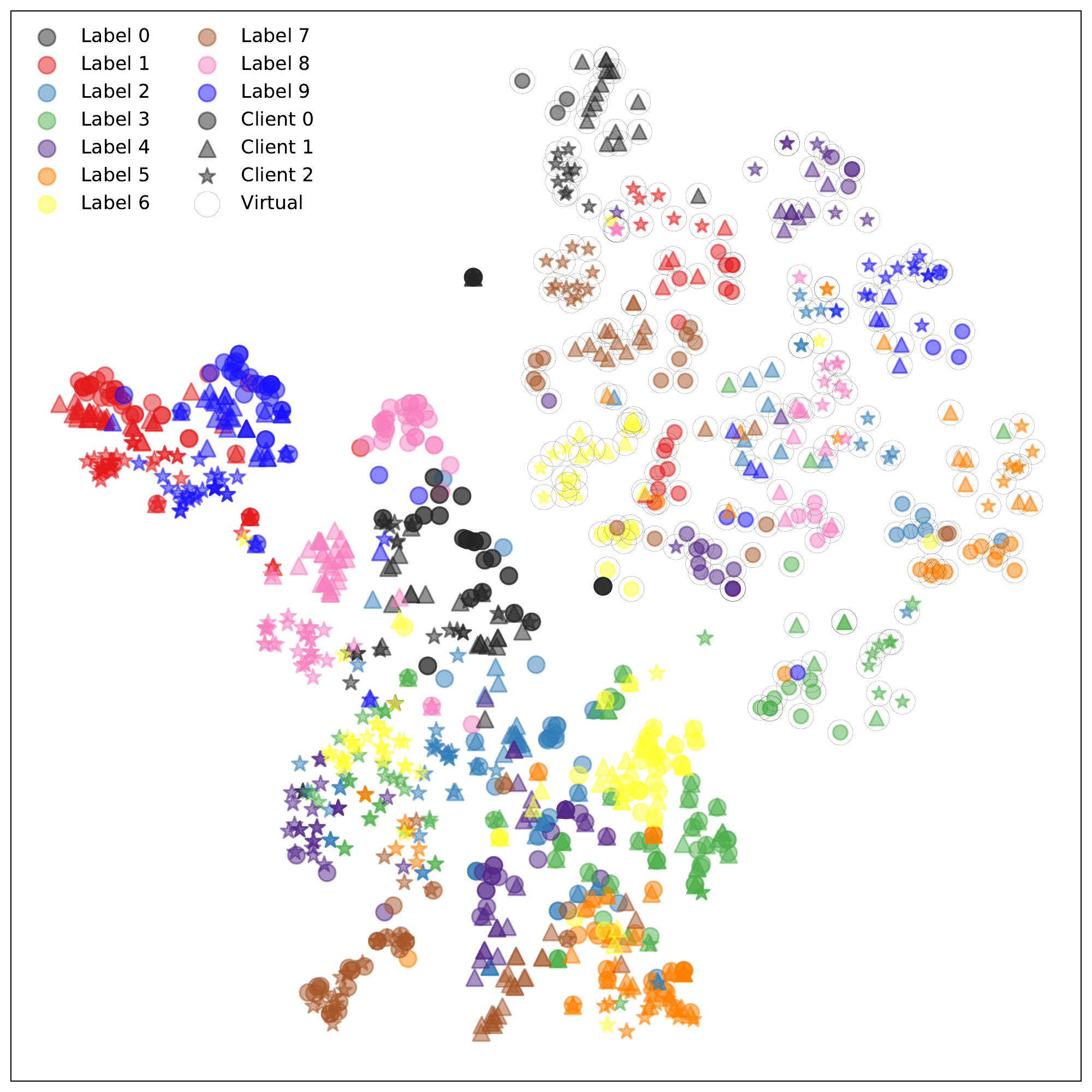}}
    \subfigure[at $999$-th round. ]{\includegraphics[width=0.24\linewidth]{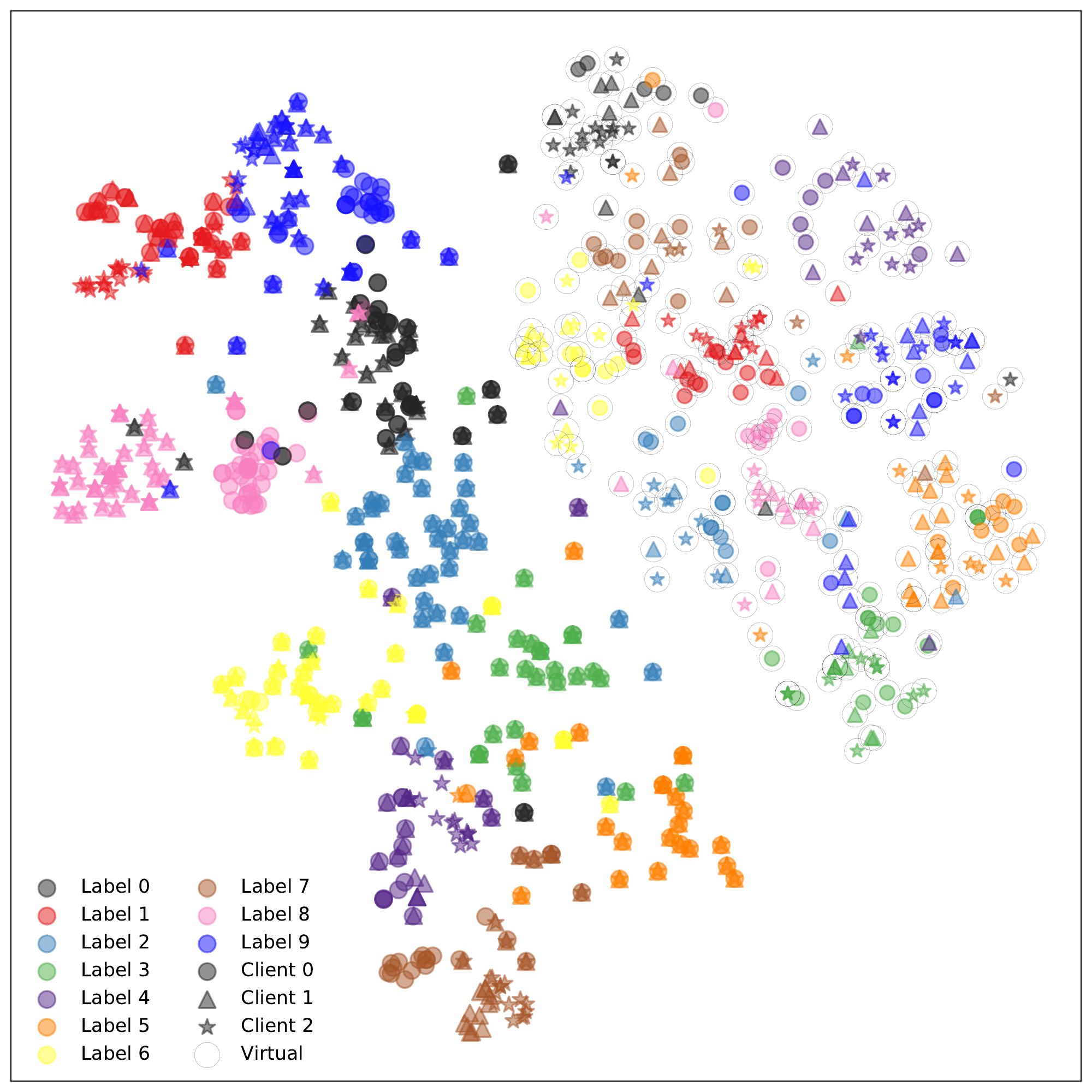}}
   \caption{Features of CIFAR10 test data of different client models \textbf{with} Naive VHL.}
\label{fig:Apppendix-Clients-Feature-NaiveVHL} 
\vspace{-0.3cm}
\end{figure*}

\begin{figure*}[htb!] 
    \setlength{\abovedisplayskip}{-2pt}
    \setlength{\abovecaptionskip}{-2pt}
    \subfigbottomskip=-1pt
    \subfigcapskip=1pt
  \centering
    \subfigure[$9$-th round. ]{\includegraphics[width=0.24\linewidth]{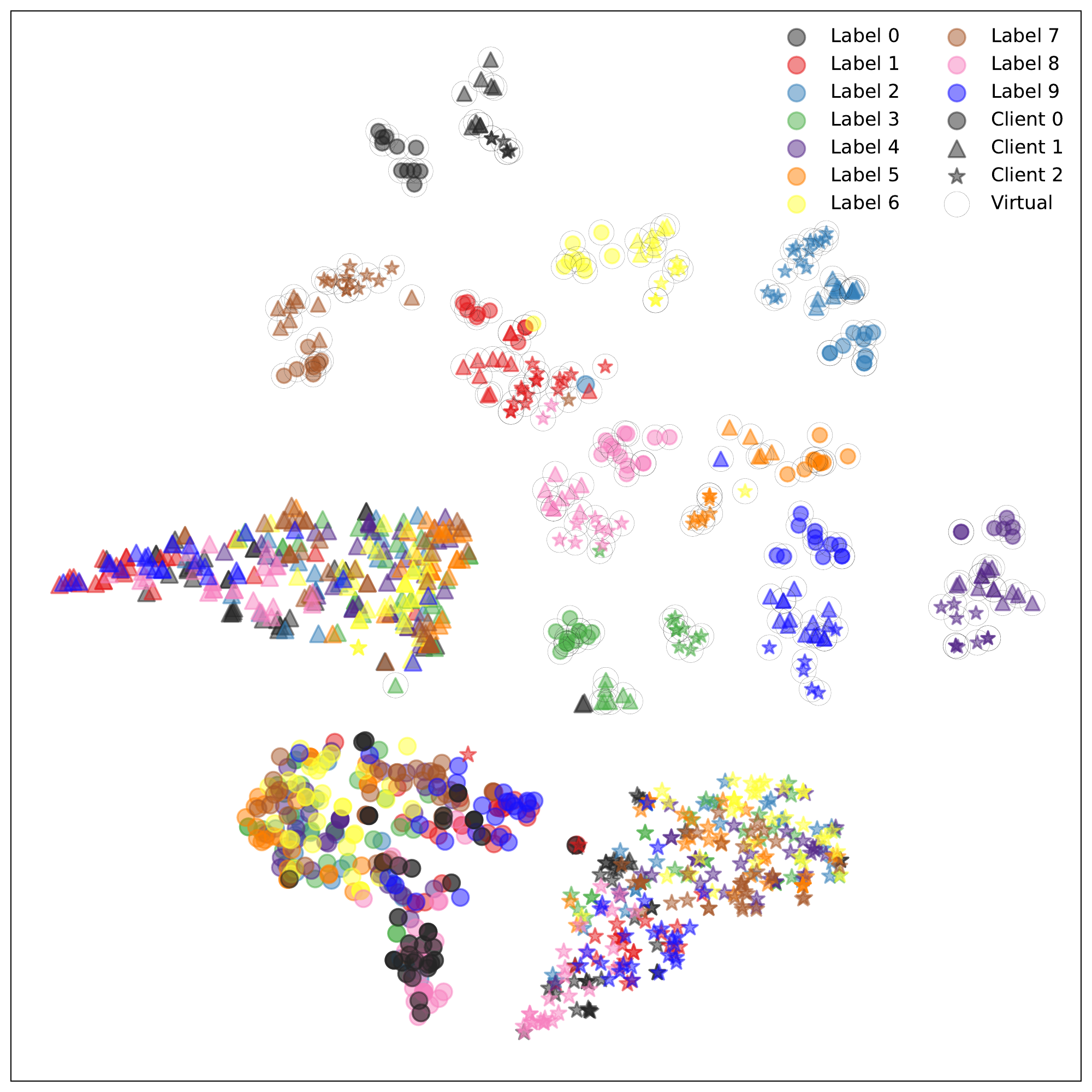}}
    \subfigure[$99$-th round. ]{\includegraphics[width=0.24\linewidth]{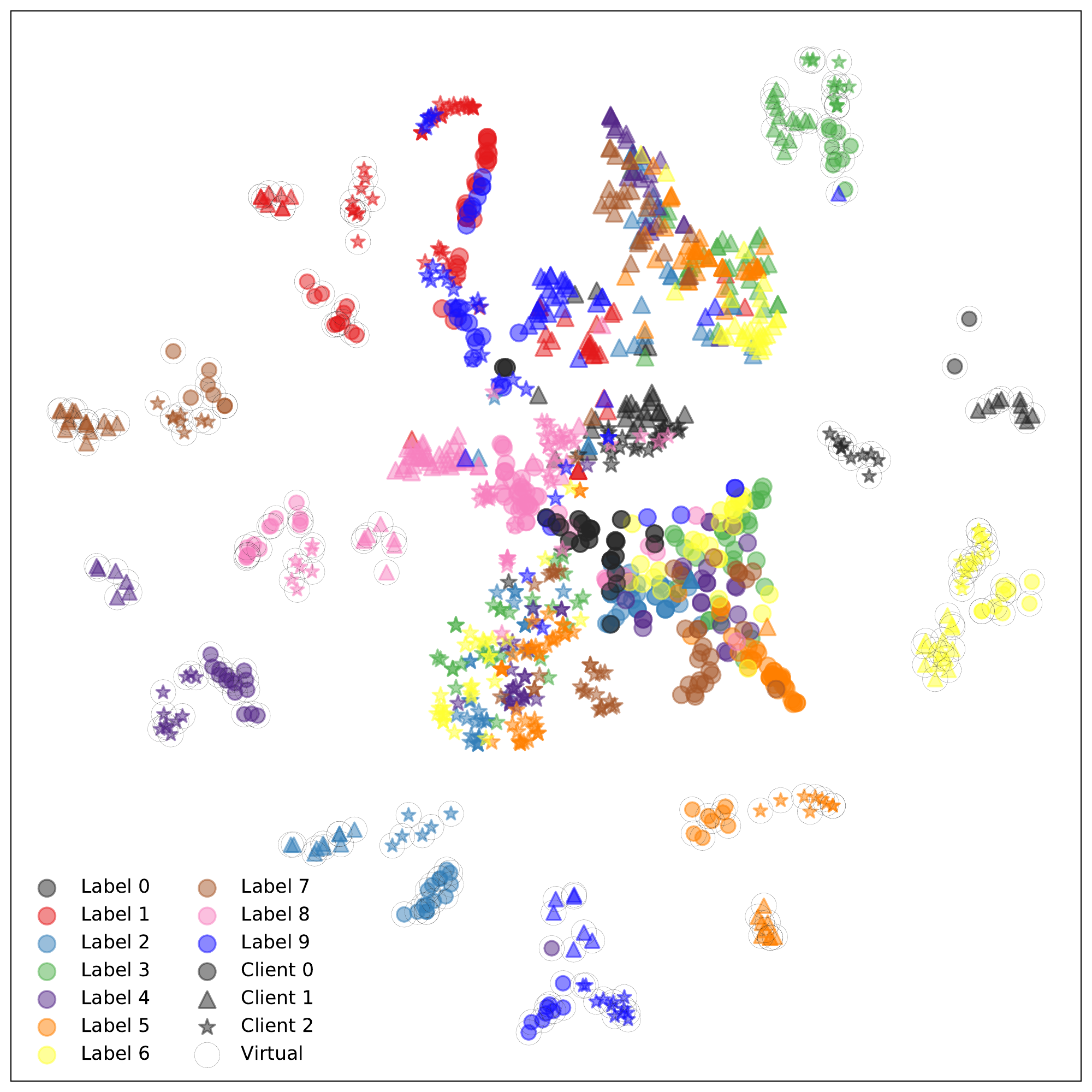}}
    \subfigure[$299$-th round. ]{\includegraphics[width=0.24\linewidth]{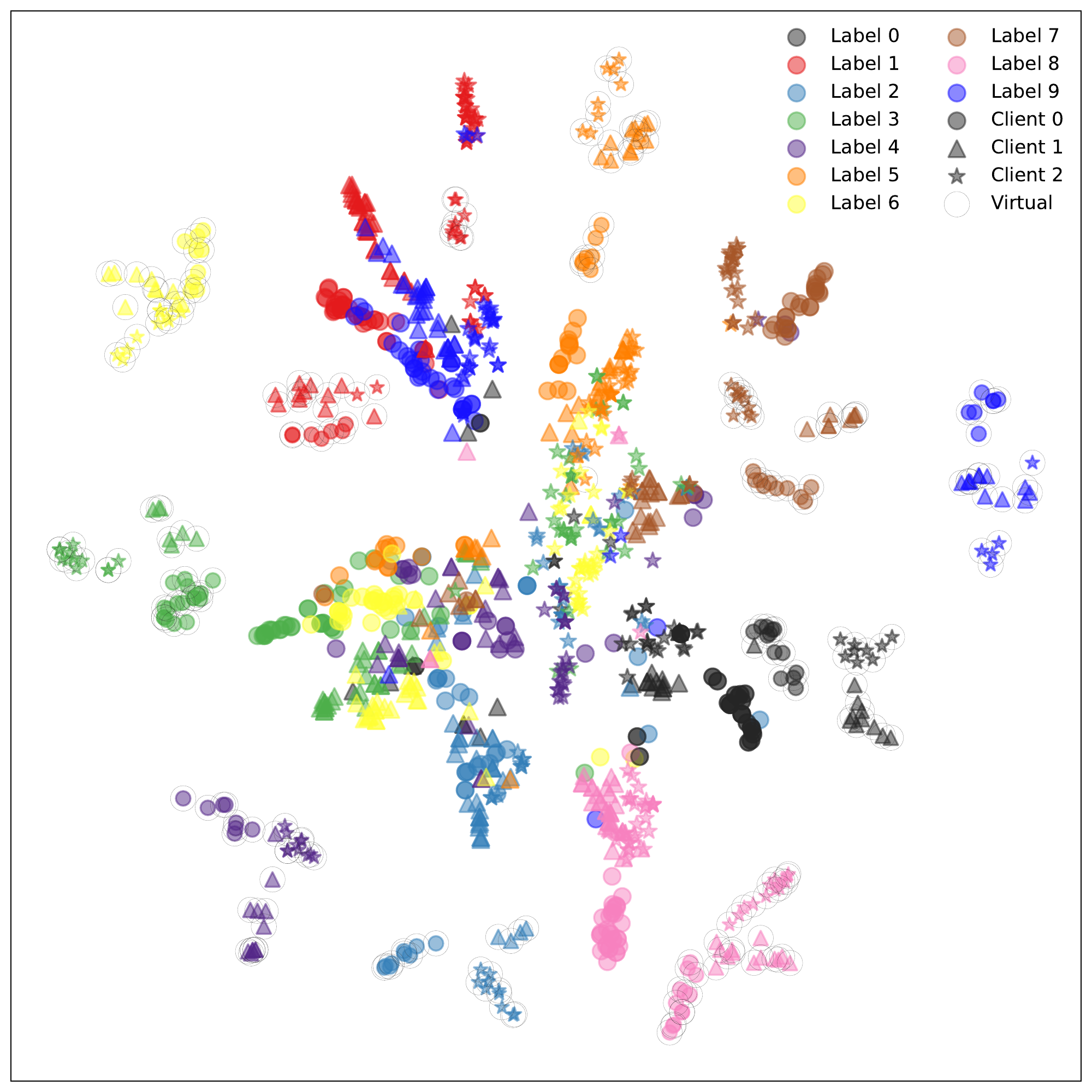}}
    \subfigure[at $999$-th round. ]{\includegraphics[width=0.24\linewidth]{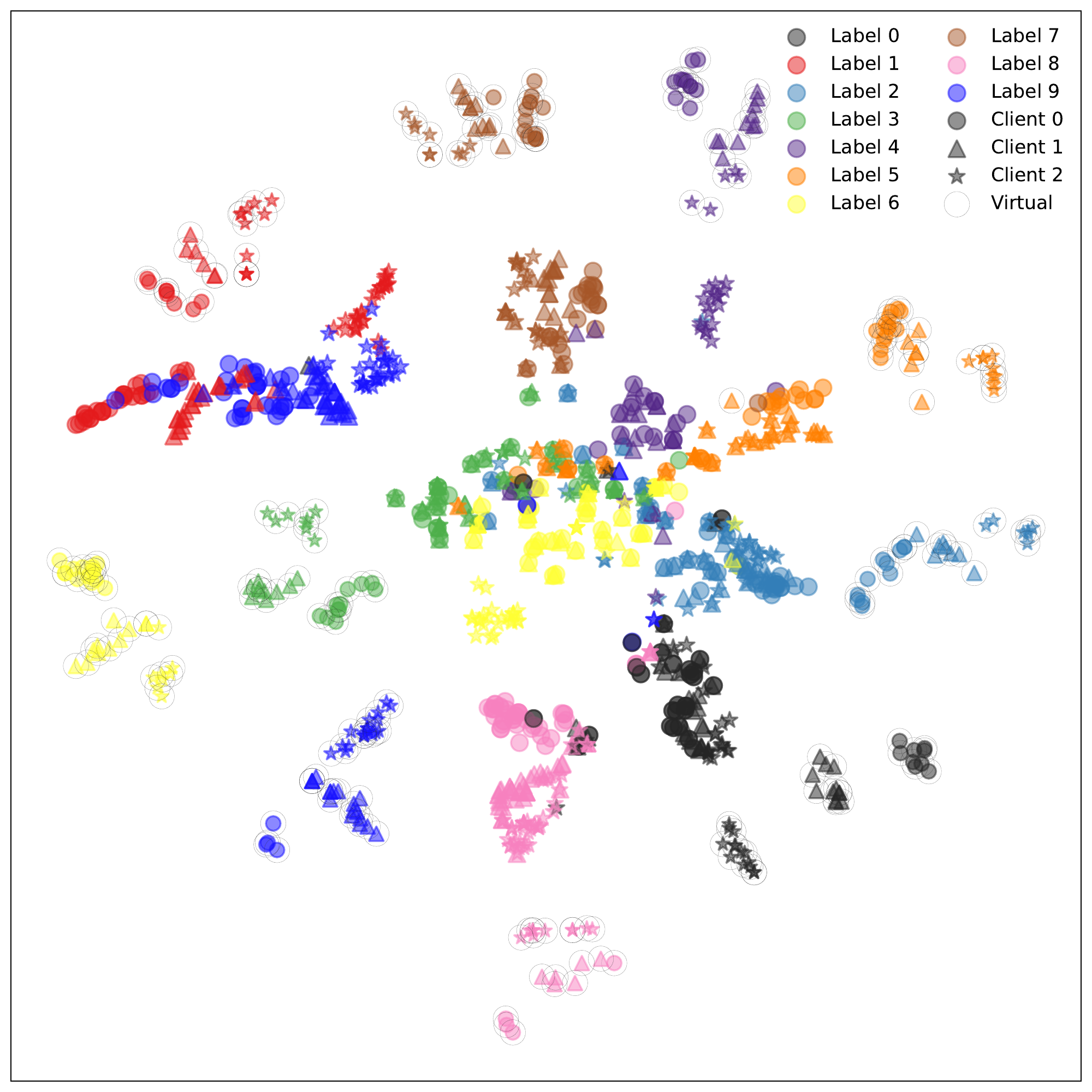}}
   \caption{Features of CIFAR10 test data of different client models \textbf{with} VHL.}
\label{fig:Apppendix-Clients-Feature-VHL} 
\vspace{-0.3cm}
\end{figure*}

\begin{figure*}[htb!] 
    \setlength{\abovedisplayskip}{-2pt}
    \setlength{\abovecaptionskip}{-2pt}
    \subfigbottomskip=-1pt
    \subfigcapskip=1pt
  \centering
    \subfigure[$9$-th round. ]{\includegraphics[width=0.24\linewidth]{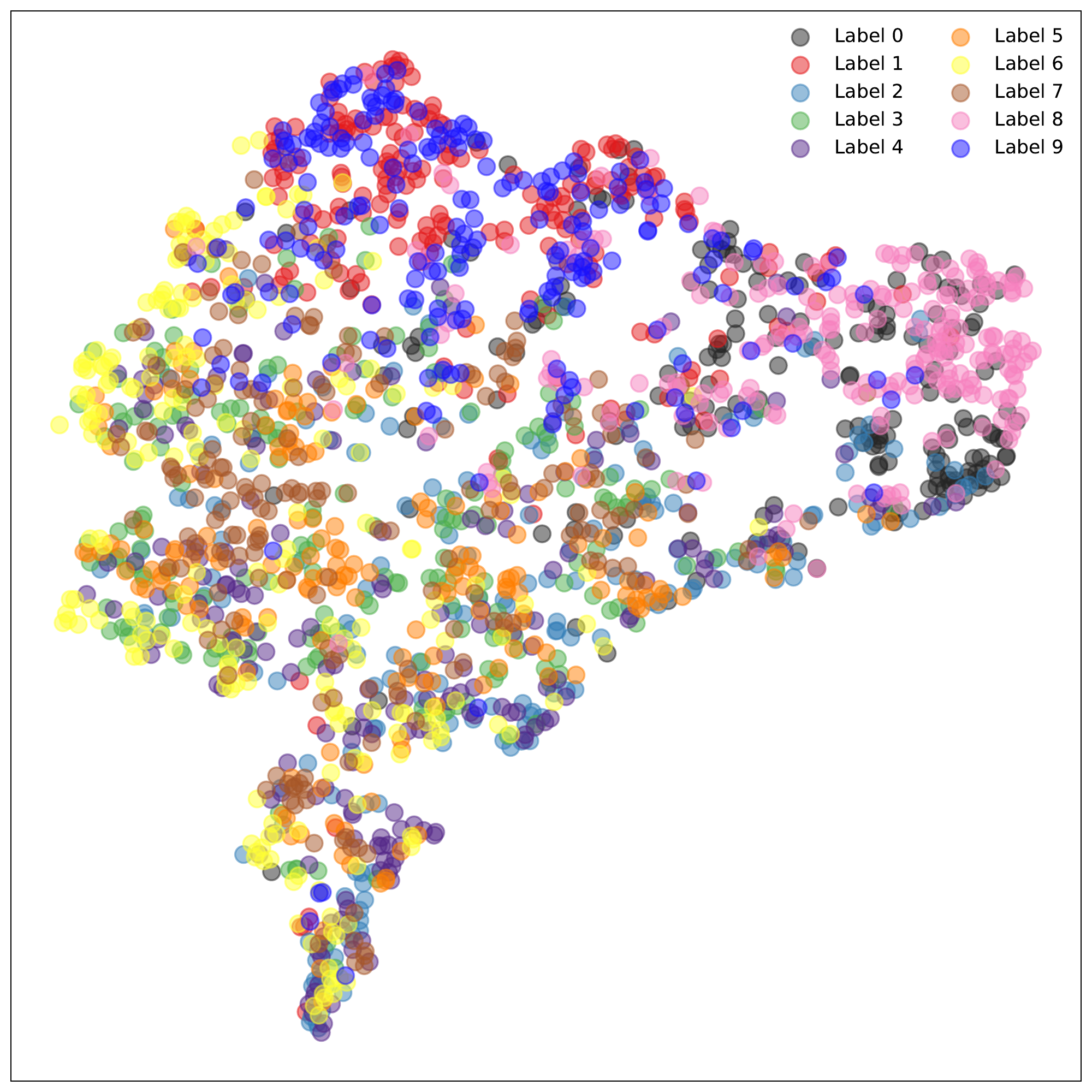}}
    \subfigure[$99$-th round. ]{\includegraphics[width=0.24\linewidth]{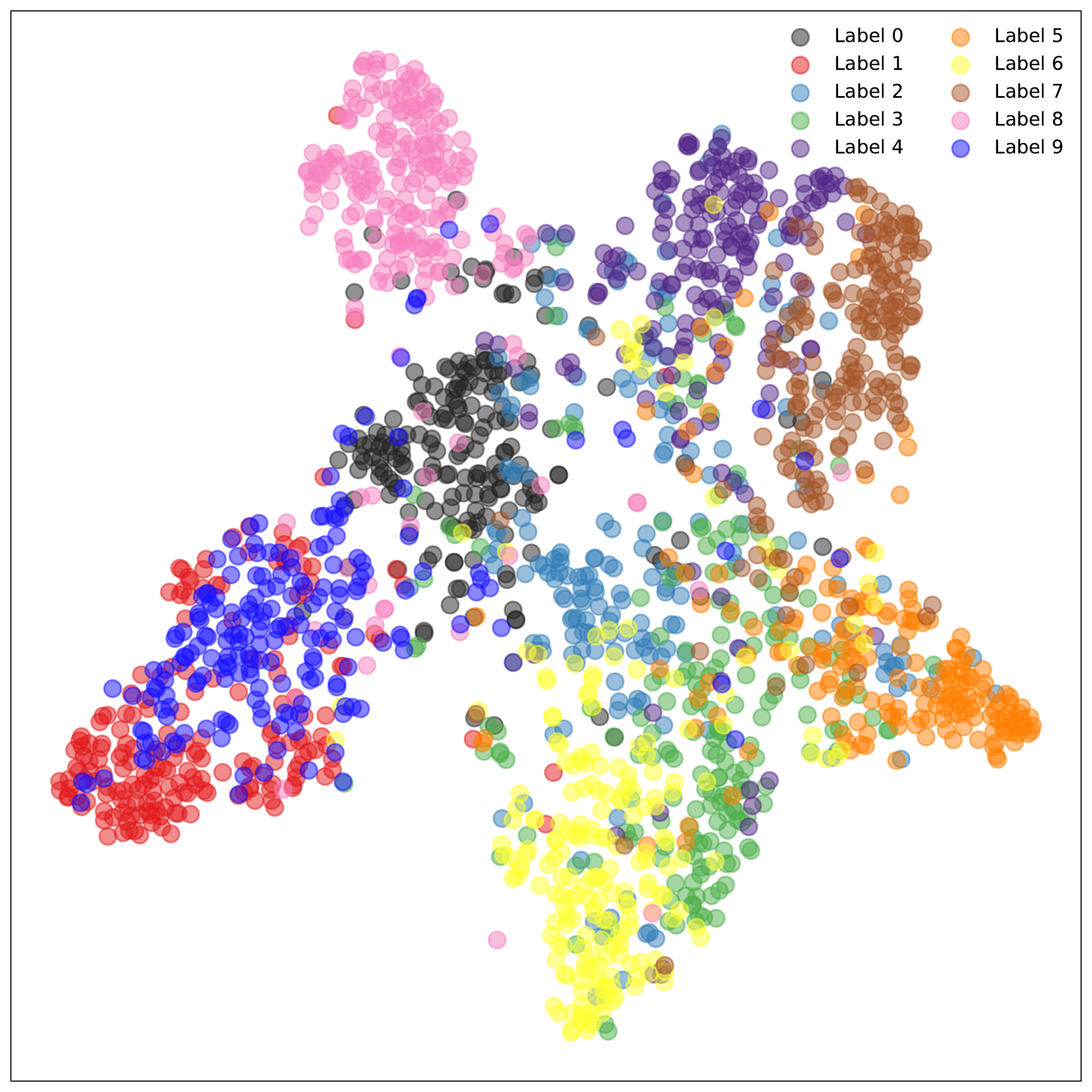}}
    \subfigure[$299$-th round. ]{\includegraphics[width=0.24\linewidth]{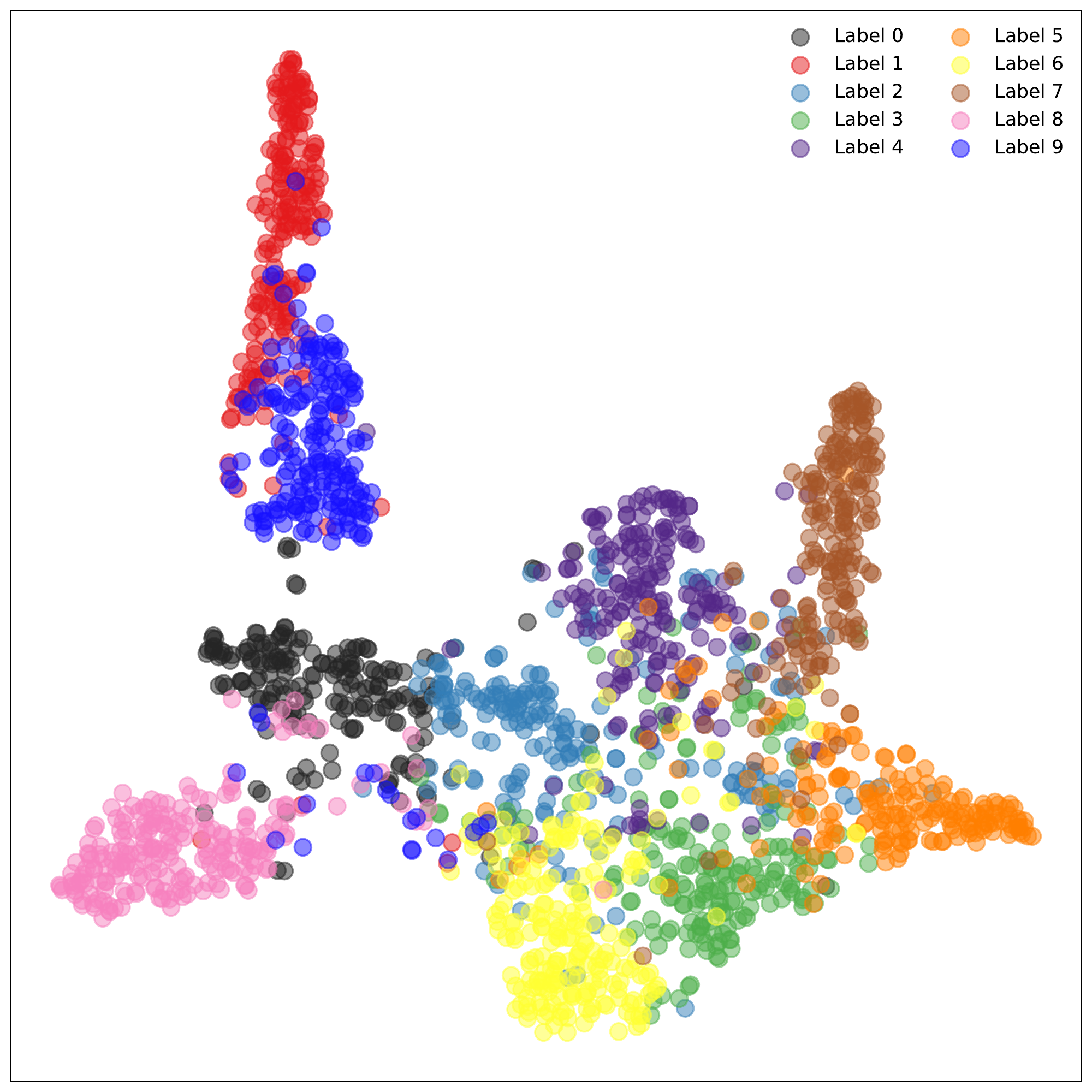}}
    \subfigure[$999$-th round. ]{\includegraphics[width=0.24\linewidth]{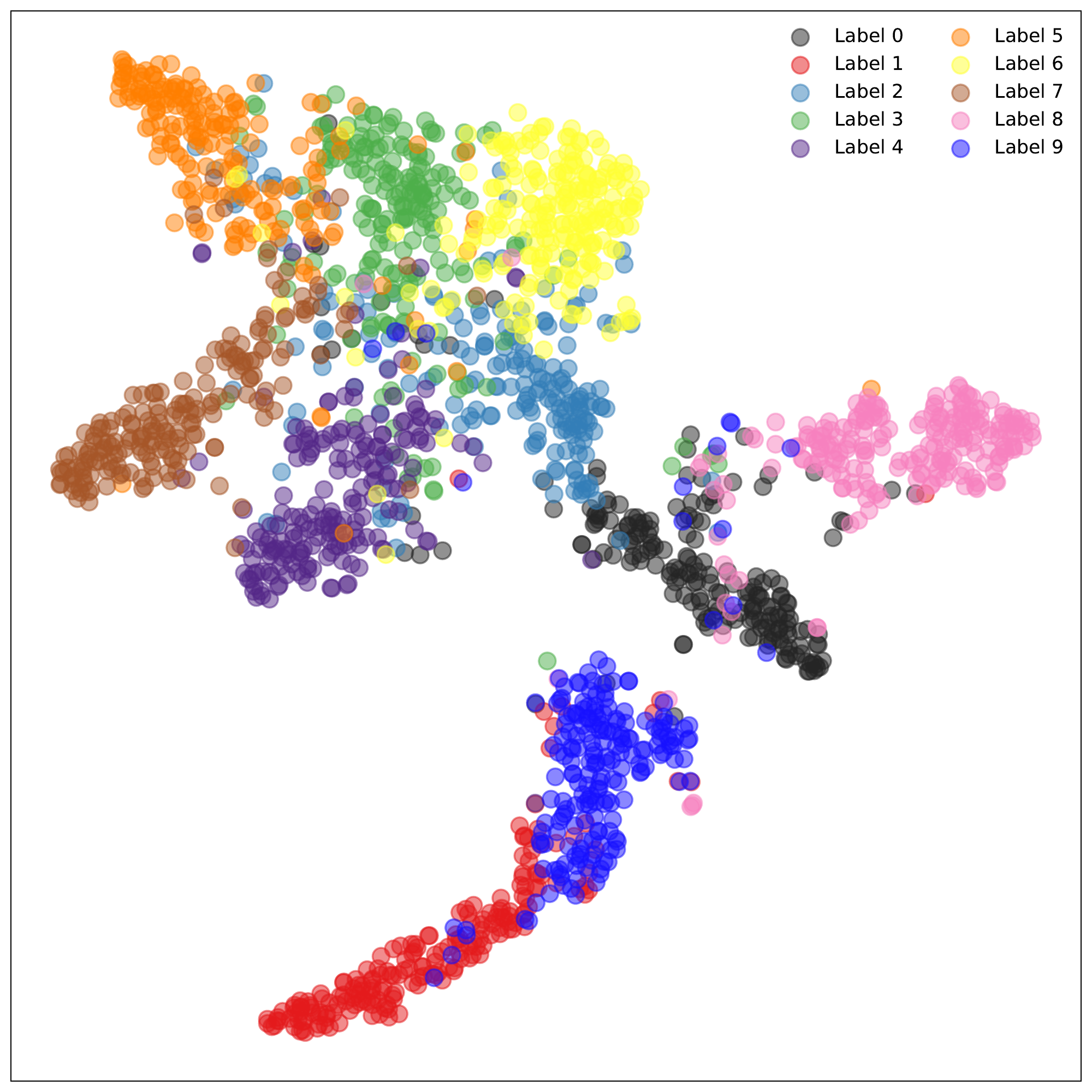}}
   \caption{Features of CIFAR10 test data of server model \textbf{without} VHL.}
\label{fig:Apppendix-Server-Feature-fedavg} 
\vspace{-0.3cm}
\end{figure*}
% \vspace{-0.3cm}

\begin{figure*}[htb!] 
    \setlength{\abovedisplayskip}{-2pt}
    \setlength{\abovecaptionskip}{-2pt}
    \subfigbottomskip=-1pt
    \subfigcapskip=1pt
  \centering
    \subfigure[$9$-th round. ]{\includegraphics[width=0.24\linewidth]{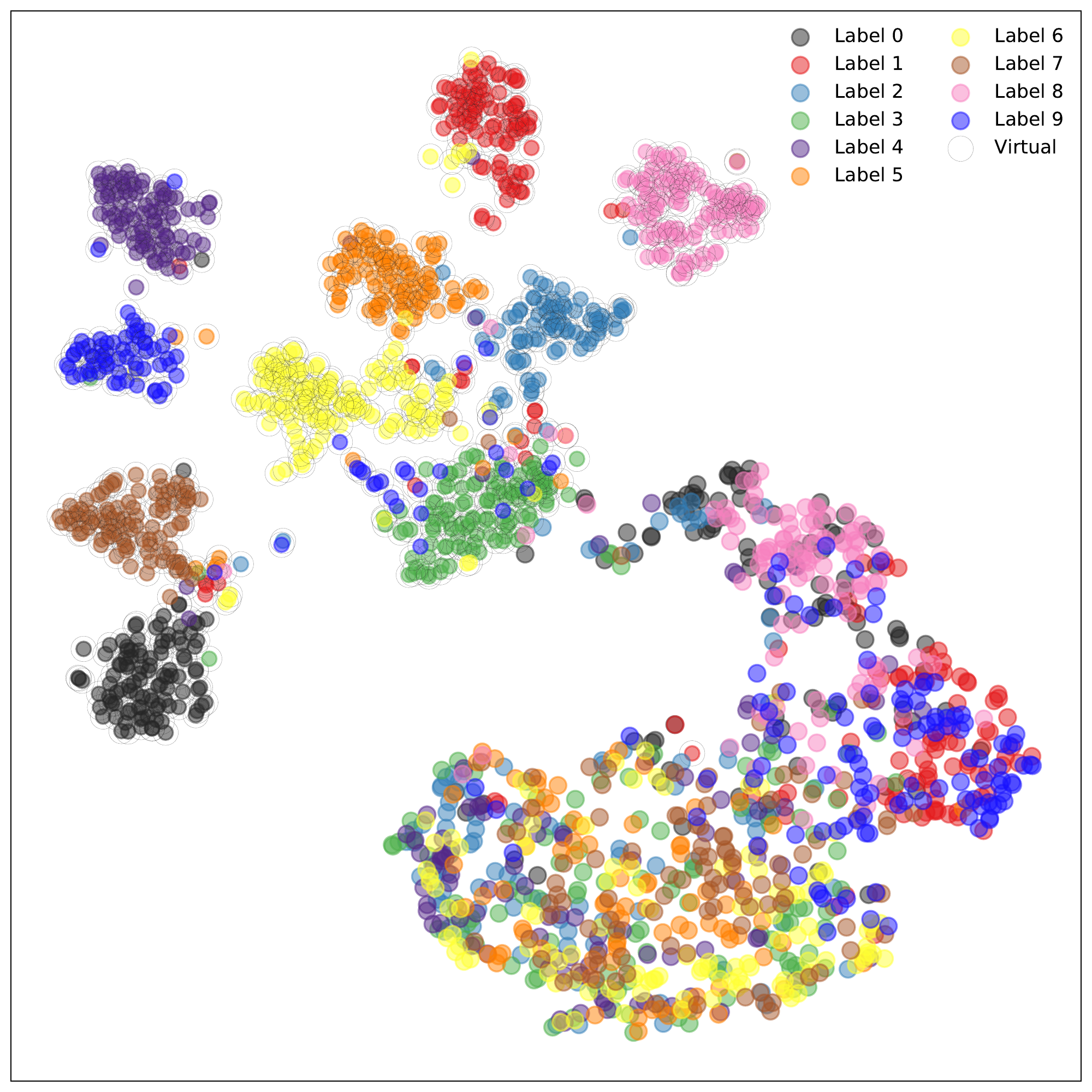}}
    \subfigure[$99$-th round. ]{\includegraphics[width=0.24\linewidth]{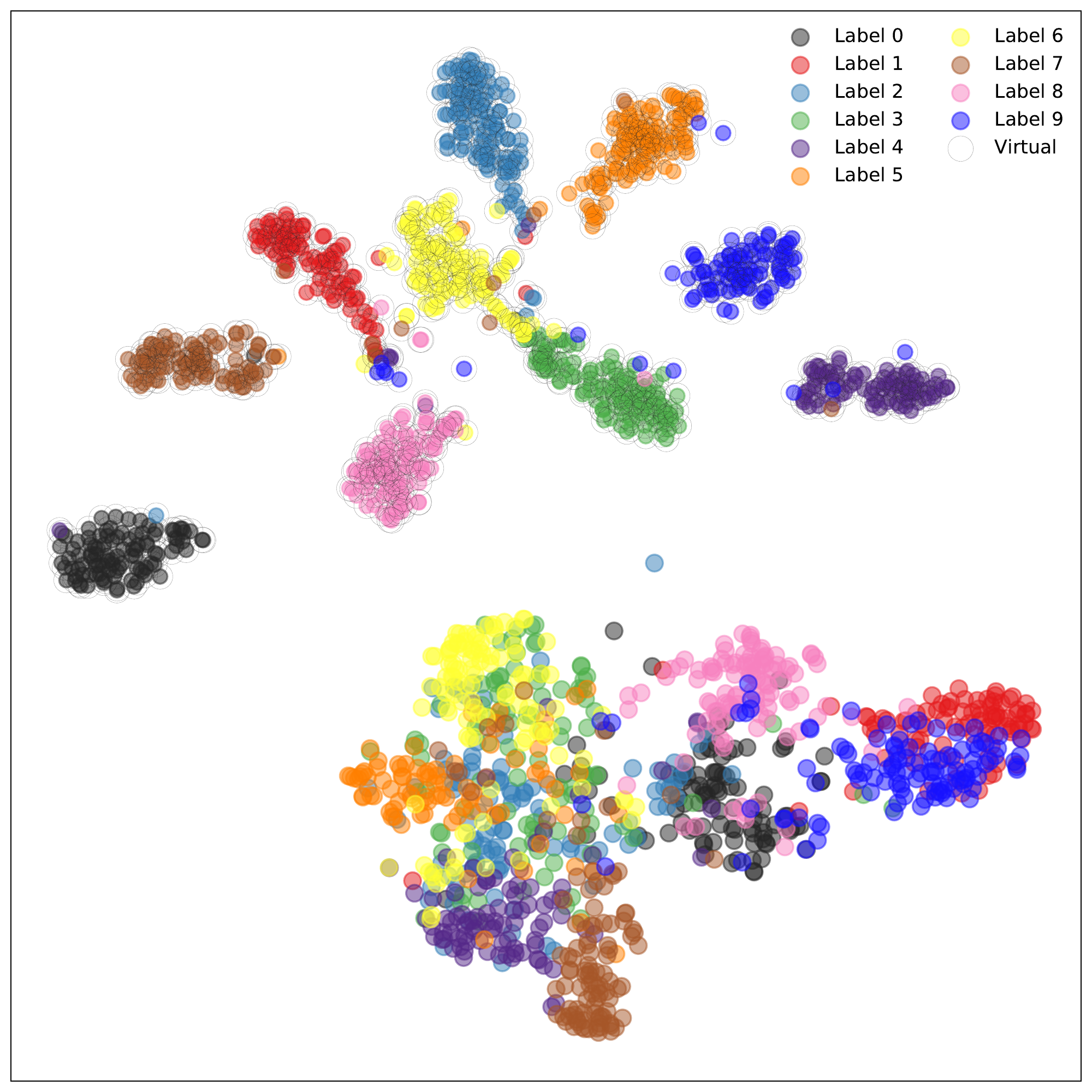}}
    \subfigure[$299$-th round. ]{\includegraphics[width=0.24\linewidth]{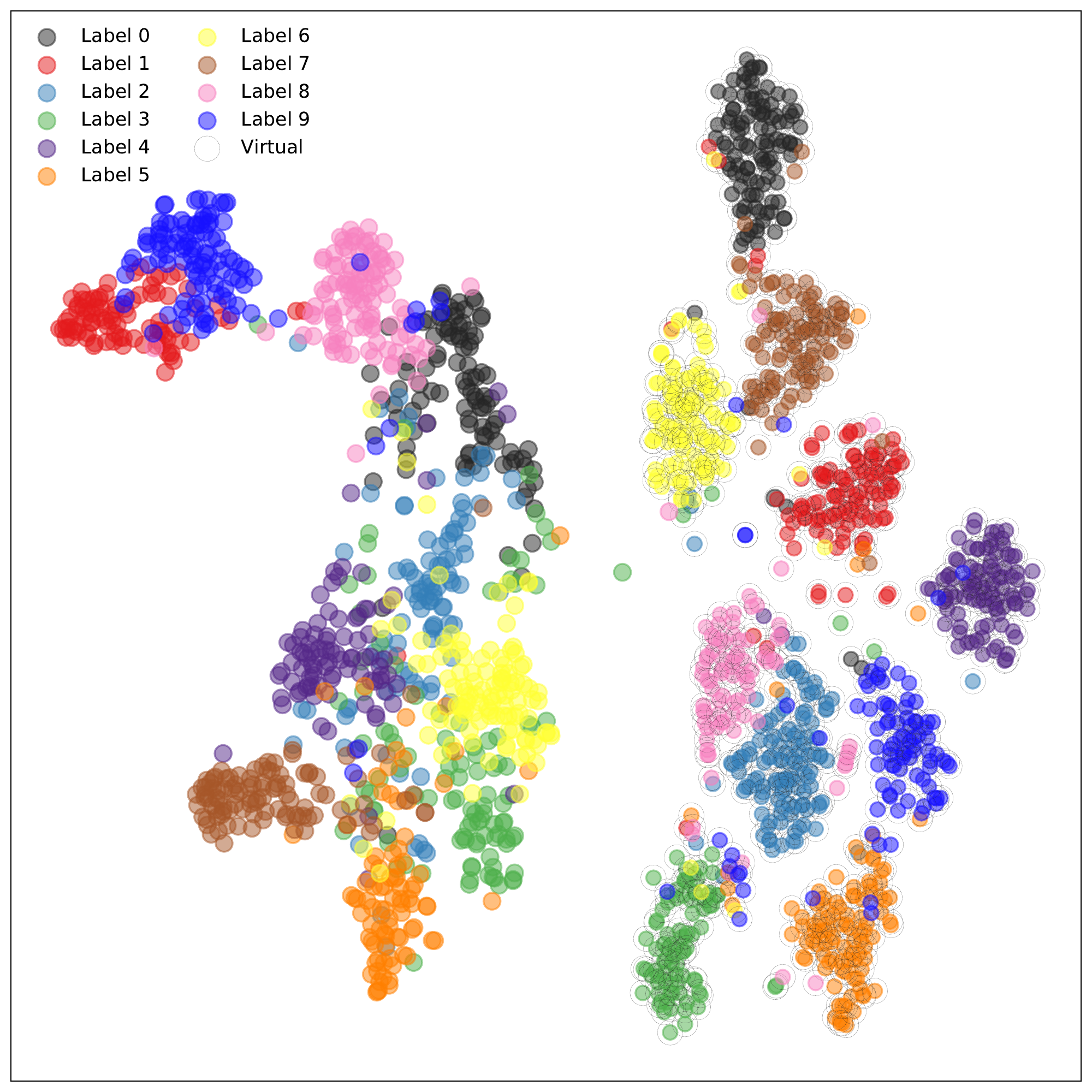}}
    \subfigure[at $999$-th round. ]{\includegraphics[width=0.24\linewidth]{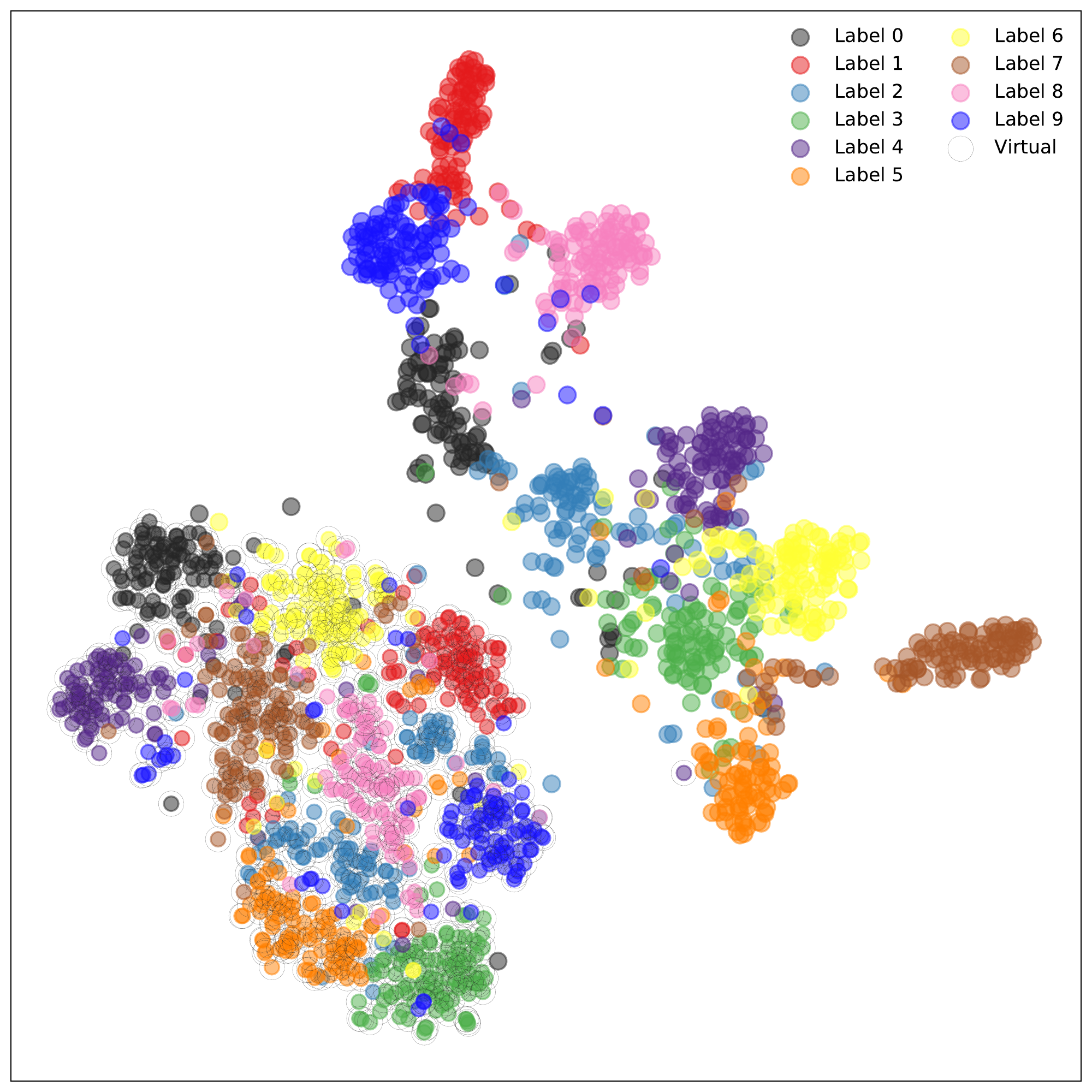}}
   \caption{Features of CIFAR10 test data of server model \textbf{with} Naive VHL.}
\label{fig:Apppendix-Server-Feature-NaiveVHL} 
\vspace{-0.3cm}
\end{figure*}

\begin{figure*}[htb!] 
    \setlength{\abovedisplayskip}{-2pt}
    \setlength{\abovecaptionskip}{-2pt}
    \subfigbottomskip=-1pt
    \subfigcapskip=1pt
  \centering
    \subfigure[$9$-th round. ]{\includegraphics[width=0.24\linewidth]{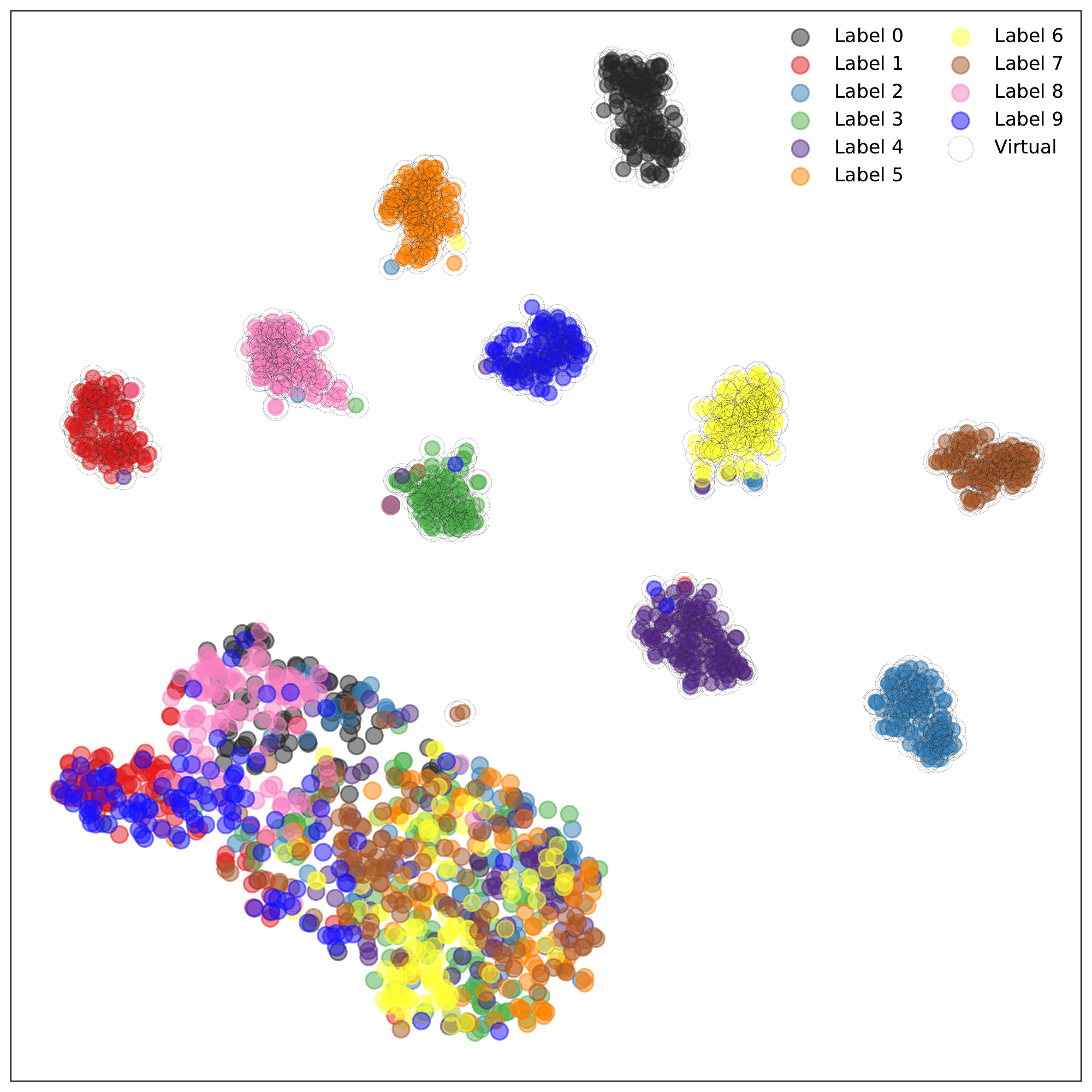}}
    \subfigure[$99$-th round. ]{\includegraphics[width=0.24\linewidth]{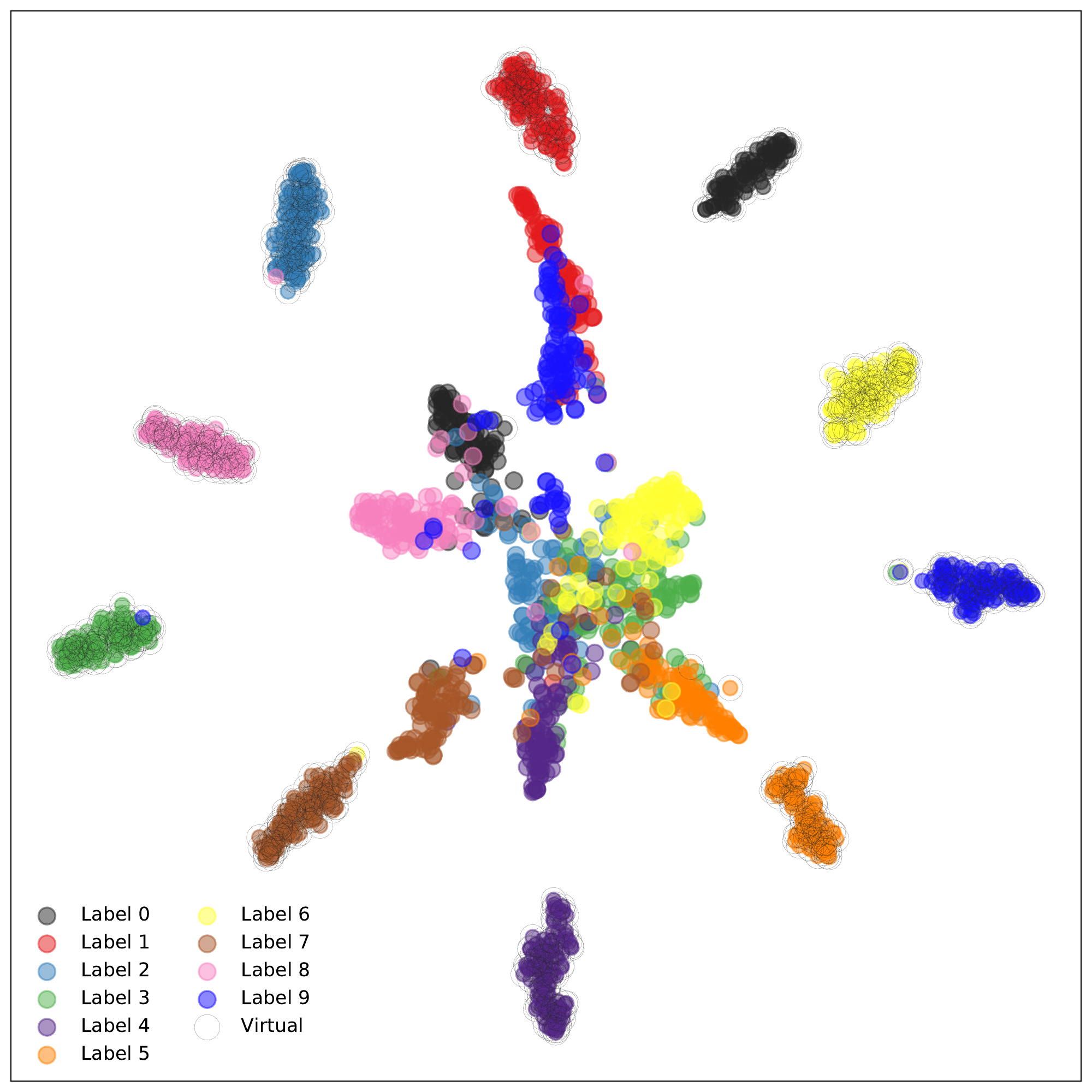}}
    \subfigure[$299$-th round. ]{\includegraphics[width=0.24\linewidth]{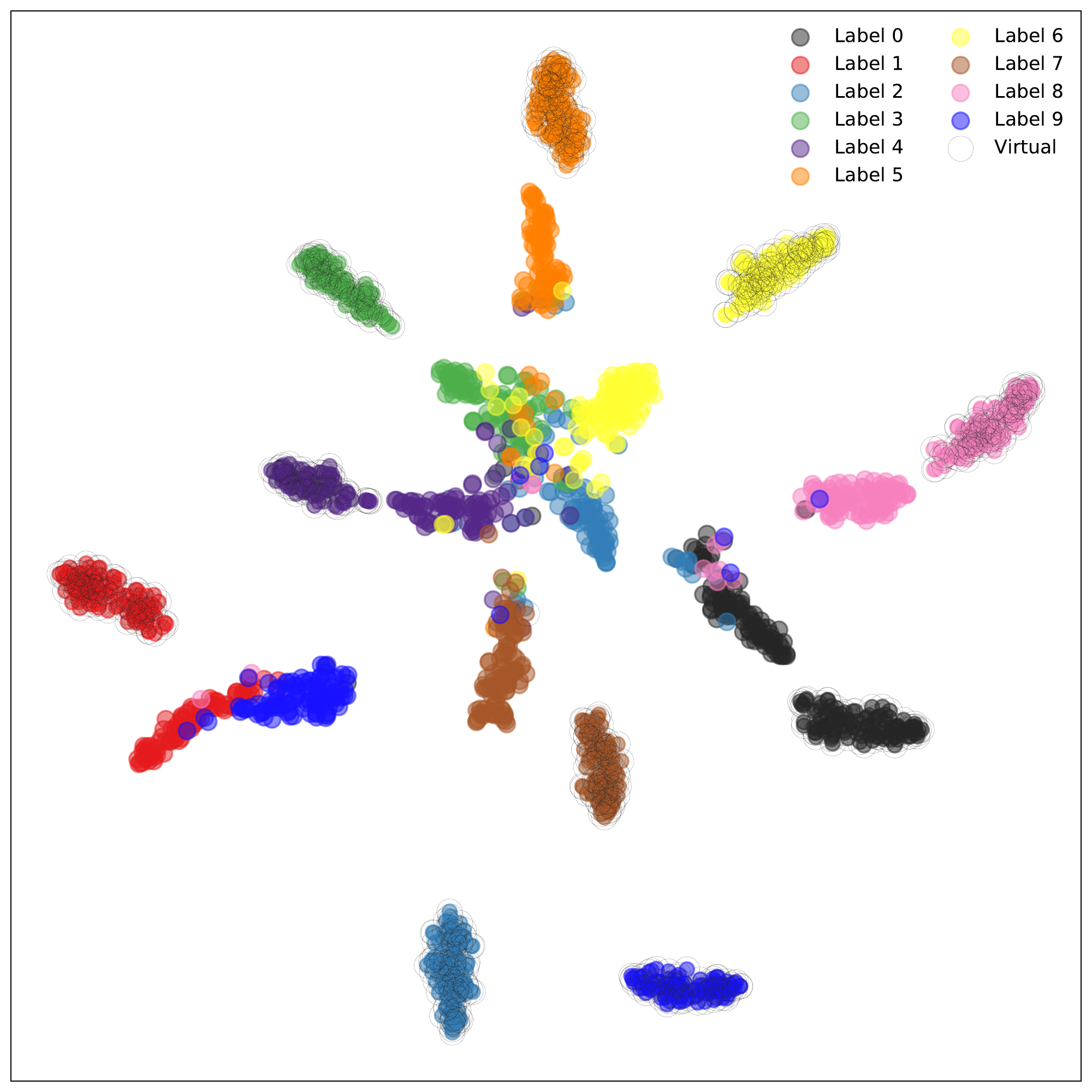}}
    \subfigure[at $999$-th round. ]{\includegraphics[width=0.24\linewidth]{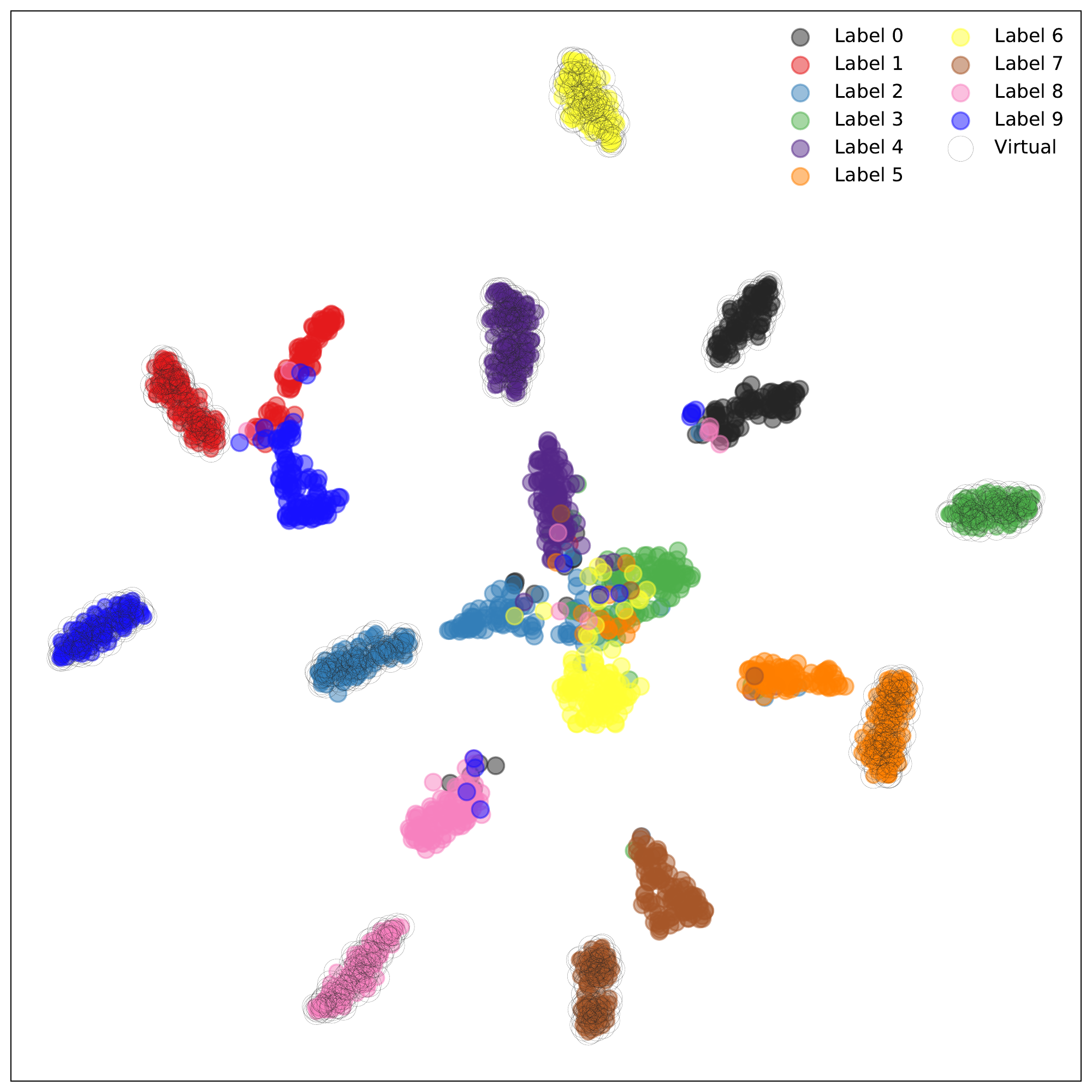}}
   \caption{Features of CIFAR10 test data of server model \textbf{with} VHL.}
\label{fig:Apppendix-Server-Feature-VHL} 
\vspace{-0.3cm}
\end{figure*}

\begin{figure*}[h!]
    \setlength{\abovedisplayskip}{-2pt}
    \setlength{\abovecaptionskip}{-2pt}
    \subfigbottomskip=-1pt
    \subfigcapskip=1pt
  \centering
% \!\!\!\!\!\!\!\!
     \subfigure[$a=0.1$, $K=10$, $E=1$ ]{\includegraphics[width=0.24\textwidth]{Convergence/normal-resnet18_v2-cifar10-0.1-10-1.pdf}}
     \subfigure[$a=0.1$, $K=10$, $E=5$ ]{\includegraphics[width=0.24\textwidth]{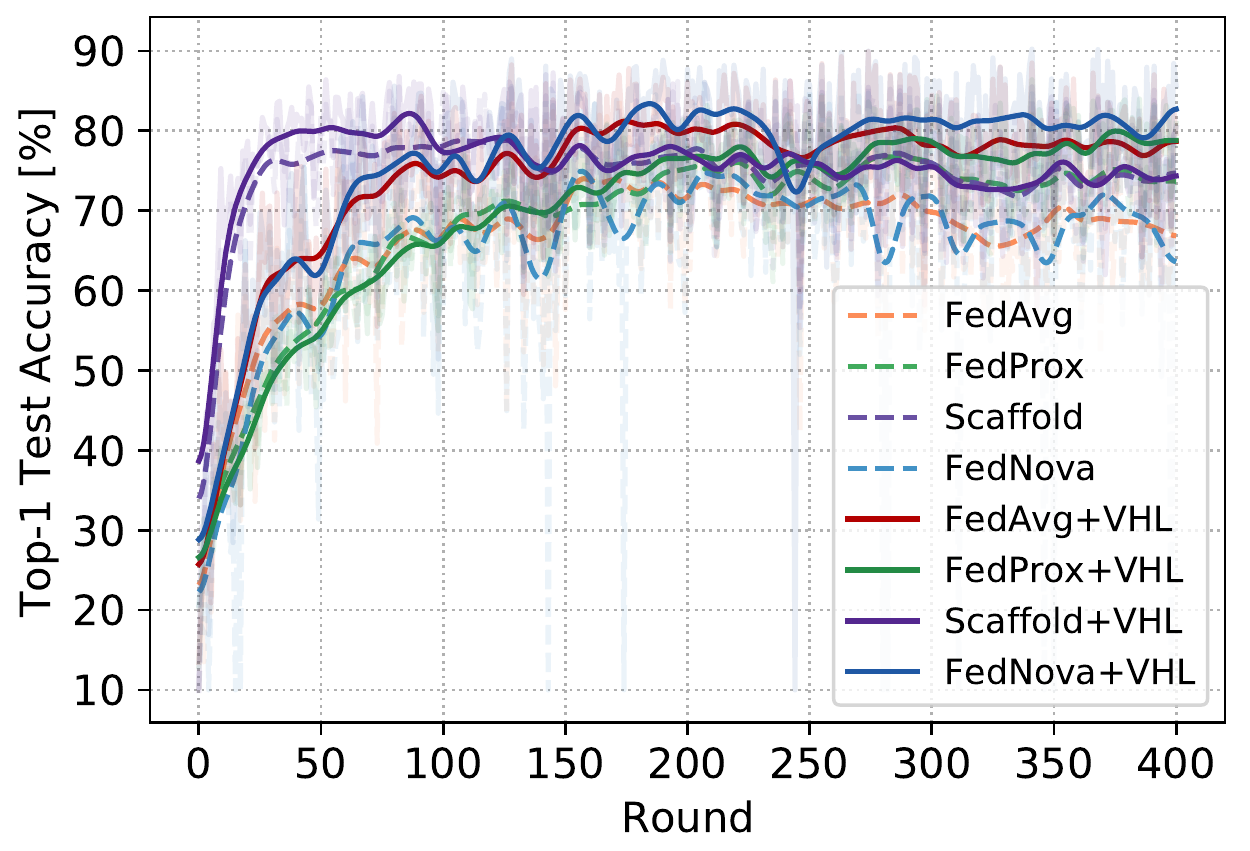}}
     \subfigure[$a=0.1$, $K=100$, $E=1$ ]{\includegraphics[width=0.24\textwidth]{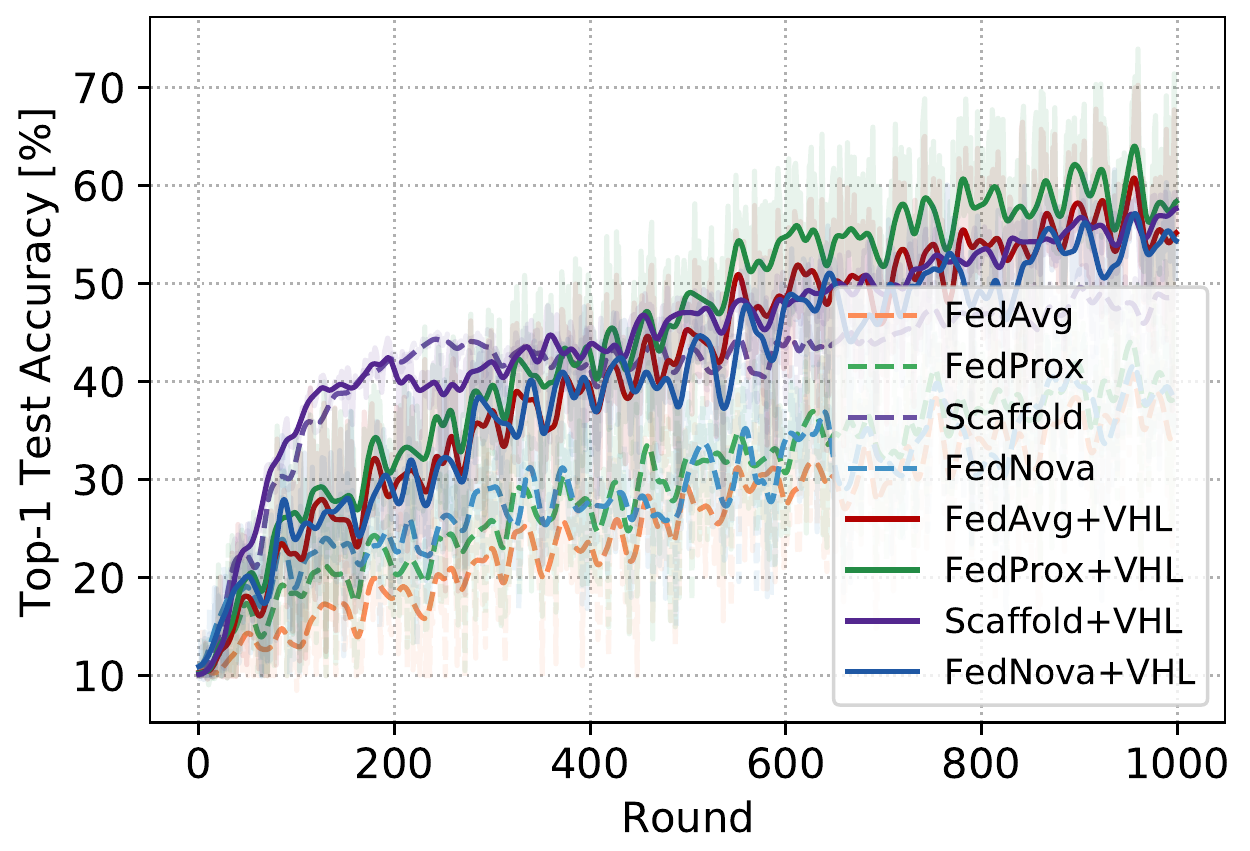}}
     \subfigure[$a=0.05$, $K=10$, $E=1$ ]{\includegraphics[width=0.24\textwidth]{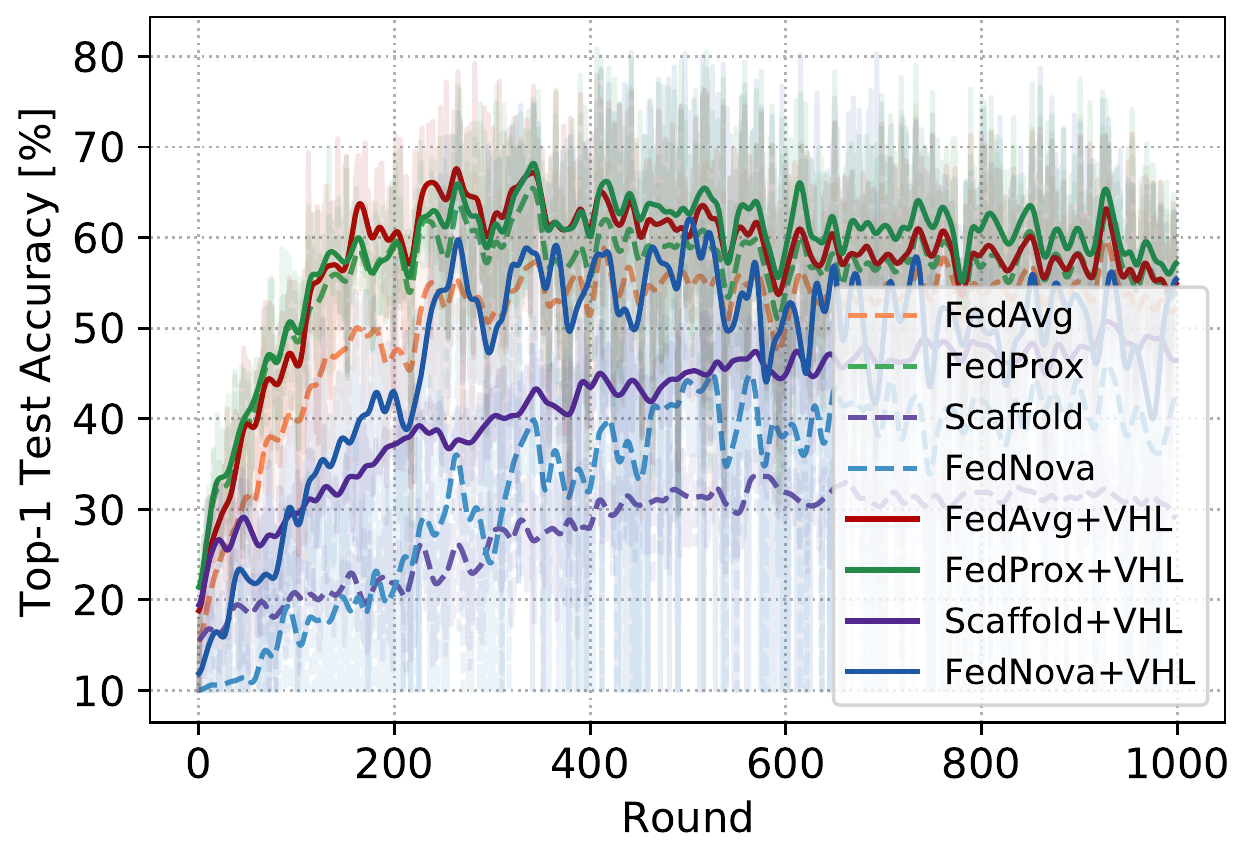}}
    \caption{Convergence comparison of CIFAR-10.}
    \label{fig:Convergence-cifar10}
\vspace{-0.3cm}
% \vspace{-0.5cm}
\end{figure*}

\begin{figure*}[h!]
    \setlength{\abovedisplayskip}{-2pt}
    \setlength{\abovecaptionskip}{-2pt}
    \subfigbottomskip=-1pt
    \subfigcapskip=1pt
  \centering
% \!\!\!\!\!\!\!\!
     \subfigure[$a=0.1$, $K=10$, $E=1$ ]{\includegraphics[width=0.24\textwidth]{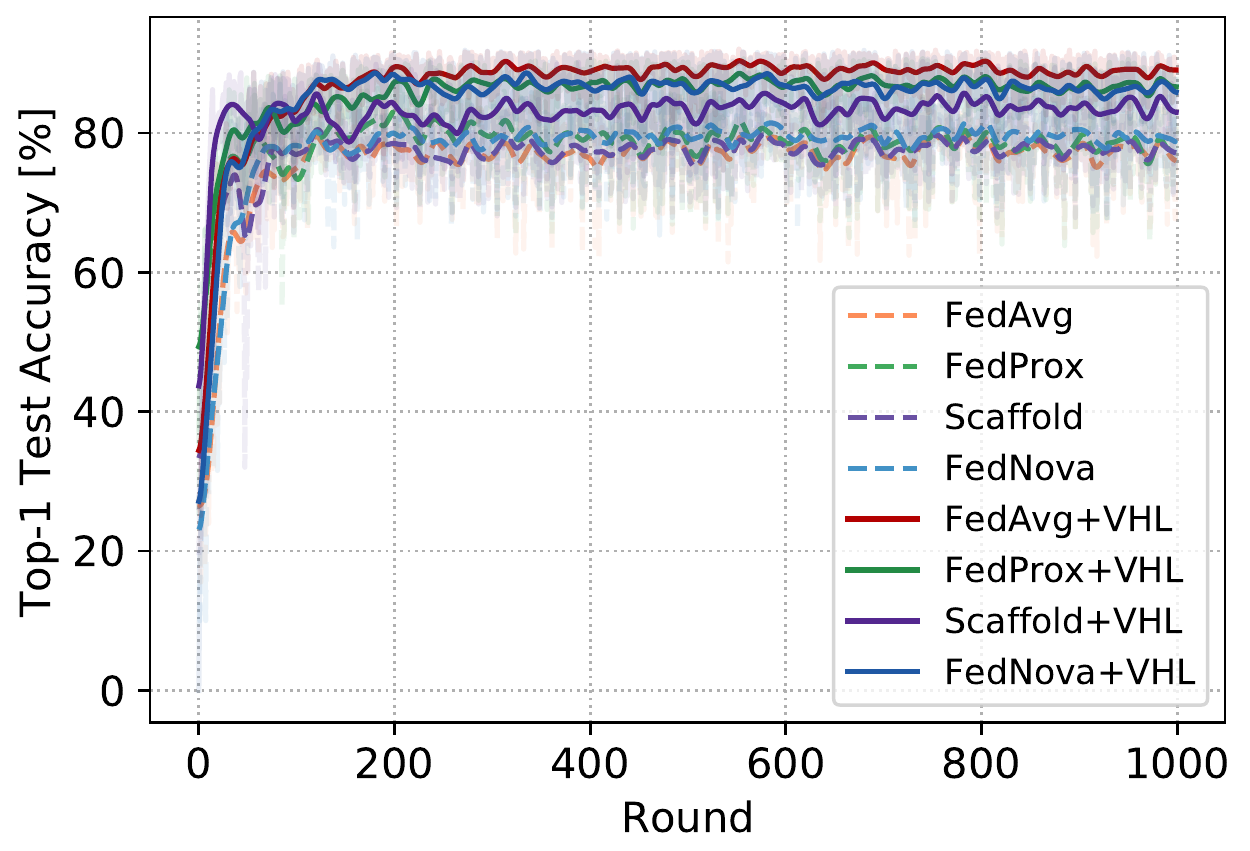}}
     \subfigure[$a=0.1$, $K=10$, $E=5$ ]{\includegraphics[width=0.24\textwidth]{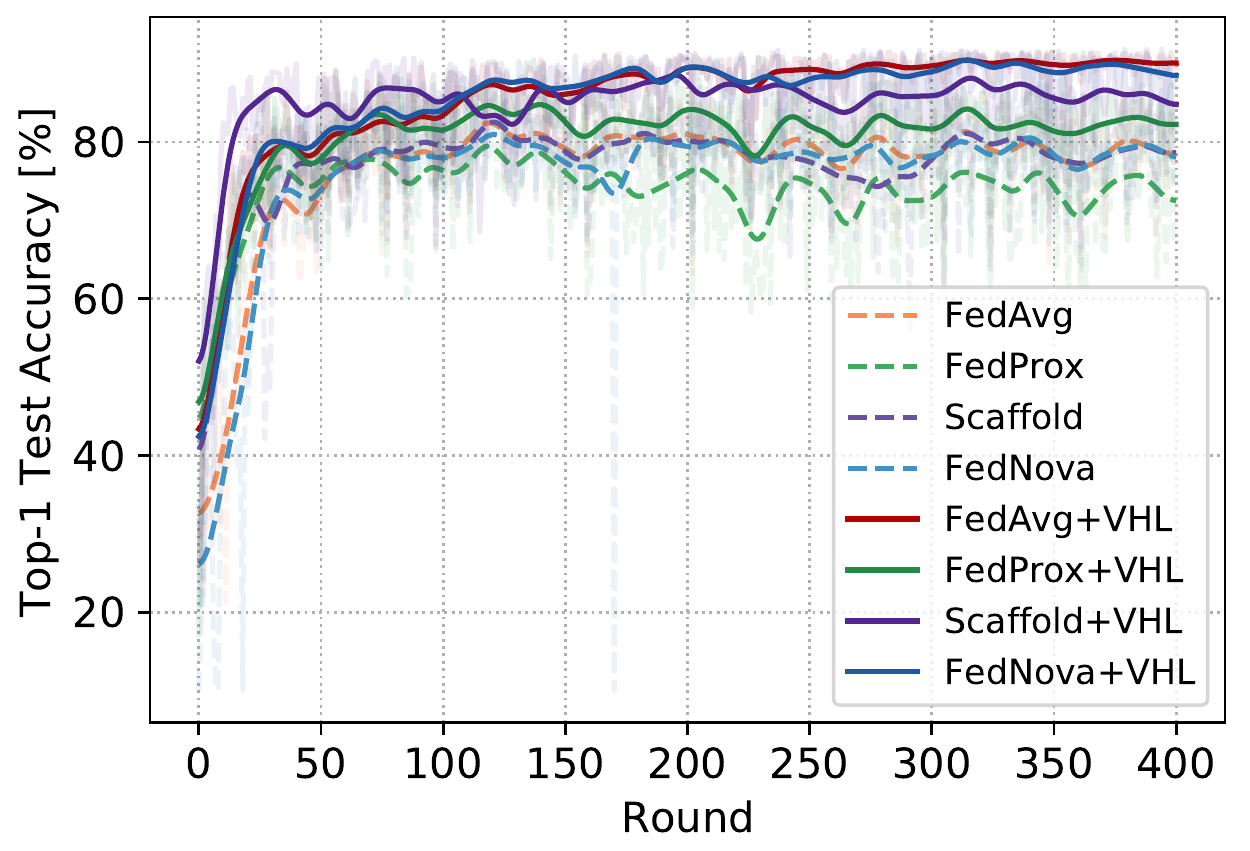}}
     \subfigure[$a=0.1$, $K=100$, $E=1$ ]{\includegraphics[width=0.24\textwidth]{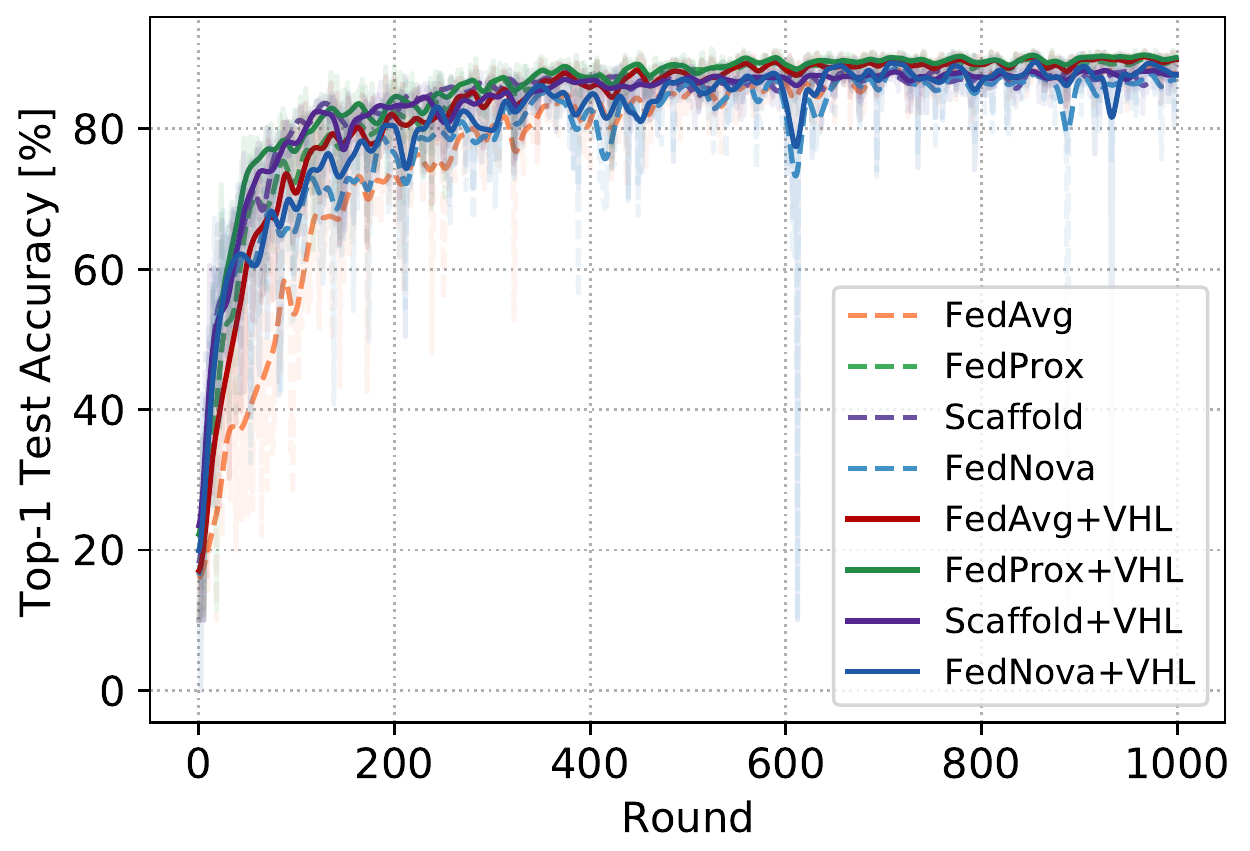}}
     \subfigure[$a=0.05$, $K=10$, $E=1$ ]{\includegraphics[width=0.24\textwidth]{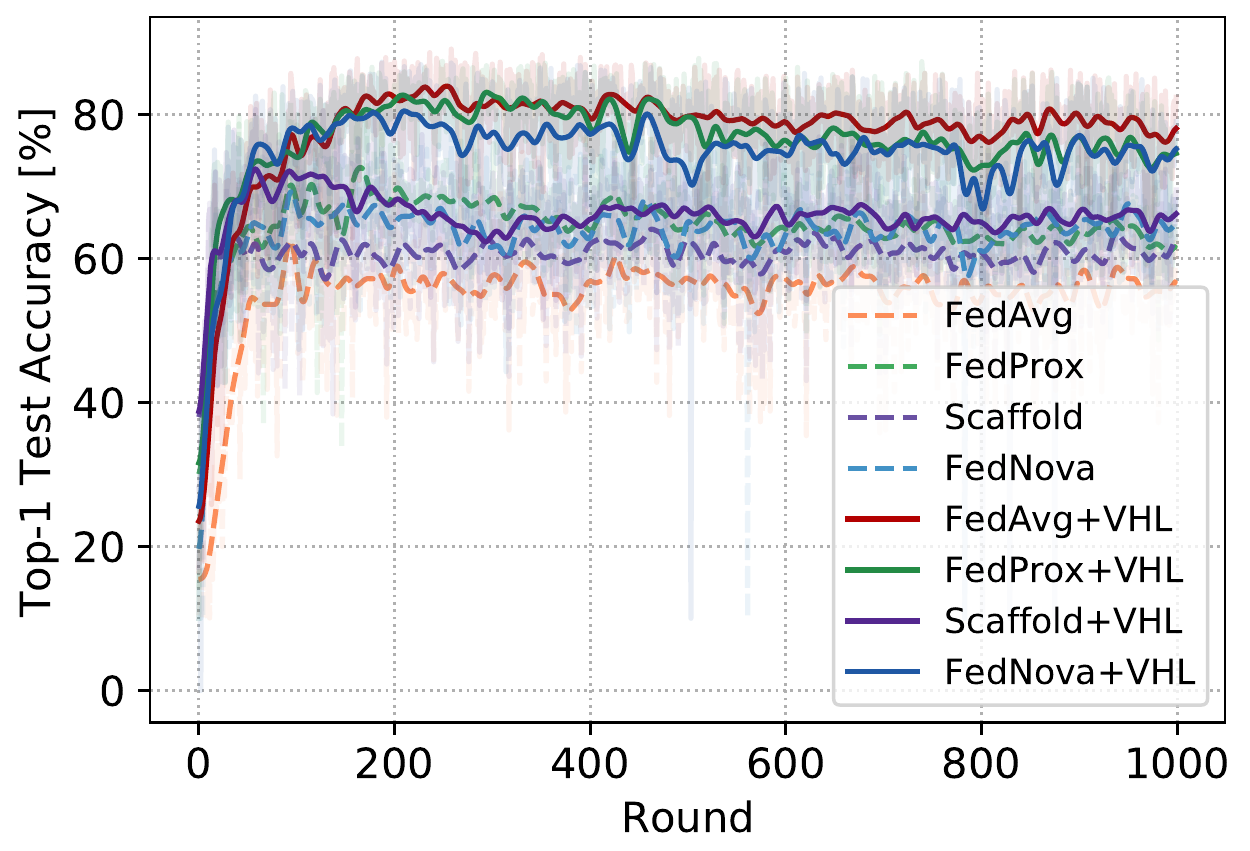}}
    \caption{Convergence comparison of FMNIST.}
    \label{fig:Convergence-fmnist}
\vspace{-0.3cm}
% \vspace{-0.5cm}
\end{figure*}

\begin{figure*}[h!]
    \setlength{\abovedisplayskip}{-2pt}
    \setlength{\abovecaptionskip}{-2pt}s
    \subfigbottomskip=-1pt
    \subfigcapskip=1pt
  \centering
% \!\!\!\!\!\!\!\!
     \subfigure[$a=0.1$, $K=10$, $E=1$ ]{\includegraphics[width=0.24\textwidth]{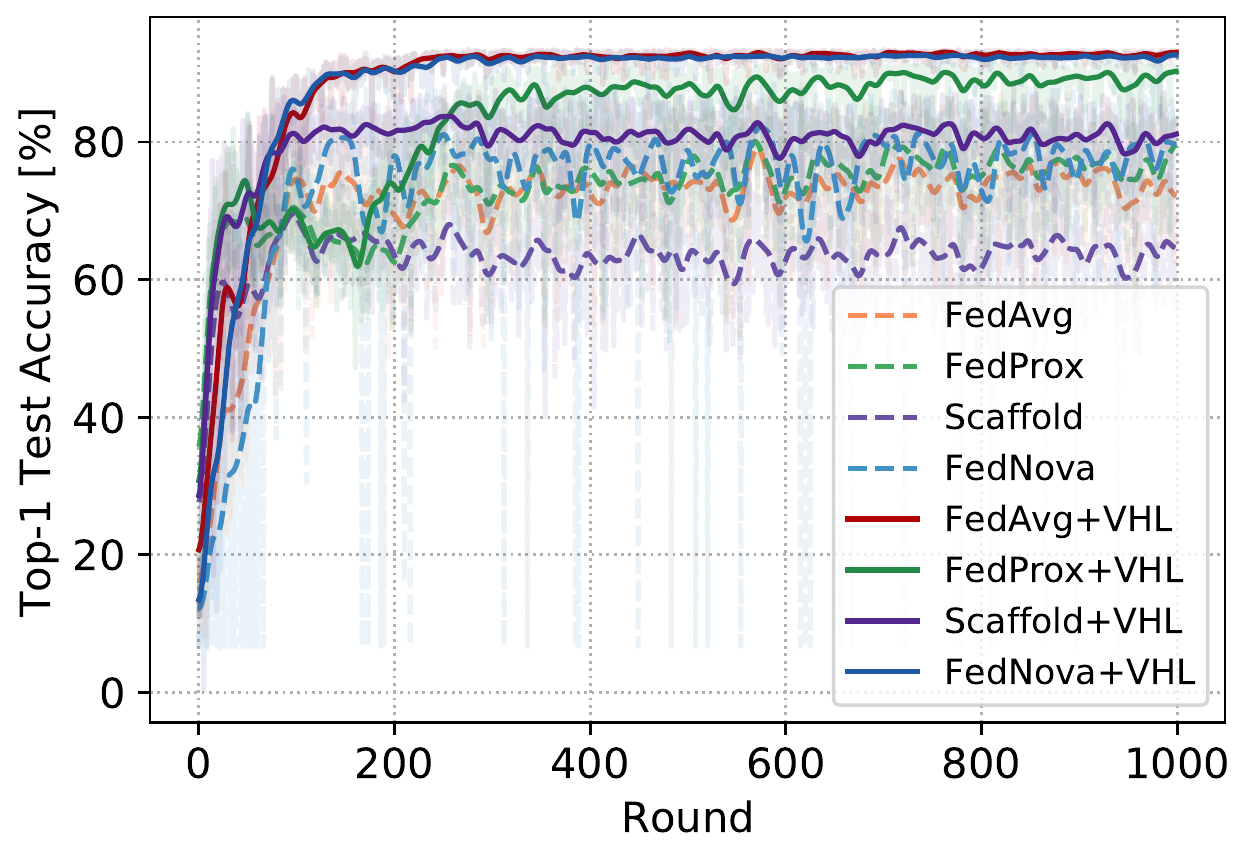}}
     \subfigure[$a=0.1$, $K=10$, $E=5$ ]{\includegraphics[width=0.24\textwidth]{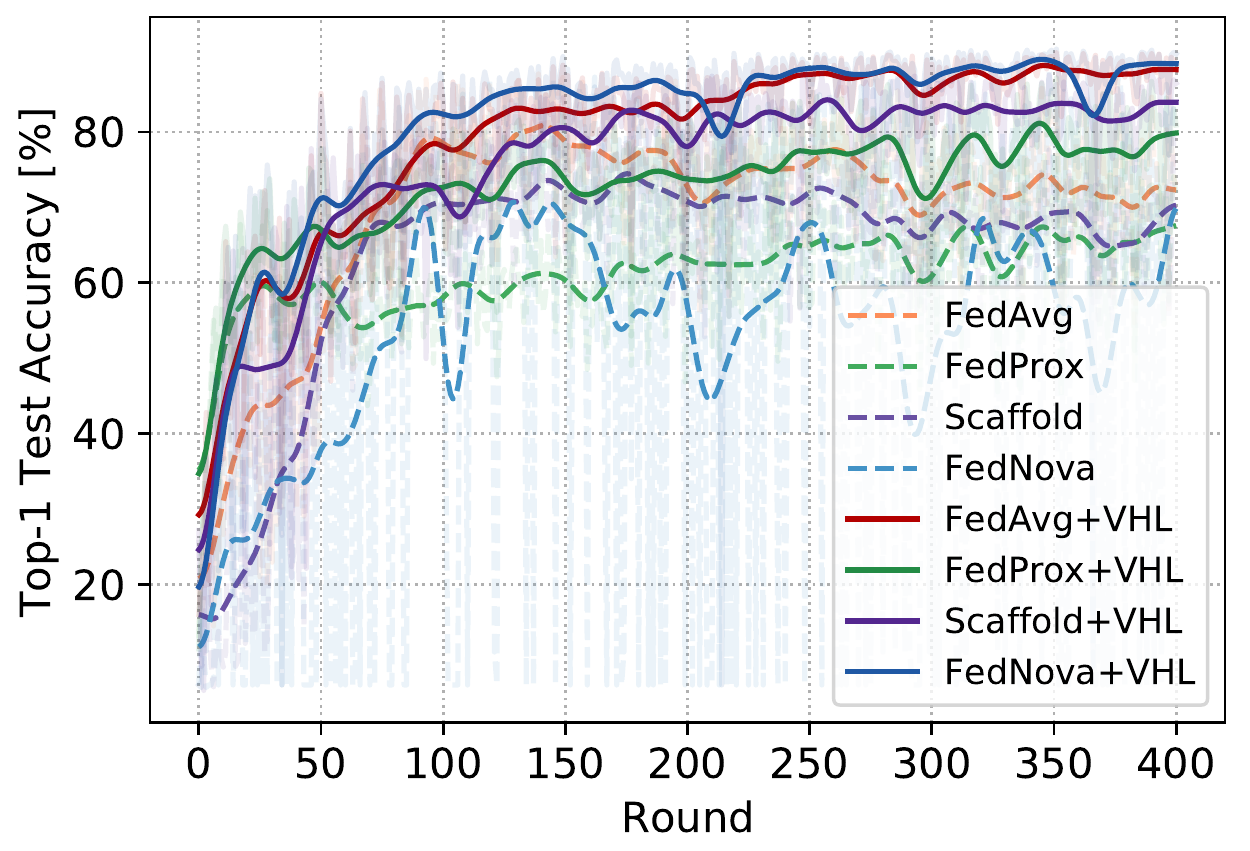}}
     \subfigure[$a=0.1$, $K=100$, $E=1$ ]{\includegraphics[width=0.24\textwidth]{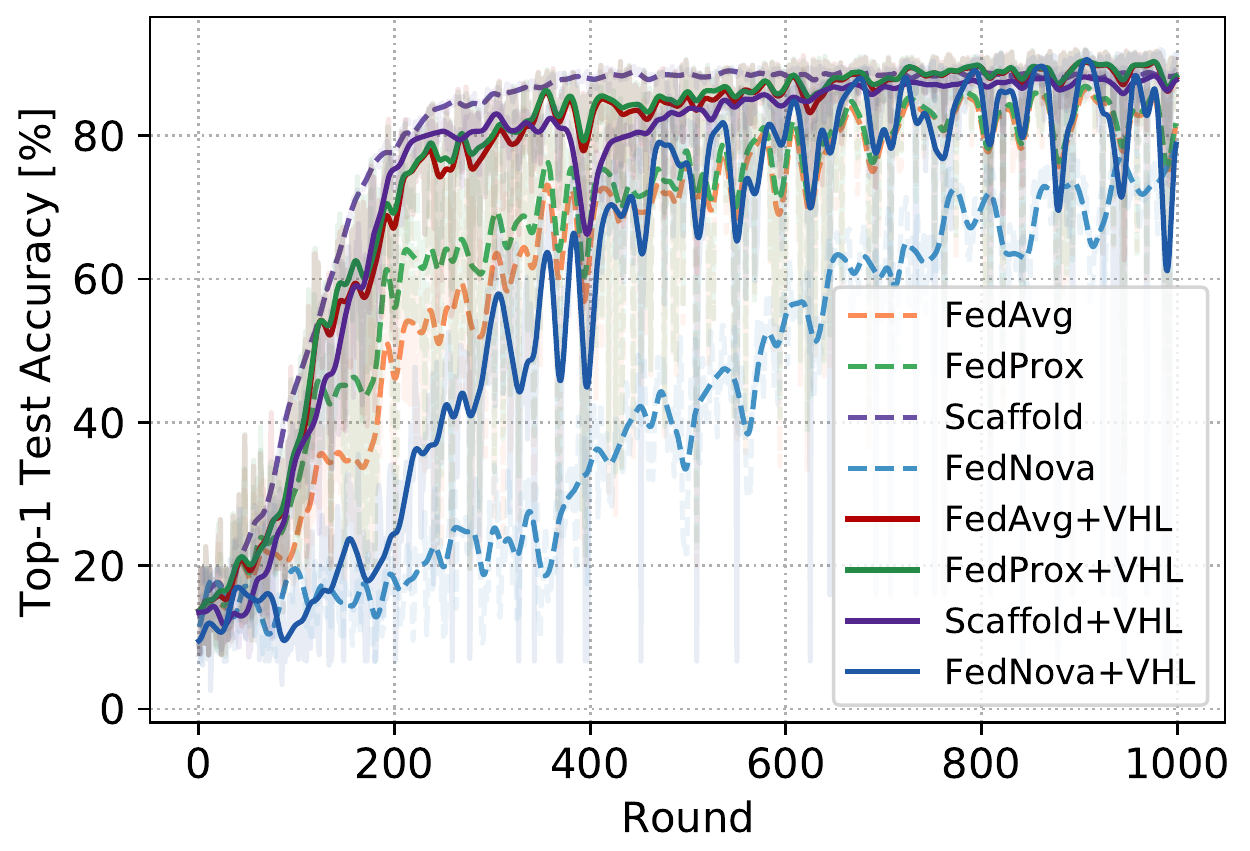}}
     \subfigure[$a=0.05$, $K=10$, $E=1$ ]{\includegraphics[width=0.24\textwidth]{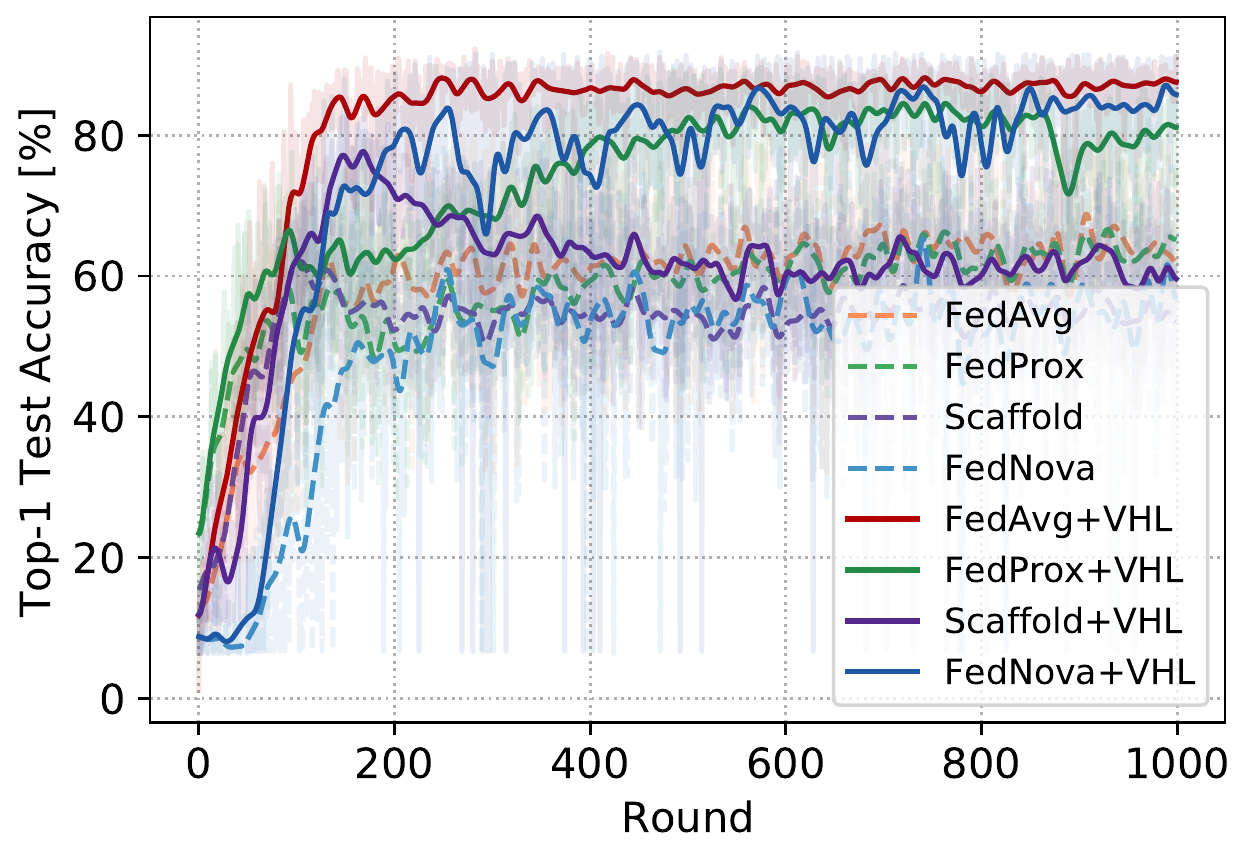}}
    \caption{Convergence comparison of SVHN.}
    \label{fig:Convergence-SVHN}
\vspace{-0.3cm}
% \vspace{-0.5cm}
\end{figure*}

\begin{figure*}[h!]
    \setlength{\abovedisplayskip}{-2pt}
    \setlength{\abovecaptionskip}{-2pt}
    \subfigbottomskip=-1pt
    \subfigcapskip=1pt
  \centering
% \!\!\!\!\!\!\!\!
     \subfigure[$a=0.1$, $K=10$, $E=1$ ]{\includegraphics[width=0.24\textwidth]{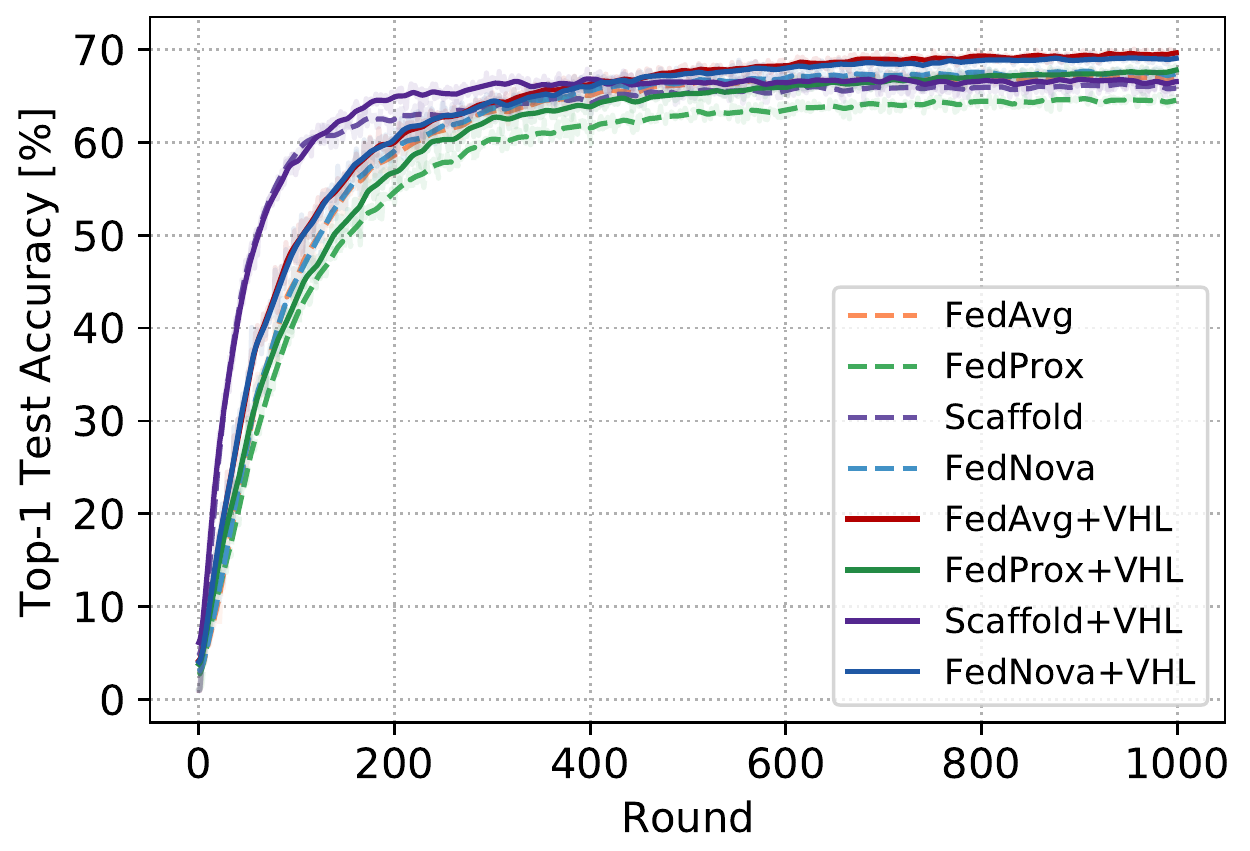}}
     \subfigure[$a=0.1$, $K=10$, $E=5$ ]{\includegraphics[width=0.24\textwidth]{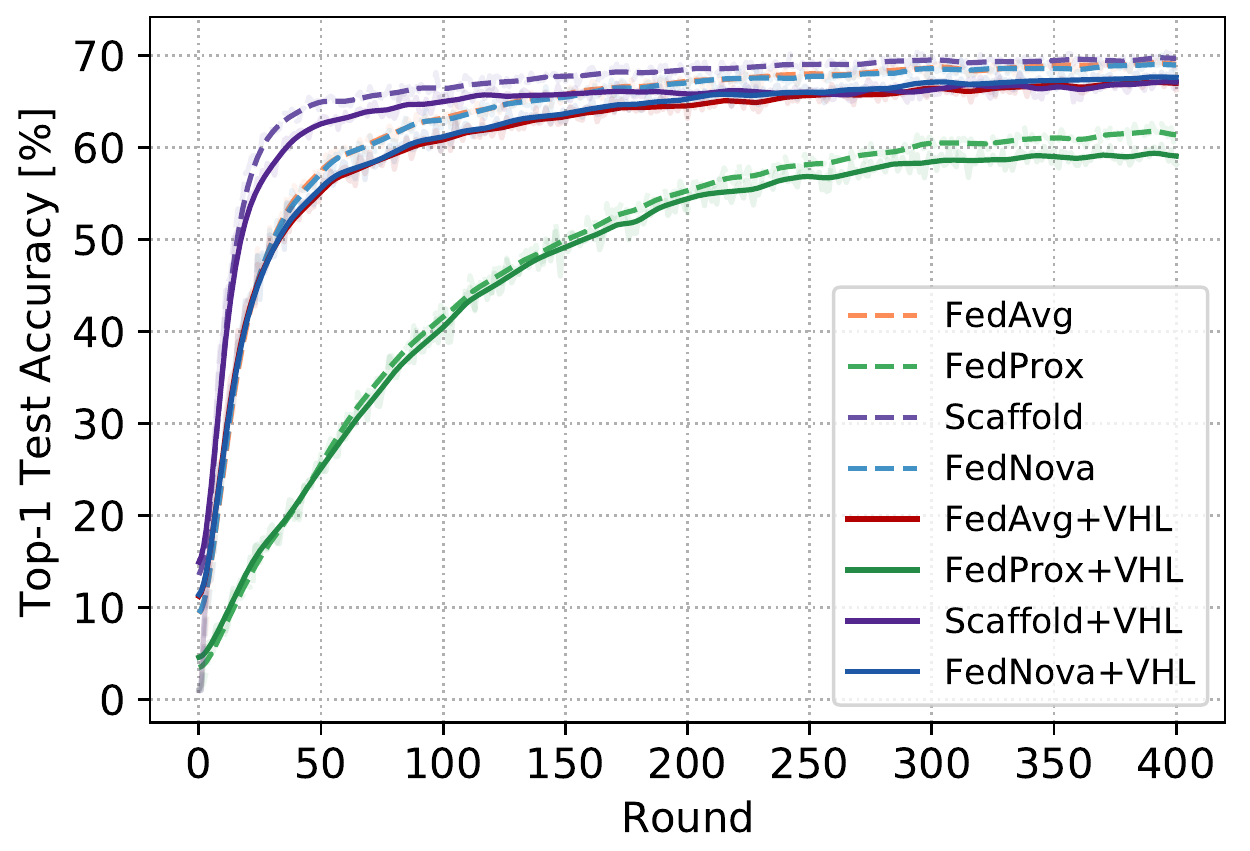}}
     \subfigure[$a=0.1$, $K=100$, $E=1$ ]{\includegraphics[width=0.24\textwidth]{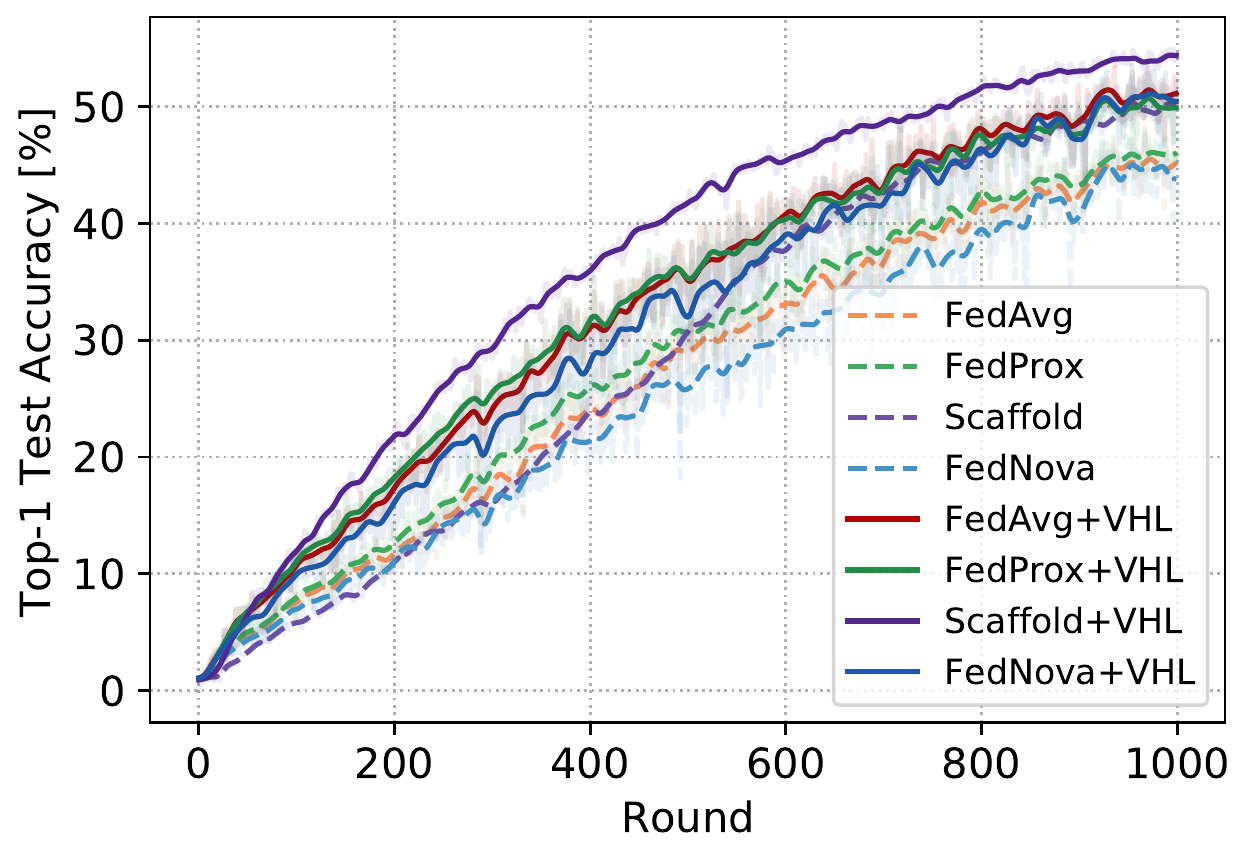}}
     \subfigure[$a=0.05$, $K=10$, $E=1$ ]{\includegraphics[width=0.24\textwidth]{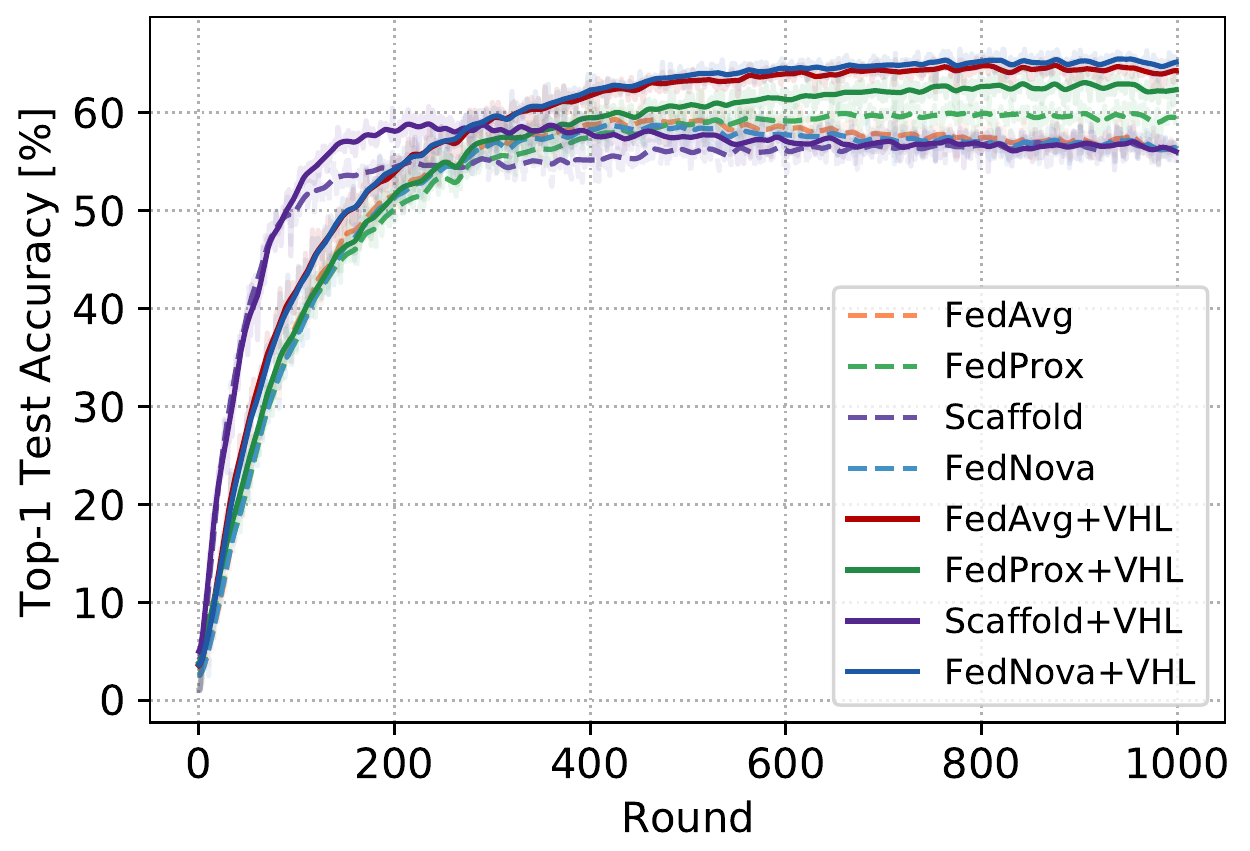}}
    \caption{Convergence comparison of CIFAR100.}
    \label{fig:Convergence-CIFAR100}
\vspace{-0.3cm}
% \vspace{-0.5cm}
\end{figure*}

\subsection{More Facets of VHL}\label{appendix:MoreFacets}

We provide more comparisons of convergence trends of three experiments, including loading pretrained models (Figure~\ref{fig:Ablation-Convergence} (a)), VHL with different batch size (Figure~\ref{fig:Ablation-Convergence} (b)), training accuracy with VHL but different virtual data (Figure~\ref{fig:Ablation-Convergence} (c)).

\textbf{Naive VHL.} As summarized in Sec.~\ref{sec:DiffFacetsVHL}, the naive VHL also could benefit FL. We discuss reasons of this in detail here. As pointed by~\citep{luo2021no, Fed2}, local training on local severe imbalanced dataset would lead to the severe preference of the classifer towards the local dominant classes. Due to this bias of classifer, the shallow layers of the model cannot learn good representations of the raw data. Therefore, introducing more common virtual data could alleviate such imbalance situation, even these virtual data has no relationship with the raw data. And the vision information owned by the virtual data could help clients learn better feature extractor, like~\citep{baradad2021learning}.

\textbf{Batch Size of VHL.} Similar with the contrastive learning~\citep{khosla2020supervised}, because VHL needs to pull data of the same class together, the batch size impacts the effect of VHL. As shown in Table~\ref{tab:FurtherStudy} and Figure~\ref{fig:Ablation-Convergence} (b), larger batch size of virtual data improve both convergence and model performance. However, larger batch size means the extra calculation cost, which may be improved in the future.

\textbf{Using Realistic Data.} To our surprise, using the Tiny-ImageNet as the virtual data to conduct VHL cannot outperform using the virtual noise data. According to the classic transfer learning~\citep{da1, baradad2021learning}, transferring knowledge from the realistic data should benefit generalization performance better than from the virtual images that seem like noise (see Figure~\ref{fig:NoiseStyleGan32x32c10}, ~\ref{fig:NoiseStyleGan32x32c100}, ~\ref{fig:Gaussian_Noise} and ~\ref{fig:cifar_conv_decoder}). We suspect the reason is that the Tiny-ImageNet is more difficult to be separated than the generated virtual dataset, leading to the inconsistent features of the Tiny-ImageNet. Then the calibration loses its effect. To validate this conjecture, We plot the training accuracy of the virtual dataset in Figure~\ref{fig:Ablation-Convergence} (c).

\begin{figure}[htb!]
    \setlength{\abovedisplayskip}{-2pt}
    \setlength{\abovecaptionskip}{-2pt}
    \subfigbottomskip=-1pt
    \subfigcapskip=1pt
    \small
    \centering
     \subfigure[Using Model Pretrained on Virtual Dataset. ]{\includegraphics[width=0.32\textwidth]{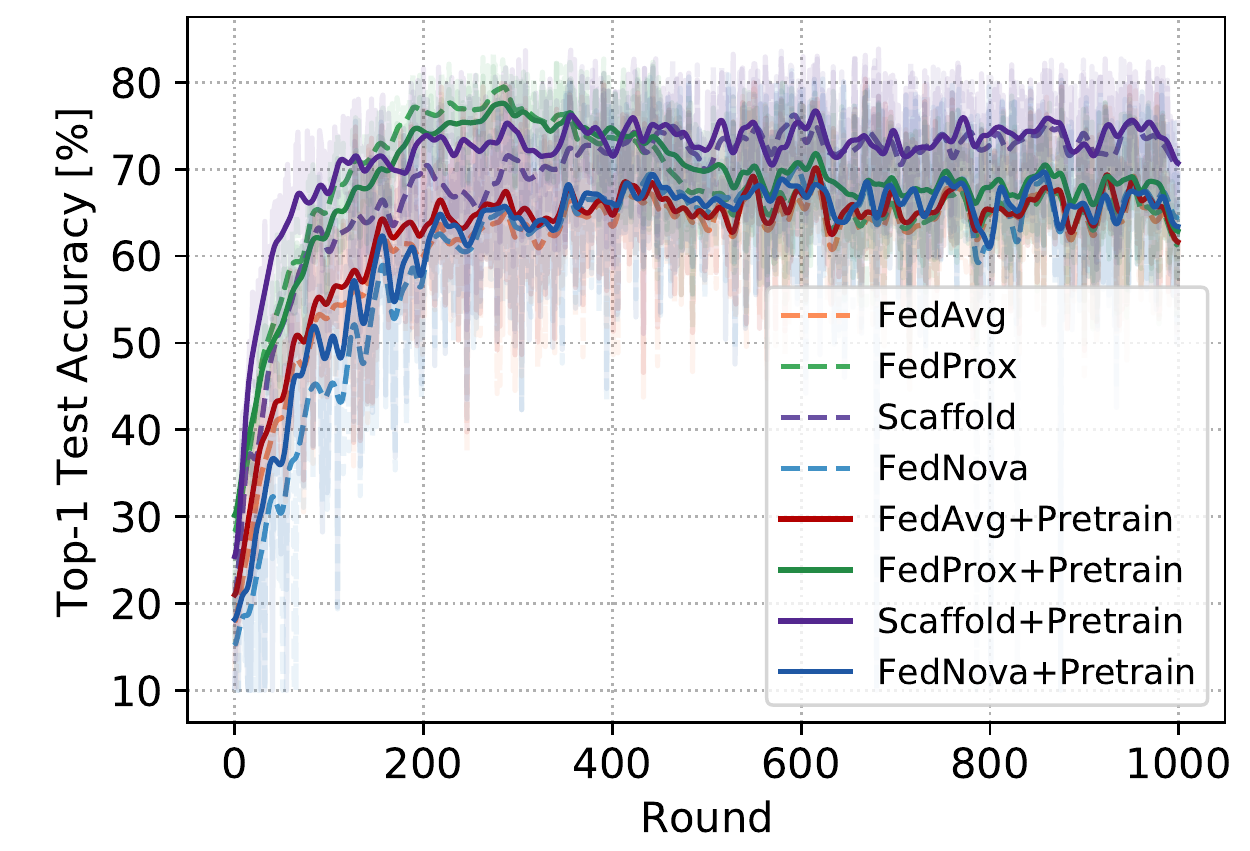}}
     \subfigure[VHL with different $B_v$ of virtual dataset. ]{\includegraphics[width=0.32\textwidth]{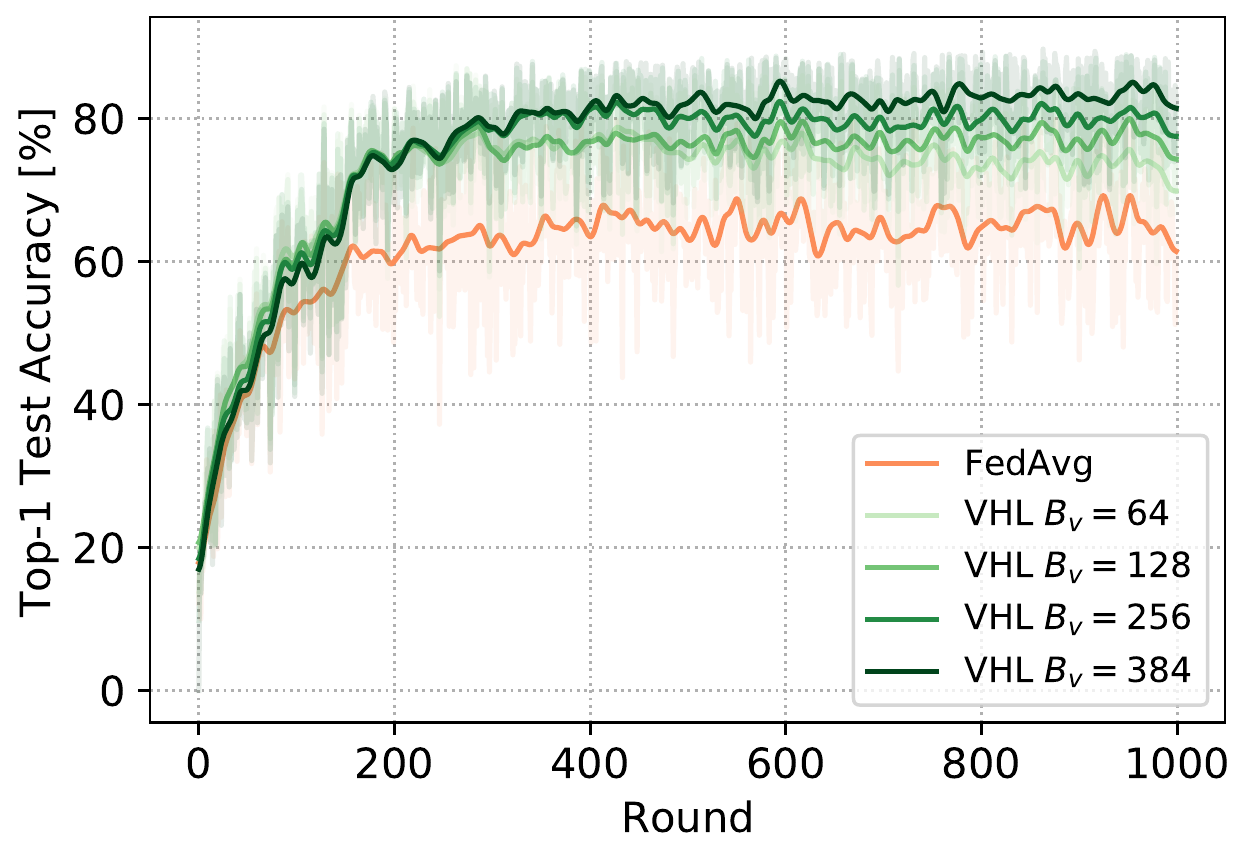}}
     \subfigure[Training ACC. on Virtual Dataset. ]{\includegraphics[width=0.32\textwidth]{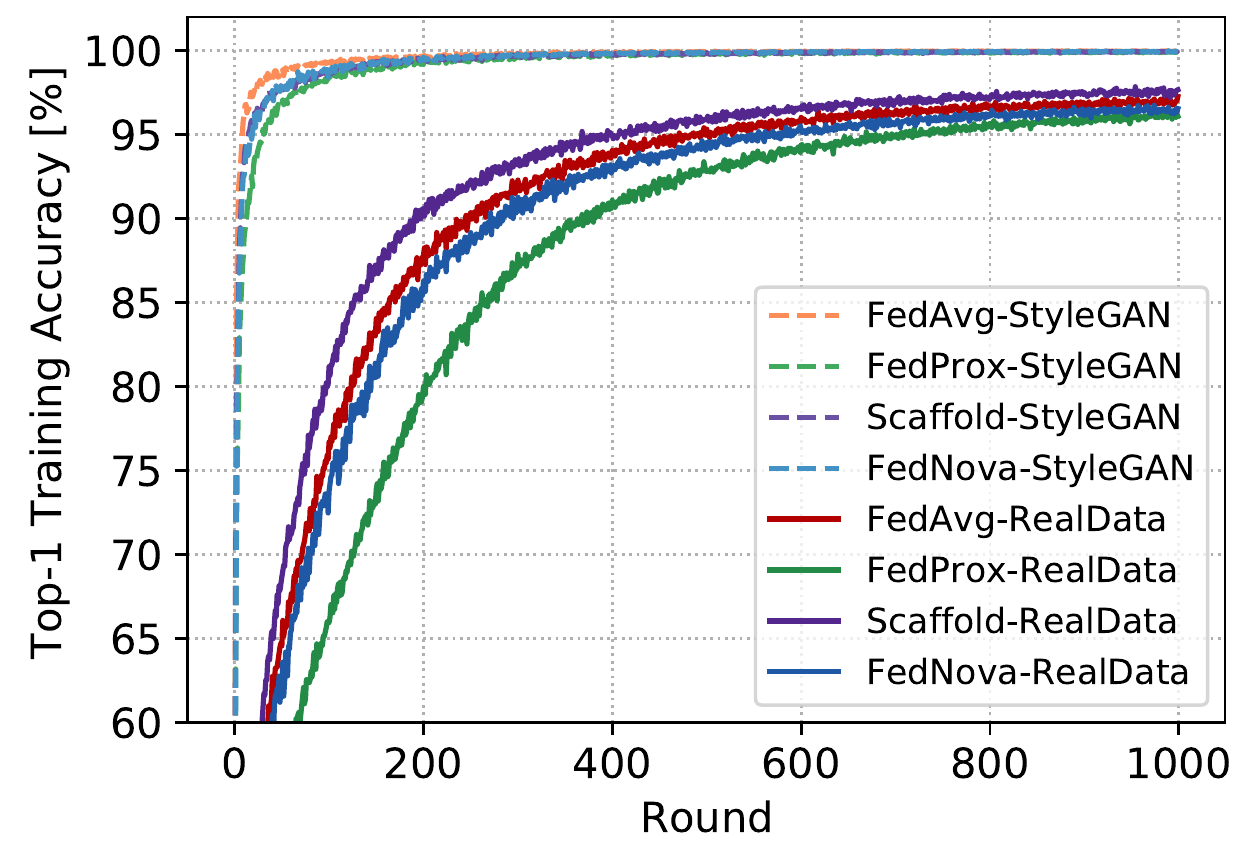}}
\caption{Convergence of more experiments.}
\label{fig:Ablation-Convergence}
\vspace{-0.1cm}
% \vspace{-0.5cm}
\end{figure}

% \begin{figure}[htb!]
% \small
% \setlength{\abovedisplayskip}{-2pt}
% \subfigbottomskip=-1pt
% \subfigcapskip=1pt
% \centering
% {\includegraphics[width=0.5\linewidth]{Ablation/UsingPretrained-resnet18_v2-cifar10-0.1-10-1.pdf}}
% \caption{Convergence of Using Model Pretrained on Virtual Dataset.}
% \label{fig:Ablation-Pretrain}
% % \vspace{-0.5cm}
% \end{figure}

% \begin{figure}[htb!]
% \small
% \setlength{\abovedisplayskip}{-2pt}
% \subfigbottomskip=-1pt
% \subfigcapskip=1pt
% \centering
% {\includegraphics[width=0.5\linewidth]{Ablation/AuxDataAcc-DifDataset-resnet18_v2-cifar10-0.1-10-1.pdf}}
% \caption{Training Accuracy on Virtual Dataset.}
% \label{fig:Ablation-TrainAcc-VirtualDataset}
% % \vspace{-0.5cm}
% \end{figure}

% \begin{figure}[htb!]
% \small
% \centering
% \setlength{\abovedisplayskip}{-2pt}
% \subfigbottomskip=-1pt
% \subfigcapskip=1pt
% {\includegraphics[width=0.9\linewidth]{Ablation/VFA-resnet18_v2-cifar10-0.1-10-1.pdf}}
% \caption{Virtual Feature Alignment.}
% \label{fig:Ablation-VFA}
% % \vspace{-0.5cm}
% \end{figure}

% \begin{figure}[htb!]
% \small
% \centering
% \setlength{\abovedisplayskip}{-2pt}
% \subfigbottomskip=-1pt
% \subfigcapskip=1pt
% {\includegraphics[width=0.9\linewidth]{Ablation/NaiveVHL-resnet18_v2-cifar10-0.1-10-1.pdf}}
% \caption{Simply Training with Virtual Dataset.}
% \label{fig:Ablation-NaiveVHL}
% % \vspace{-0.5cm}
% \end{figure}

\subsection{Comparisons of Calculation Cost}\label{appendix:CalculationCost}

As mentioned in section~\ref{sec:limitations}, training on the virtual data causes the extra calculation cost. Thus, similar with comparing the convergence speed, we compare the total calculation cost of different algorithms to achieve the target accuracy in Table~\ref{tab:CalculationCost}. It shows how many data samples that each algorithm needs to process to achieve the target accuracy. We use $N_{total}$ and $R_{total}$ to represent the number of processed samples and the number of communication rounds to achieve the target accuracy, $R_{total}$ are shown in Table~\ref{tab:MainResult-cifar10},~\ref{tab:MainResult-fmnist},~\ref{tab:MainResult-SVHN} and ~\ref{tab:MainResult-cifar100}. And the $D_{data}$ represents the total size of all datasets from clients.

Note that for $E=1$ and $K=10$ setting, the $N_{total}= D_{data} \times R_{total} \times 0.5$, as the server randomly selects 5 clients each round. For $E=5$ and $K=10$ setting, $N_{total}= D_{data} \times R_{total} \times 0.5 \times 5$, as each client conducts 5 local epochs before communication. For $E=1$ and $K=100$ setting, $N_{total}= D_{data} \times R_{total} \times 0.1 $, as the server randomly selects 10 clients out of 100 clients each round. Because we sample virtual data with the same batch size as the natural data, the extra calculation cost is doubled for VHL.

Table~\ref{tab:CalculationCost} shows that for most of cases, VHL needs to process similar or less number of data samples to attain the target test accuracy. However, for CIFAR-100 dataset, VHL shows more calculation cost. This trend aligns with the communication, i.e. VHL shows less performance gains on CIFAR-100. To address this problem, one may consider reducing the sampling size of virtual data. Or, we may explore more interesting properties of VHL to make it computation-friendly, like the VFA discussed in section~\ref{sec:DiffFacetsVHL}.

\begin{table*}[t]
\centering
\caption{Calculation cost ($\times D_{data} $ ) to achieve the target accuracy.} 
\vspace{1pt}
\scriptsize{
\begin{tabular}{cc|c|c|cccc}
\toprule[1.5pt]
  \!\! &  & Target ACC  & w/(w/o) \textbf{VHL} & {\small FedAvg} & \small FedProx &\small SCAFFOLD &\small FedNova \\
\cmidrule[1.5pt]{1-8}
\multirow{8}{*}{CIFAR-10} &  \multirow{2}{*}{$a=0.1$, $E=1$, $K=10$} &  \multirow{2}{*}{100}  & without          & 143.5 & 94 & 145.5 & 175.5  \\ 
        &    &      &  with          &  128 & 128 & \textbf{90} & 128 \\
\cmidrule[1pt]{2-8}
        &  \multirow{2}{*}{$a=0.05$, $E=1$, $K=10$} &  \multirow{2}{*}{100} &  without          & 205.5 & \textbf{100.5} & Nan & Nan \\ 
        &     &     &  with          &  112 & 151 & Nan & 247 \\
\cmidrule[1pt]{2-8}
        &  \multirow{2}{*}{$a=0.1$, $E=5$, $K=10$} &  \multirow{2}{*}{100}  &  without          & 637.5 & Nan & \textbf{165} & 317.5 \\ 
        &    &      &  with          &  455 & 1275 & 225 & 335 \\
\cmidrule[1pt]{2-8}
        &  \multirow{2}{*}{$a=0.1$, $E=1$, $K=100$}  &  \multirow{2}{*}{100}  &  without          & 95.7 & 84.2 & 66.4 & Nan \\ 
        &      &    &  with          & 77 & \textbf{65} & 95.8 & 110.8 \\
\cmidrule[1.5pt]{1-8}
\multirow{8}{*}{FMNIST} &  \multirow{2}{*}{$a=0.1$, $E=1$, $K=10$}  &  \multirow{2}{*}{100} &  without          & 59.5 & 67.5 & 71.5 & 41.5 \\ 
        &   &       &  with          &   52 & 31 & \textbf{14} & 52 \\
\cmidrule[1pt]{2-8}
        &  \multirow{2}{*}{$a=0.05$, $E=1$, $K=10$}  &  \multirow{2}{*}{100} &  without          &  212.5 & \textbf{20.5} & Nan & 269   \\
        &    &      &  with          &  53 & 30 & 58 & 30 \\
\cmidrule[1pt]{2-8}
        &  \multirow{2}{*}{$a=0.1$, $E=5$, $K=10$} &  \multirow{2}{*}{100}  &  without          & 695 & Nan & 262.5 & 482.5 \\
        &    &      &  with          &   255 & 370 & \textbf{100} & 255 \\
\cmidrule[1pt]{2-8}
        &  \multirow{2}{*}{$a=0.1$, $E=1$, $K=100$} &  \multirow{2}{*}{100}  &  without          &65.8 & \textbf{49.1} & Nan & Nan \\
        &      &    &  with          &  87.2 & 56.6 & Nan & Nan \\
\cmidrule[1.5pt]{1-8}
\multirow{8}{*}{SVHN} &  \multirow{2}{*}{$a=0.1$, $E=1$, $K=10$} &  \multirow{2}{*}{100}  &  without          & 125.5 & Nan & Nan & 125.5 \\
        &    &      &  with          & \textbf{75} & 271 & Nan & \textbf{75} \\
\cmidrule[1pt]{2-8}
        &  \multirow{2}{*}{$a=0.05$, $E=1$, $K=10$} &  \multirow{2}{*}{100}  &  without          & 178.5 & Nan & Nan & 370.5 \\
        &     &     &  with          &  \textbf{94} & 320 & 147 & 128 \\
\cmidrule[1pt]{2-8}
        &  \multirow{2}{*}{$a=0.1$, $E=5$, $K=10$} &  \multirow{2}{*}{100}  &  without          & \textbf{327.5} & Nan & Nan & 405 \\ 
        &     &     &  with          & 725 & 1755 & 1050 & 375 \\
\cmidrule[1pt]{2-8}
        &  \multirow{2}{*}{$a=0.1$, $E=1$, $K=100$}  &  \multirow{2}{*}{100} &  without          & \textbf{61.8} & \textbf{61.8} & 64.3 & Nan \\ 
        &    &      &  with          &  72.4 & 71.2 & 193.6 & 135.2 \\
\cmidrule[1.5pt]{1-8}
\multirow{8}{*}{CIFAR-100} &  \multirow{2}{*}{$a=0.1$, $E=1$, $K=10$} &  \multirow{2}{*}{100}  &  without          & 248.5 & Nan & 383 & \textbf{236} \\ 
        &   &       &  with          & 384 & 617 & 294 & 384 \\
\cmidrule[1pt]{2-8}
        &  \multirow{2}{*}{$a=0.05$, $E=1$, $K=10$} &  \multirow{2}{*}{100}  &  without          & \textbf{257} & Nan & Nan & Nan \\ 
        &  &        &  with         & 354 & 482 & Nan & 320 \\
\cmidrule[1pt]{2-8}
        &  \multirow{2}{*}{$a=0.1$, $E=5$, $K=10$} &  \multirow{2}{*}{100}  &  without          & 707.5 & Nan & \textbf{427.5} & 730 \\
        &    &      &  with          &   1635 & Nan & 1455 & Nan \\
\cmidrule[1pt]{2-8}
        &  \multirow{2}{*}{$a=0.1$, $E=1$, $K=100$} &  \multirow{2}{*}{100}  &  without &    96.7 & 95.5 & \textbf{82.7} & 96.7 \\ 
        &   &       &  with          &  143.4 & 143.4 & 131.2 & 159.4 \\
\bottomrule[1.5pt] 
\end{tabular}
}
\label{tab:CalculationCost}
\end{table*}

\end{document}